\documentclass[11pt]{article}
\usepackage[margin=1in]{geometry}
\usepackage[T1]{fontenc}
\usepackage[utf8]{inputenc}
\usepackage{amsmath,amssymb,amsthm,mathtools}
\usepackage{parskip}
\usepackage{graphicx}
\usepackage{booktabs}
\usepackage{enumitem}
\usepackage{xcolor}
\usepackage[colorlinks=true,linkcolor=blue!50!black,citecolor=blue!50!black,urlcolor=blue!50!black]{hyperref}
\usepackage[activate=false]{microtype}
\usepackage[section]{placeins}
\usepackage{float}

\newtheorem{proposition}{Proposition}

\newtheorem{corollary}{Corollary}
\theoremstyle{definition}

\theoremstyle{remark}
\newtheorem{remark}{Remark}

\title{The Map Behind the Flow:\\Finite-Step Gradient Descent as a Dynamical System}

\author{Thomas Hofmann \\[2mm] Department of Computer Science, ETH Zurich \\ thomas.hofmann@inf.ethz.ch}
\date{July 6, 2026}

\begin{document}
\maketitle
\begin{abstract}
\noindent
Many characteristic phenomena of deep learning are dynamical: they concern not only which minima exist, but how gradient descent reaches, avoids, or selects among them. Edge-of-stability behavior, sharpness oscillations, catapult phases, balancing, and implicit movement toward flatter representations are effects of the training map itself. They are therefore poorly captured by the small-step limit, where gradient descent is replaced by gradient flow. This paper studies fixed-step gradient descent directly, as a discrete dynamical system, in a hierarchy of exactly solvable models that retain basic structures of deep learning: depth, factorization, width, data coupling, activation, and stochasticity.

Our starting point is the balanced scalar reduction of a deep linear chain. This reduction gives a quartic loss and a cubic gradient map whose post-edge behavior can be analyzed explicitly. Under the natural large-depth scaling, the dynamics converges to a universal Ricker-type map. Thus the edge of stability is not a breakdown of optimization, but the first bifurcation of the training map.
Embedding the scalar map back into factored models shows how these dynamical regimes appear as learning phenomena. Finite steps break conservation laws of gradient flow and contract factorization imbalance; residual oscillations move parameters toward flatter, more balanced representations. In wider linear networks, different singular modes cross their edges at different learning rates, producing a ladder of spectral edges and allowing the optimal learning rate to lie beyond the first edge.  Data coupling, nonlinear activations, and stochastic targets preserve the same organizing principle: finite-step oscillations drive alignment, balancing, and representation selection. We show that the learning rate is not merely a numerical stability parameter. It is a structural parameter of the training dynamics, determining its attractors and shaping which representations gradient descent selects.
\end{abstract}
\tableofcontents
\section{Introduction}
\label{sec:intro}

At the dawn of the modern age, Descartes described his method as one that begins 
\textit{with objects the simplest and easiest to know} \cite{Descartes}. 
We follow the same principle here. 
The simplest model is not the final object of interest; it is a controlled setting in which finite-step effects can be isolated from the additional complications of high-dimensional neural-network training.

The central question of this paper is what happens when gradient descent is studied as a discrete dynamical system rather than as a small-step approximation to gradient flow. 
Given a loss \(L\), fixed-step gradient descent defines the map
\[
    g_a(\theta)=\theta-a\nabla L(\theta), \qquad a>0 .
\]
Training is therefore the iteration of a nonlinear map. 
The local quadratic approximation determines the first stability threshold through the spectrum of the Hessian. 
Beyond that threshold, however, the quadratic approximation no longer controls the dynamics. 
The subsequent behavior is governed by the nonlinear geometry of the gradient map: its invariant regions, attracting cycles, bifurcations, and escape mechanisms.

This perspective is motivated by experimental evidence. 
Modern neural networks are routinely trained with learning rates at or beyond the local stability limit. 
Cohen et al.~\cite{cohen2021gradient} observed that full-batch training drives the largest Hessian eigenvalue toward the threshold \(2/a\) --- the \emph{edge of stability} --- and then continues to make progress while oscillating. 
Subsequent work has analyzed mechanisms by which optimization remains organized beyond this point, including two-step oscillations \cite{chenbruna2023}, higher-order self-stabilization \cite{damian2023self}, trajectory alignment near bifurcations \cite{songyun2023trajectory}, and catapult dynamics at large learning rates \cite{lewkowycz2020large}. 
Post-edge phenomena have also been identified in exactly tractable models, including quadratic regression, matrix factorization, and deep linear networks, where increasing the learning rate can lead from convergence through oscillation and period doubling to chaos \cite{zhu2023understanding,chen2023stability,Ghosh2025DeepMatrixFactorizationEOS,ghosh2025dln}.

Our goal is to develop this post-edge picture systematically. 
We proceed through a sequence of models chosen so that every claim can be reduced to a statement about an explicit map. 
The paper is organized in two main parts, followed by a conclusion and appendices.

\medskip
\emph{Part I} studies the balanced scalar reduction
\[
    \ell_n(x)=\frac{1}{4n}\bigl(x^{2n}-1\bigr)^2 .
\]
Already the quartic case \((n=1)\) contains the post-edge phenomena of interest: an explicit two-cycle, sharpness hovering around \(2/a\), period doubling, band merging, and escape. 
Under the large-depth scaling \(a=c/n\), the maps converge to a universal Ricker-type limit, making the post-edge phase diagram depth-independent. 
The finite-depth theory interpolates analytically between the quartic and limiting regimes. 
In this interpolation, the quotient maps retain one-dimensional structure, while divergence is controlled by a separate outer-scale escape mechanism.

\medskip
\emph{Part II} reintroduces, one by one, the ingredients absent from the scalar model. 
The two-factor system shows that finite-step gradient descent breaks a conservation law of gradient flow: residual oscillations contract factorization imbalance and move the parameters toward flatter, more balanced minima. 
Wider linear networks reveal a ladder of spectral edges and an exact alignment threshold; they also show that the optimal learning rate can lie beyond the first edge. 
Data coupling, nonlinear activations, and stochastic targets preserve the same organizing principle: finite-step dynamics drives alignment, balancing, and representation selection. 
In the stochastic setting, persistent label noise makes the balanced representation transversely attracting below the edge and connects continuously to deterministic post-edge dynamics through a universal crossover.

\medskip
The overall picture that emerges is this: gradient descent can cross the local stability threshold and remain organized, not because the quadratic instability is an artifact, but because nonlinear finite-step effects---oscillation, self-stabilization, and balancing---reshape the dynamics beyond the edge. 
In the models studied here, these mechanisms admit exact descriptions and provide a dynamical explanation for edge-of-stability training, post-edge oscillations, balancing, and representation selection.

The exposition follows this principle. 
Rather than beginning with full neural-network models and then simplifying them, we first solve the underlying dynamical systems and then embed them into progressively richer learning models.

\part{Scalar Dynamical Systems}
\label{part:scalar-dynamical-systems}
\section{The Cubic Map}
\label{sec:cubic}

Our starting point is the simple quartic potential
\[
    \ell(x)=\frac14(x^2-1)^2 ,
\]
with minima at \(x=\pm1\) and a local maximum at \(x=0\). 
Fixed-step gradient descent with learning rate \(a>0\) gives the cubic map
\[
    g_a(x)=x-a x(x^2-1).
\]
As we shall see, this seemingly elementary model has surprisingly far-reaching consequences for learning dynamics.

\subsection{Stable Fixed Points}

Passing to the quotient coordinate
\[
    u=x^2
\]
removes the symmetry \(x\mapsto -x\), while retaining both the loss value and the radial dynamics. The induced map is
\[
    h_a(u)=u\bigl(1-a(u-1)\bigr)^2 .
\]
It maps the positive half-line into itself and is unimodal on the interval
\[
    0\leq u\leq 1+\frac1a .
\]
The fixed point corresponding to the minima \(x=\pm1\) is \(u=1\), and
\[
    h_a'(1)=1-2a .
\]
Thus \(u=1\) is locally stable precisely when
\[
    |1-2a|<1
    \qquad\Longleftrightarrow\qquad
    0<a<1 .
\]

The remaining fixed points follow from
\[
    h_a(u)=u
    \qquad\Longleftrightarrow\qquad
    u=0
    \quad\text{or}\quad
    \bigl(1-a(u-1)\bigr)^2=1 .
\]
Hence
\[
    u=0,\qquad u=1,\qquad u=1+\frac{2}{a}.
\]
Here \(u=0\) is the quotient of the barrier \(x=0\). The point
\[
    u=1+\frac{2}{a}
\]
is the outer fixed point of the quotient dynamics. In \(x\)-coordinates it corresponds to the sign-flip fixed points of \(g_a\), since
\[
    g_a(x)=-x
    \qquad\Longleftrightarrow\qquad
    x^2=1+\frac{2}{a}.
\]
By contrast, the sign of one gradient step changes already at
\[
    x^2=1+\frac1a,
\]
where the multiplier \(1-a(x^2-1)\) vanishes.

For \(0<a<1\), the fixed point \(u=1\) attracts the whole interval
\[
    0<u<1+\frac{2}{a}.
\]
Equivalently, the two minima \(x=\pm1\) are stable fixed points of the gradient map \(g_a\), and together they attract every initial point with
\[
    0<|x|<\sqrt{1+\frac{2}{a}}\,.
\]
The two basins are not separated by sign. If \(x_0^2>1+1/a\), then the first step changes sign, and the orbit may settle in the opposite well.
\subsection{Loss Monotonicity} 

Local stability of the fixed point is weaker than monotone decrease of the loss. 
We track the residual
\[
    r=u-1 .
\]
Since the loss is proportional to \(r^2\), one step decreases the loss precisely when the residual contracts in absolute value. The quotient map gives
\[
    r_+=h_a(1+r)-1=r q_a(r),
\]
where
\[
    q_a(r)=a^2r^2+(a^2-2a)r+(1-2a).
\]
Thus one-step loss decrease is equivalent to
\[
    |r_+|\le |r|,
    \qquad\text{or equivalently}\qquad
    |q_a(r)|\le 1 .
\]

On the basin interval
\[
    -1<r<\frac{2}{a},
\]
the upper bound \(q_a(r)\le 1\) is automatic. Indeed, \(q_a(r)-1\) vanishes at the two endpoints \(r=-1\) and \(r=2/a\) and is nonpositive between them. The only obstruction is therefore the lower bound \(q_a(r)\ge -1\). Completing the square gives
\[
    q_a(r)
    =
    a^2\left(r-\frac{2-a}{2a}\right)^2
    -
    a-\frac{a^2}{4}.
\]
Hence monotone loss decrease throughout the whole basin holds exactly when
\[
    \min_r q_a(r) = -a-\frac{a^2}{4}\ge -1 
    \qquad \Longleftrightarrow \qquad
    a^2+4a-4\le 0 .
\]
The positive root is
\[
    a_{\mathrm{cat}}=2(\sqrt2-1).
\]

\begin{proposition}[Monotone-loss regime]
For
\[
    0<a \le a_{\mathrm{cat}} := 2(\sqrt2-1) = 0.828427 \ldots ,
\]
gradient descent decreases the loss monotonically for every initial condition in the basin
\[
    0<|x|<\sqrt{1+\frac{2}{a}} .
\]
For \(a_{\mathrm{cat}}<a<1\), the minima remain locally stable and all points in the basin still converge to one of them, but the loss need not decrease at every step.
\end{proposition}
\subsection{Period-Two Orbit}
\label{sec:quartic-two-cycle}

At the edge of stability \(a=1\), the minima lose local stability. The next object is the period-two orbit born inside each well. In the positive well we look for two distinct points
\begin{align*}
    0<x_-<1<x_+
\end{align*}
such that
\begin{align*}
    g_a(x_-)=x_+,
    \qquad
    g_a(x_+)=x_- .
\end{align*}
For the cubic gradient map this means
\begin{align*}
    x_+=(1+a)x_- -a x_-^3,
    \qquad
    x_-=(1+a)x_+ -a x_+^3 .
\end{align*}
Adding and subtracting the two equations gives (cf.~Appendix~\ref{sec:appendix-quartic})
\begin{align*}
    x_-+x_+
    =
    \sqrt{1+\frac3a},
    \qquad
    x_-x_+
    =
    \frac1a .
\end{align*}
Thus \(x_-\) and \(x_+\) are the roots of
\begin{align*}
    t^2
    -
    \sqrt{1+\frac3a}\,t
    +
    \frac1a
    =
    0 .
\end{align*}
By symmetry, the negative well contains the reflected orbit \(-x_+,-x_-\).

\begin{proposition}[Period-two orbit]
For \(a>1\), the gradient map \(g_a\) has a period-two orbit in the positive well given by
\begin{align*}
    x_\pm
    =
    \frac12
    \left(
        \sqrt{1+\frac3a}
        \pm
        \sqrt{1-\frac1a}
    \right).
\end{align*}
It bifurcates from the fixed point \(x=1\) at \(a=1\). The reflected orbit gives the corresponding period-two orbit in the negative well.
\end{proposition}

The stability of the period-two orbit is determined by the derivative of the two-step return map. For the cycle
\begin{align*}
    x_- \mapsto x_+ \mapsto x_-,
\end{align*}
we define the multiplier
\begin{align*}
    \mu_{\mathrm{cyc}}(a)
    :=
    (g_a^2)'(x_-)
    =
    g_a'(x_-)g_a'(x_+).
\end{align*}
Using the explicit cycle points above, this simplifies to (cf.~Appendix~\ref{sec:appendix-cycle-multiplier})
\begin{align}
    \mu_{\mathrm{cyc}}(a)
    =
    9-2(1+a)^2.
    \label{eq:quartic-cycle-multiplier}
\end{align}
The two-cycle is attracting precisely when
\begin{align*}
    |\mu_{\mathrm{cyc}}(a)|<1 .
\end{align*}
Since the orbit is born at \(a=1\), the relevant branch is \(a>1\), and the stability interval is
\begin{align*}
    a_1 := 1<a<\sqrt5-1 =: a_2.
\end{align*}
Inside this interval there is a superstable point, where \(\mu_{\mathrm{cyc}}(a)=0\). At the upper endpoint \(a_2\), the multiplier reaches \(-1\), and the two-cycle undergoes its own period-doubling bifurcation.

\begin{proposition}[Stability of the two-cycle]
For the quartic gradient map \(g_a\), the period-two orbit born at \(a=1\) is attracting exactly for
\begin{align*}
    1<a<\sqrt5-1 = 1.23606  \ldots =: a_2.
\end{align*}
It is superstable at
\begin{align*}
    a_{0} :=\frac3{\sqrt2}-1 = 1.12132\ldots .
\end{align*}
\end{proposition}

\subsection{Sharpness Hovering}
\label{sec:quartic-sharpness-hovering}

We now reverse the usual local-stability viewpoint. 
Instead of fixing \(x\) and asking which learning rates are stable, we fix \(a\) and ask which curvature would place a point at the one-step stability edge. 
Since
\[
    g_a'(x)=1-a\ell''(x),
\]
the upper edge is reached when \(g_a'(x)=-1\), equivalently
\[
    \ell''(x)=\frac2a .
\]
For the quartic loss,
\[
    \ell''(x)=3x^2-1,
\]
so the critical-sharpness point satisfies
\[
    x^2=\frac13\left(1+\frac2a\right).
\]

The post-edge two-cycle is not selected by this pointwise equation. 
It is selected by the two-step return map, and its stability is determined by the product of the two one-step multipliers along the cycle. 
The two points therefore straddle the one-step stability edge:
\[
    \ell''(x_-)
    <
    \frac2a
    <
    \ell''(x_+),
\]
or equivalently,
\[
    x_-^2
    <
    \frac13\left(1+\frac2a\right)
    <
    x_+^2 .
\]
The proof is given in Appendix~\ref{sec:appendix-sharpness-hovering}. 
Thus the flatter point \(x_-\) lies below the one-step stability edge, while the sharper point \(x_+\) lies beyond it. 
In this precise sense, the orbit hovers around the sharpness edge.

\begin{proposition}[Sharpness hovering]
For the post-edge two-cycle \(x_-\mapsto x_+\mapsto x_-\), the two orbit points lie on opposite sides of the one-step stability edge:
\[
    \ell''(x_-)
    <
    \frac2a
    <
    \ell''(x_+) .
\]
Thus the cycle is not characterized by pointwise critical sharpness, but by alternation between subcritical and supercritical curvature.
\end{proposition}

A useful scalar summary is the average curvature
\[
    \frac{\ell''(x_-)+\ell''(x_+)}2
    =
    \frac12+\frac{3}{2a}.
\]
Near \(a=1\), this is close to the nominal edge value \(2/a\), but the sharper statement is the straddling inequality above.
\subsection{Period Doubling, Accumulation, and Chebyshev Endpoint}
\label{sec:quartic-period-doubling}

At
\[
    a_2 = \sqrt5-1,
\]
the multiplier of the two-cycle reaches \(-1\). 
The stable two-cycle loses stability, and a stable four-cycle is born. 
Continuing in the same way gives the period-doubling cascade
\[
    2
    \;\longrightarrow\;
    4
    \;\longrightarrow\;
    8
    \;\longrightarrow\;
    16
    \;\longrightarrow\;
    \cdots .
\]
Thus the edge of stability at \(a=1\) is only the first transition: it is followed by a stable two-cycle, then by successive period doublings.

The period-doubling thresholds accumulate numerically at
\[
    a_\infty
    =
    1.302283\ldots .
\]
The gaps between successive thresholds shrink geometrically, with ratios approaching the Feigenbaum constant. 
The numerical values and defining equations are given in Appendix~\ref{sec:appendix-period-doubling}. 
Beyond the accumulation point, the dynamics enters chaotic regimes, interrupted by periodic windows. 
In optimization language, the route is
\[
    \text{fixed point}
    \;\longrightarrow\;
    \text{two-cycle}
    \;\longrightarrow\;
    \text{period-doubling cascade}
    \;\longrightarrow\;
    \text{chaotic finite-step dynamics}.
\]

The other exact landmark occurs at \(a=2\). 
At this learning rate,
\[
    g_2(x)=3x-2x^3 .
\]
After the rescaling \(x=\sqrt2 z\), this becomes
\[
    \frac1{\sqrt2}g_2(\sqrt2 z)
    =
    3z-4z^3
    =
    -T_3(z),
\]
where \(T_3(z)=4z^3-3z\) is the third Chebyshev polynomial. 
Thus, in the cosine coordinate \(z=\cos\theta\), the dynamics is angle tripling up to a sign:
\[
    T_3(\cos\theta)=\cos(3\theta).
\]
The interval \([-\sqrt2,\sqrt2]\) is invariant. 
In the rescaled coordinate \(z\), the invariant density is the arcsine density
\[
    d\mu(z)=\frac{1}{\pi\sqrt{1-z^2}}\,dz .
\]
For \(a>2\), this Chebyshev invariant interval is destroyed and trajectories outside the surviving bounded set escape.

\begin{proposition}[Chebyshev endpoint]
At \(a=2\), the quartic gradient map is conjugate, after the rescaling \(x=\sqrt2 z\), to the negative Chebyshev map \(-T_3\) on \([-1,1]\). 
Thus \(a=2\) is an explicitly solvable chaotic endpoint of the bounded quartic dynamics.
\end{proposition}
\subsection{Envelope of the Post-Edge Band}
\label{sec:quartic-envelope}

The outer envelope of the bounded post-edge dynamics is controlled by the critical values of the cubic map. 
The positive critical point is
\[
    x_c=\sqrt{\frac{1+a}{3a}},
\]
and its image is
\[
    E(a)=g_a(x_c)
    =
    \frac{2(1+a)^{3/2}}{3\sqrt{3a}} .
\]
By odd symmetry, the negative critical value is \(-E(a)\). 
For \(1/2\le a\le2\), the interval \([-E(a),E(a)]\) is mapped into itself; at \(a=2\), this gives \(E(2)=\sqrt2\), the Chebyshev interval. 
Thus bounded iterates overshoot the minima \(\pm1\) by at most \(E(a)-1\).

The same critical-value curve also detects when the two wells cease to be dynamically separated. 
Let
\[
    z_a=\sqrt{\frac{1+a}{a}}
\]
be the positive zero of \(g_a\). 
The positive side remains separated from the negative side as long as the positive critical value stays below this zero. 
Since
\[
    \frac{E(a)}{z_a}
    =
    \frac{2(1+a)}{\sqrt{27}},
\]
this holds exactly for
\[
    a\le a_\ast:=\frac{\sqrt{27}}2-1 .
\]
At \(a=a_\ast\), the critical value reaches the sign-change zero \(z_a\). 
For larger \(a\), cross-well transitions become possible, and the two chaotic bands merge through the origin.

\begin{proposition}[Envelope and well merging]
The outer envelope of the bounded post-edge dynamics is the critical-value curve
\[
    E(a)=\frac{2(1+a)^{3/2}}{3\sqrt{3a}} .
\]
The two wells remain dynamically separated up to
\[
    a_\ast=\frac{\sqrt{27}}2-1\approx1.5980\ldots .
\]
At this threshold the positive critical value reaches the positive zero of the map, and cross-well transitions become possible.
\end{proposition}
\subsection{Connection to the Classical Cubic Family}

The direct calculations above are a rescaled version of the classical cubic family studied by May and by Rogers--Whitley,
\[
    f_\alpha(z)=\alpha z^3+(1-\alpha)z,
    \qquad
    0\le \alpha\le4 .
\]
Indeed, with
\[
    \rho_a=\sqrt{\frac{a+2}{a}},
    \qquad
    x=\rho_a z,
\]
one obtains
\[
    \rho_a^{-1}g_a(\rho_a z)
    =
    (1+a)z-a\rho_a^2z^3
    =
    (1+a)z-(a+2)z^3
    =
    -f_{a+2}(z).
\]
Thus the parameter dictionary is simply
\[
    \alpha=a+2 .
\]
Since \(f_\alpha\) is odd, the global sign disappears in even iterates:
\[
    (-f_\alpha)^2=f_\alpha^2 .
\]
The quartic gradient map therefore inherits the classical cubic bifurcation structure.

Under this dictionary, the edge \(a=1\) corresponds to \(\alpha=3\), the two-cycle stability endpoint \(a=\sqrt5-1\) corresponds to \(\alpha=1+\sqrt5\), and the Chebyshev endpoint \(a=2\) corresponds to \(\alpha=4\). 
The Rogers--Whitley snapback threshold
\[
    \alpha_\ast=1+\frac{\sqrt{27}}2
\]
becomes
\[
    a_\ast=\alpha_\ast-2=\frac{\sqrt{27}}2-1 .
\]
This threshold is distinct from the period-doubling accumulation point \(a_\infty\approx1.3023\): the latter is the accumulation of the principal cascade, while \(a_\ast\) marks the global snapback, or cross-well band-merging, mechanism.

\begin{proposition}[Classical cubic dictionary]
The quartic gradient map
\[
    g_a(x)=(1+a)x-a x^3
\]
is linearly conjugate, up to the global sign, to the classical cubic family
\[
    f_\alpha(z)=\alpha z^3+(1-\alpha)z
\]
with
\[
    \alpha=a+2 .
\]
Thus the post-edge dynamics of gradient descent on the quartic potential is the optimization realization of the classical cubic bifurcation diagram.
\end{proposition}

\subsection{Bifurcation Diagram}

The quartic map separates several notions that are often conflated. 
The edge of stability is the local threshold \(a=1\), where the minima \(x=\pm1\) lose one-step stability. 
Crossing \(a=1\), however, does not destroy organized finite-step dynamics. 
It creates a stable two-cycle, followed by a period-doubling cascade, localized chaotic bands, cross-well chaos, and finally the Chebyshev endpoint at \(a=2\). 
Collecting the thresholds established above,
\[
    \underbrace{2\sqrt2-2}_{\text{monotone loss}}
    < \underbrace{1}_{\text{edge}}
    < \underbrace{\tfrac3{\sqrt2}-1}_{\text{superstable}}
    < \underbrace{\sqrt5-1}_{\text{two-cycle flip}}
    < \underbrace{a_\infty}_{\text{cascade limit}}
    < \underbrace{\tfrac{\sqrt{27}}2-1}_{\text{wells merge}}
    < \underbrace{2}_{\text{Chebyshev}} .
\]

The tool that makes the critical orbit decisive is the Schwarzian derivative
\[
    Sf=\frac{f'''}{f'}-\frac32\left(\frac{f''}{f'}\right)^2,
    \qquad f'\neq0 .
\]
For the quartic map one computes
\[
    Sg_a(x)
    =
    -\,\frac{6a\,(1+a+6a x^2)}{(1+a-3a x^2)^2}
    <0
\]
on every regular branch. 
By Singer's theorem, each attracting cycle of a negative-Schwarzian map attracts a critical point. 
The two critical points
\[
    \pm x_c
    =
    \pm\sqrt{\frac{1+a}{3a}}
\]
therefore locate all attracting cycles. 
This is why the bifurcation diagram can be generated by iterating the critical orbits.

Figure~\ref{fig:quartic-phase-diagram} shows the resulting phase diagram. 
For each learning rate \(a\), we iterate the two critical points \(\pm x_c\), whose forward orbits organize the attracting dynamics. 
The dashed lines mark the main transitions: the edge \(a=1\), the two-cycle flip, the accumulation of the principal period-doubling cascade, the cross-well threshold, and the Chebyshev endpoint.

\begin{figure}[ht]
    \centering
    \includegraphics[width=0.8\linewidth]{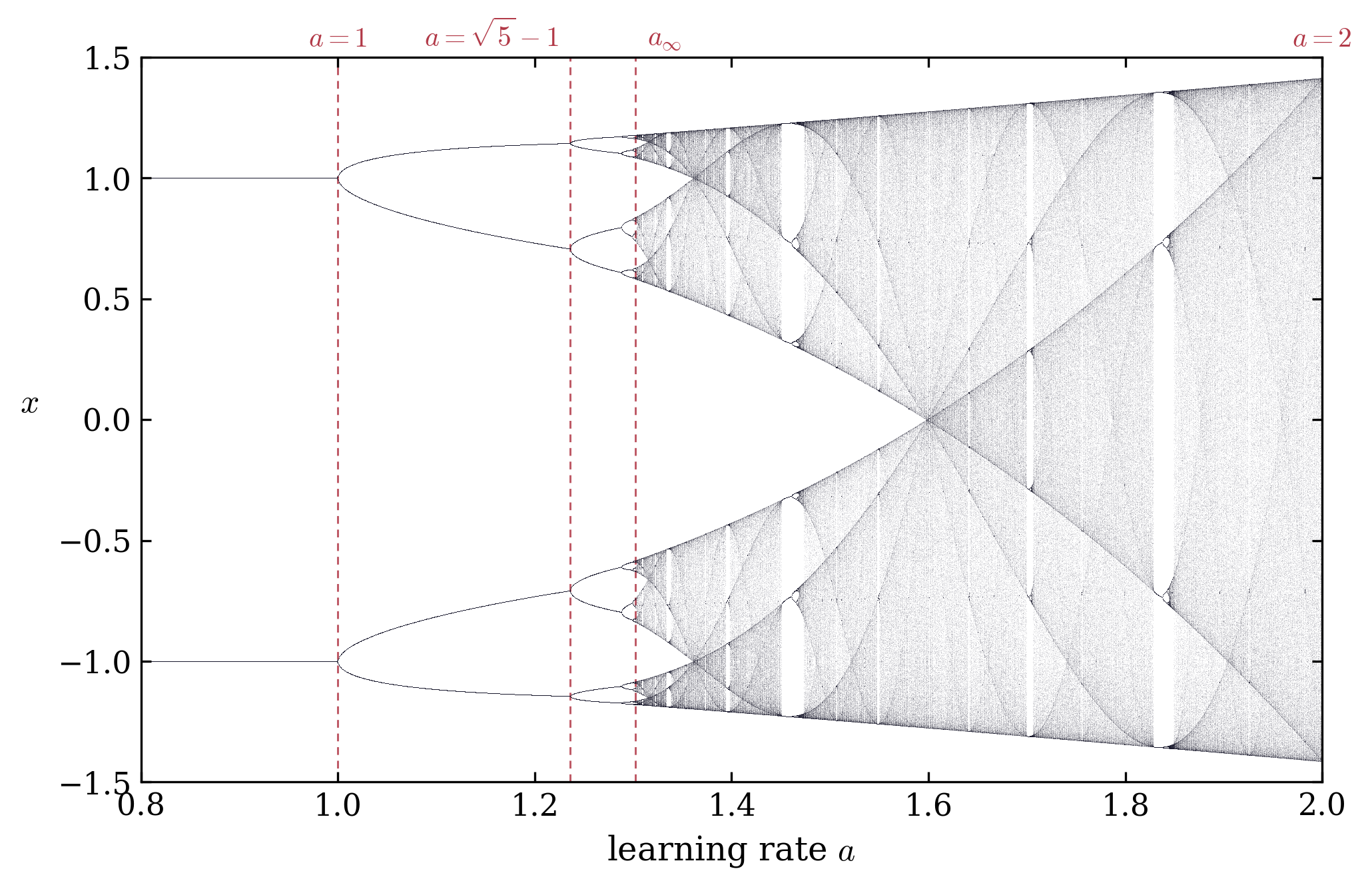}
    \caption{Global phase diagram of the cubic gradient map \(g_a(x)=(1+a)x-a x^3\). 
    The diagram is generated by iterating the two critical points \(\pm x_c=\pm\sqrt{(1+a)/(3a)}\). 
    Dashed lines mark \(a=1\), \(a_2=\sqrt5-1\), \(a_\infty\approx1.3023\), \(a_\ast=\sqrt{27}/2-1\), and \(a=2\).}
    \label{fig:quartic-phase-diagram}
\end{figure}

\subsection{Conclusion}

The quartic model shows that the local stability threshold is not an escape threshold. 
When \(a\) crosses \(1\), the attracting fixed point is replaced by an organized post-edge regime: a stable two-cycle, a period-doubling cascade, chaotic bands, cross-well transitions, and finally the Chebyshev endpoint at \(a=2\). 
These regimes are controlled by explicit features of the cubic gradient map---multipliers, critical values, and negative Schwarzian structure. 
The edge of stability is therefore not where finite-step optimization ends, but where its nonlinear dynamics begins.

\subsection{History and Related Work}

At the level of one-dimensional dynamics, the map \(g_a\) belongs to the classical odd cubic family studied by May~\cite{may1974biological} and Rogers--Whitley~\cite{rogers1983chaos}, with the parameter-space picture developed further by Branner--Hubbard~\cite{brannerhubbard1988} and Milnor~\cite{milnor1992}. 
The scalar bifurcation structure recorded above is therefore classical: the period-doubling cascade, its accumulation point \(a_\infty\), the snapback threshold \(a_\ast\), and the Chebyshev endpoint \(a=2\) all belong to the standard interval-dynamical theory of cubic maps.

The optimization interpretation is more recent. 
The catapult mechanism at large learning rates was described by Lewkowycz et al.~\cite{lewkowycz2020large}, while Zhu et al.~\cite{zhu2024quadratic} and Agarwala et al.~\cite{agarwala2023} analyzed catapult dynamics and edge-of-stability behavior in local quadratic models of gradient descent. 
Chen et al.~\cite{chen2023stability} derived the same cubic quotient map, in an affine coordinate, from phase retrieval and two-layer quadratic networks, and gave a global classification with explicit thresholds. 
In particular, the monotone-loss threshold \(2\sqrt2-2\), the divergence boundary at \(a=2\), and the transition from periodic to chaotic dynamics are already present in their analysis.

The point of the present section is to recast this classical cubic dynamics in a form adapted to learning and to the extensions that follow. 
The quotient coordinate \(u=x^2\) separates radial dynamics from sign changes, while the residual coordinate \(r=u-1\) exposes the monotone-loss threshold directly. 
The closed-form envelope \(E(a)\), the curvature-straddling interpretation of sharpness hovering, and the coincidence of the superstable point with the endpoint of the two-step contraction certificate refine the scalar picture in learning-theoretic terms. 
The substantive extension comes next: the same quotient perspective applies to \(u=x^{2n}\) at arbitrary depth and leads, under the scaling \(a=c/n\), to the universal Ricker limit.
\section{The Depth Limit}
\label{sec:depth-limit}

\subsection{Polynomial Family} 

The quartic model is exactly solvable, but it is only the first member of the normalized polynomial depth family
\begin{align*}
    \ell_n(x)=\frac{1}{4n}\left(x^{2n}-1\right)^2 .
\end{align*}
The corresponding family of gradient maps is
\begin{align*}
    g_{a,n}(x)=x-a\,x^{2n-1}(x^{2n}-1),
    \qquad
    g_{a,n}'(1)=1-2na .
\end{align*}
Thus the positive minimum is linearly stable for \(0<a<1/n\). We therefore write
\begin{align*}
    a=\frac{c}{n}
\end{align*}
and denote the rescaled map by
\begin{align*}
    g_{c,n}(x)
    :=
    x-\frac{c}{n}x^{2n-1}(x^{2n}-1).
\end{align*}
In this parameterization the edge of stability is always at \(c=1\) and we can consider a meaningful \(n \to \infty\) limit.

\subsection{The Quotient Map}

As in the quartic case, we remove the sign symmetry by passing to the quotient coordinate
\[
    u=x^{2n},
    \qquad u>0 .
\]
The induced map is
\[
    h_{c,n}(u)
    =
    u
    \left(
        1-\frac{c}{n}\,u^{(n-1)/n}(u-1)
    \right)^{2n}.
\]
For fixed \(c\), the factor in parentheses is a \(1+O(1/n)\) perturbation on compact subsets of \((0,\infty)\). Hence the quotient dynamics has the depth-independent limit
\[
    h_c(u)
    =
    u\exp\{2c\,u(1-u)\}.
\]
Thus the large-depth dynamics on the \(u=O(1)\) scale is governed by a Ricker-type map.

\begin{proposition}[Ricker-type depth limit]
For fixed \(c>0\), the quotient maps \(h_{c,n}\) converge uniformly on compact subsets of \((0,\infty)\) to
\[
    h_c(u)=u\exp\{2c\,u(1-u)\}.
\]
Consequently, the depth-\(n\) gradient dynamics with learning rate \(a=c/n\) has a universal large-depth limit in the coordinate \(u=x^{2n}\).
\end{proposition}

The derivation, including a uniform \(O(1/n)\) error bound on compact subsets of \((0,\infty)\), is given in Appendix~\ref{app:scaling-around-minimum}.
\subsection{Effective Gradient Map}

It is useful to pass once more to additive coordinates
\[
    w=\log u=2n\log |x| .
\]
The limiting map becomes
\[
    \Phi_c(w)
    =
    \log h_c(e^w)
    =
    w+2c\,e^w(1-e^w).
\]
Equivalently,
\[
    \Phi_c(w)
    =
    w-2c\,V'(w),
    \qquad
    V(w)=\frac12(e^w-1)^2 .
\]
Thus, in the depth limit, the additive-coordinate dynamics is gradient descent on the universal exponential potential \(V\), with effective learning rate \(2c\).

The minimum corresponds to
\[
    u=1,
    \qquad
    w=0 .
\]
Moreover,
\[
    h_c'(1)=\Phi_c'(0)=1-2c .
\]
Thus the limiting fixed point is locally stable precisely for
\[
    0<c<1 .
\]
In the large-depth scaling \(a=c/n\), the edge of stability is therefore \(c=1\), exactly matching the finite-depth edge \(a_1(n)=1/n\).
\subsection{Shape of the Limiting Map}

The quotient map \(h_c\) is a one-hump map on \((0,\infty)\). Indeed,
\[
    h_c'(u)
    =
    \exp\{2c\,u(1-u)\}
    \bigl(1+2c\,u(1-2u)\bigr),
\]
so its unique critical point is
\[
    u_{\mathrm{crit}}(c)
    =
    \frac{1+\sqrt{1+4/c}}{4}.
\]
The map increases on \((0,u_{\mathrm{crit}}(c))\), decreases on \((u_{\mathrm{crit}}(c),\infty)\), and satisfies
\[
    h_c(u)\to0
    \qquad
    \text{as}
    \qquad
    u\downarrow0
    \quad\text{or}\quad
    u\to\infty .
\]
Thus \(h_c((0,\infty))\) is bounded, and every positive orbit of the limiting map is bounded after one step, for every finite \(c>0\).

Near the fixed point, the additive map has the expansion
\[
    \Phi_c(w)
    =
    (1-2c)w
    -3c\,w^2
    -\frac73 c\,w^3
    +O(w^4).
\]
The linear term gives the local edge \(c=1\). 
The local normal form shows that the flip is supercritical: the fixed point loses stability by giving birth to a stable period-two orbit. 
Thus the edge \(c=1\) is not an escape threshold, but the entrance into organized post-edge oscillations. 
The later thresholds are not local; they depend on the full one-hump shape of \(h_c\). 
In particular, because \(h_c\) maps \((0,\infty)\) into a bounded interval for every finite \(c>0\), the large-depth limit retains post-edge oscillatory dynamics but has no finite-\(c\) divergence threshold.
\subsection{Schwarzian Structure and Critical Orbits}
\label{sec:ricker-schwarzian}

The one-hump shape suggests that the dynamics should be organized by the forward orbit of \(u_{\mathrm{crit}}(c)\). 
Here critical means critical for the one-dimensional map \(h_c\), not for the loss: it is the point at which the map changes from increasing to decreasing.

The reason this orbit is distinguished is Singer's theorem. 
For a \(C^3\) interval map with negative Schwarzian derivative, every attracting periodic orbit has an immediate basin containing either a critical point of the map or a boundary point. 
Thus attracting cycles cannot be hidden from the critical or boundary dynamics. 
In the present setting, iterating \(u_{\mathrm{crit}}(c)\) is therefore not merely a convenient way to draw the bifurcation diagram; it is the natural probe of the attracting periodic dynamics.

Recall that the Schwarzian derivative is
\[
    Sf
    =
    \frac{f'''}{f'}
    -
    \frac32\left(\frac{f''}{f'}\right)^2,
\]
defined away from critical points. 
For the Ricker limit, the required sign condition holds globally.

\begin{proposition}[Negative Schwarzian of the Ricker limit]
\label{prop:ricker-schwarzian}
For every \(c>0\), the limiting quotient map
\[
    h_c(u)=u\exp\{2c\,u(1-u)\}
\]
satisfies
\[
    Sh_c(u)<0
    \qquad
    \text{for all }u>0\text{ with }h_c'(u)\neq0 .
\]
\end{proposition}

The proof is given in Appendix~\ref{app:ricker-schwarzian}. 
Together with Singer's theorem, Proposition~\ref{prop:ricker-schwarzian} explains why the attracting periodic part of the limiting phase diagram can be generated by following the single critical orbit of \(u_{\mathrm{crit}}(c)\). 
In optimization language, it reduces the search for stable finite-step cycles from solving all possible periodic-orbit equations to tracking the most nonlinear point of the effective update map.

\begin{remark} 
The same critical orbit also organizes the visible chaotic attractor in the bifurcation diagram, but Singer's theorem itself is a statement about attracting periodic orbits. 
This is why the critical-orbit diagram displays both the period-doubling cascade and the stable periodic windows, while the theorem directly certifies only the attracting periodic part.
\end{remark}

\subsection{First Period Doubling}

At \(c=1\), the fixed point \(u=1\) undergoes a flip bifurcation and an attracting two-cycle is born. 
Write this cycle as
\[
    u_- \mapsto u_+ \mapsto u_-,
    \qquad
    u_-\neq u_+ .
\]
It is determined by
\[
    u_+=h_c(u_-),
    \qquad
    h_c(u_+)=u_- .
\]
The two-cycle loses stability when its multiplier reaches \(-1\):
\[
    h_c'(u_-)\,h_c'(u_+)=-1 .
\]
Numerically solving the combined system
\[
    u_+=h_c(u_-),
    \qquad
    h_c(u_+)=u_-,
    \qquad
    h_c'(u_-)\,h_c'(u_+)=-1
\]
gives
\[
    c_1
    =
    1.251645\ldots .
\]
Thus the limiting map has an attracting two-cycle for
\[
    1<c<c_1 .
\]

This is a depth-limit constant. 
Translated back to the original learning rate, it predicts the finite-depth scaling
\[
    a_1(n)=\frac1n,
    \qquad
    a_2(n)\sim\frac{c_1}{n},
\]
where \(a_1(n)\) is the fixed-point stability edge and \(a_2(n)\) is the two-cycle stability endpoint. 
The convergence of the finite-depth threshold \(n a_2(n)\to c_1\) is proved in Section~\ref{sec:middle-ground}.
\subsection{Cascade Without Divergence}

The first period doubling is followed by the principal period-doubling cascade. 
Let
\[
    c_0=1,
    \qquad
    c_1=1.251645\ldots .
\]
Here \(c_0\) is the flip of the fixed point, while \(c_1\) is the flip of the two-cycle. 
More generally, for \(j\ge0\), the stable \(2^j\)-cycle loses stability at \(c_j\), giving birth to a stable \(2^{j+1}\)-cycle. 
Numerically,
\[
    c_1<c_2<c_3<\cdots,
    \qquad
    c_j\uparrow c_\infty,
\]
with
\[
    c_\infty
    =
    1.326980\ldots .
\]
Thus, in the scaled parameter \(c\), the transition from fixed-point convergence to the post-cascade regime occupies the finite interval
\[
    1<c<c_\infty .
\]
In the original learning rate \(a=c/n\), this corresponds to an \(O(1/n)\) window.

For \(c>c_\infty\), the limiting map enters the standard post-cascade regime: chaotic invariant sets appear, interspersed with periodic windows. 
What is absent is a finite-\(c\) divergent phase. 
The finite-depth divergence mechanism is non-uniform in the quotient variable \(u\): it occurs on a large-\(u\) scale that is pushed out of the \(u=O(1)\) depth-limit dynamics. 
Section~\ref{sec:middle-ground} locates this escape mechanism on its own scale.
\subsection{Universal Phase Diagram}

In the limiting coordinate, the loss is, up to an irrelevant scale,
\[
    \bar\ell(u)=\frac12(u-1)^2 .
\]
Thus one-step loss decrease for every \(u>0\) is equivalent to
\[
    |h_c(u)-1|\le |u-1|
    \qquad
    \forall\,u>0 .
\]
Geometrically, this means that the updated point \(h_c(u)\) must remain between \(u\) and its reflection \(2-u\) across the minimizer \(u=1\). 
The update may overshoot the minimizer, but it must not overshoot beyond the reflected point.

The first loss-increase threshold occurs when the graph of \(h_c\) becomes tangent to the reflected graph \(u\mapsto 2-u\). 
Thus the boundary is determined by
\[
    h_c(u)+u=2,
    \qquad
    h_c'(u)=-1 .
\]
Solving these two equations numerically gives
\[
    u=1.615770\ldots,
    \qquad
    c_{\mathrm{cat}}
    =
    0.721813\ldots .
\]
Together with the flip threshold \(c=1\), the two-cycle endpoint
\[
    c_1=1.251645\ldots,
\]
and the accumulation point
\[
    c_\infty=1.326980\ldots,
\]
this gives the limiting phase diagram
\[
\begin{array}{ccl}
0<c\le c_{\mathrm{cat}}
&:&
\text{monotone convergence},\\[1mm]
c_{\mathrm{cat}}<c\le1
&:&
\text{nonmonotone, or catapult, convergence},\\[1mm]
1<c<c_\infty
&:&
\text{two-cycles and the period-doubling cascade},\\[1mm]
c>c_\infty
&:&
\text{post-cascade dynamics with chaotic intervals and periodic windows}.
\end{array}
\]
The divergent phase is absent as a finite-\(c\) phase of the depth-limit dynamics.

This phase diagram is universal in the following sense. 
After the scaling \(a=c/n\), all sufficiently deep polynomial chains are governed on the \(u=O(1)\) scale by the same Ricker-type map \(h_c\). 
The thresholds \(c_{\mathrm{cat}}\), \(1\), \(c_1\), and \(c_\infty\) are therefore depth-limit constants: depth changes the learning-rate scale, but not the limiting phase diagram.
\subsection{Numerical Phase Diagram}

Figure~\ref{fig:ricker-universal} illustrates the depth-limit dynamics. 
The left panel shows the limit map
\[
    h_c(u)=u\,e^{2cu(1-u)},
\]
a one-hump Ricker-type map with fixed point \(u=1\). 
As \(c\) increases, the graph steepens at the fixed point, which loses stability at the universal edge \(c=1\). 
The right panel shows the same dynamics in the additive coordinate
\[
    w=\log u,
    \qquad
    \Phi_c(w)=\log h_c(e^w)
             =w+2c e^w(1-e^w),
\]
where the universal period-doubling cascade is visible.

\begin{figure}[t]
    \centering
    \includegraphics[width=\linewidth]{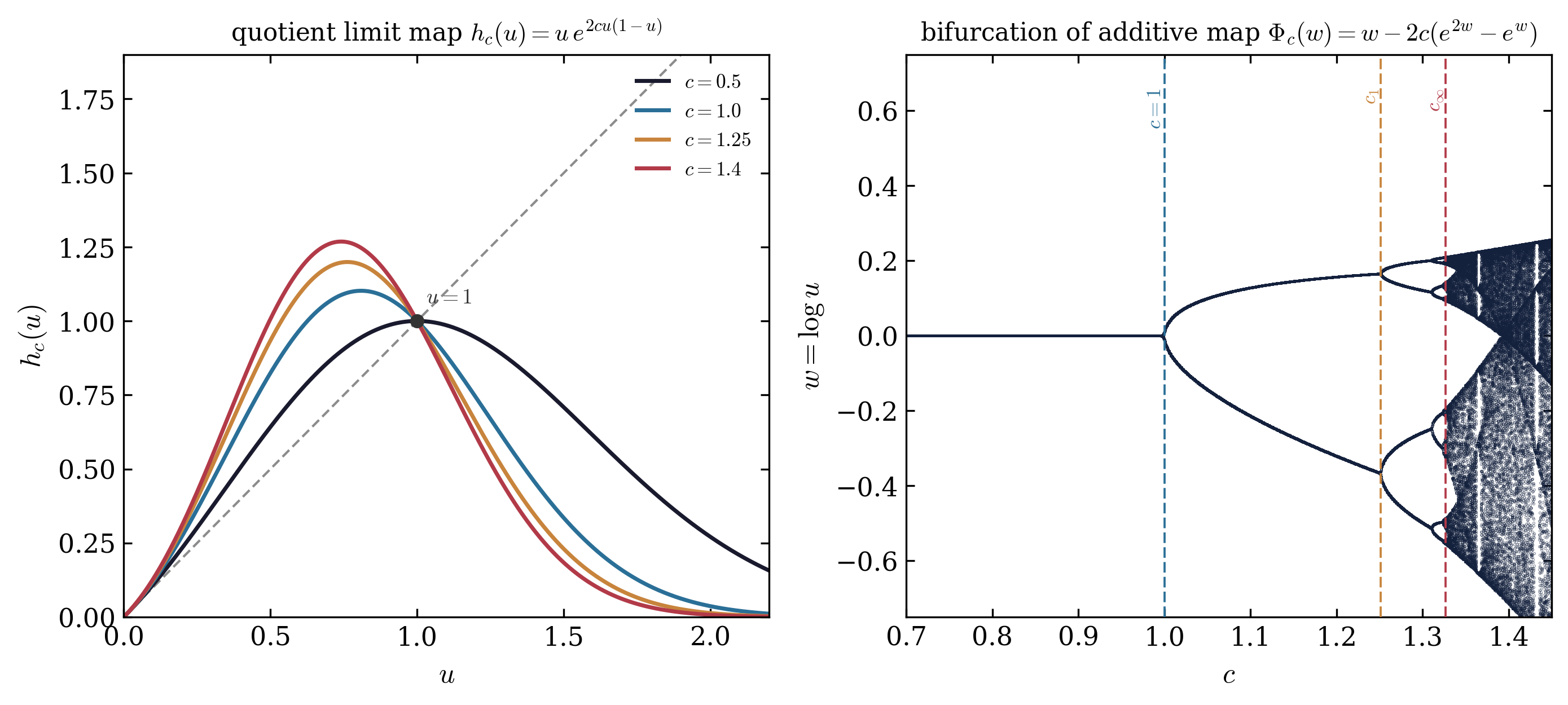}
    \caption{Universal depth-limit dynamics. 
    Left: the quotient limit map \(h_c(u)=u e^{2cu(1-u)}\) for representative values of \(c\), together with the identity line. 
    The fixed point \(u=1\) loses stability at \(c=1\). 
    Right: bifurcation diagram of the additive-coordinate map \(\Phi_c(w)=\log h_c(e^w)\), where \(w=\log u\). 
    The dashed lines mark the monotone-loss threshold \(c_{\mathrm{cat}}\approx0.7218\), the edge \(c=1\), the two-cycle endpoint \(c_1\approx1.2516\), and the accumulation point \(c_\infty\approx1.32698\).}
    \label{fig:ricker-universal}
\end{figure}
\subsection{History and Related Work}

The map \(h_c\) is a Ricker-type exponential one-hump map. 
Its origin lies in population dynamics \cite{ricker1954stock}, while its period-doubling route belongs to the classical theory of one-dimensional maps \cite{may1974biological}. 
Here it is not introduced phenomenologically; it arises as the exact large-depth limit of gradient descent on the depth-\(n\) scalar chain.

The edge-of-stability literature for deep linear networks has observed period-doubling routes to chaos empirically and analyzed the first post-edge two-cycle at finite depth
\cite{Ghosh2025DeepMatrixFactorizationEOS,ghosh2025dln}. 
The quotient map \(h_{c,n}\), its universal limit \(h_c\), and the constants \(c_{\mathrm{cat}},c_1,c_\infty\) give the corresponding depth-independent picture. 
Thus the cubic map is the exactly solvable base case, while the Ricker-type map is the universal depth-limit normal form.
\section{The Middle Ground}
\label{sec:middle-ground}

The two endpoints are explicit. At depth \(n=1\), the quotient dynamics is the cubic map analyzed above. At depth \(n=\infty\), the rescaled quotient maps converge to the Ricker-type limit
\[
    h_c(u)=u\exp\{2c\,u(1-u)\}.
\]
For finite \(2\le n<\infty\), closed-form solutions are no longer available, but the same quotient coordinate remains the right object. We therefore study
\[
    h_{c,n}(u)
    =
    u
    \left(
        1-\frac{c}{n}\,u^{(n-1)/n}(u-1)
    \right)^{2n},
    \qquad
    u=x^{2n},
\]
as an interpolation between the cubic base case and the Ricker limit.

The interpolation is structural in three senses. First, the depth sequence admits an analytic embedding near the limit, so nondegenerate thresholds vary smoothly with \(1/n\). Second, the quotient maps retain a one-hump negative-Schwarzian structure, so attracting cycles are organized by a single critical orbit. Third, the finite-depth escape mechanism is confined to a separate outer scale and disappears from the \(u=O(1)\) depth-limit dynamics.

\subsection{Finite-Depth Bifurcations}
\label{sec:finite-depth-numerics}

The fixed point \(u=1\) satisfies
\[
    h_{c,n}'(1)=1-2c
\]
exactly, for every depth. Thus the first flip occurs at
\[
    c=1
\]
for all \(n\). Beyond this point, the attracting object is no longer a fixed point but a periodic orbit.

For a period-\(m\) orbit
\[
    u_0\mapsto u_1\mapsto\cdots\mapsto u_{m-1}\mapsto u_0,
\]
the multiplier is
\[
    \mu_m(c,n)
    =
    (h_{c,n}^m)'(u_0)
    =
    \prod_{j=0}^{m-1} h_{c,n}'(u_j).
\]
A flip bifurcation of this orbit occurs when
\[
    h_{c,n}^m(u_0)=u_0,
    \qquad
    (h_{c,n}^m)'(u_0)=-1,
\]
with \(u_0\) not belonging to a lower-period orbit. For \(m=1\), this recovers the edge \(c=1\); for \(m=2\), it locates the loss of stability of the two-cycle; and the same condition continues along the period-doubling cascade.

The Lyapunov exponent gives the complementary diagnostic for the global phase portrait:
\[
    \lambda_n(c)
    =
    \lim_{T\to\infty}
    \frac1T
    \sum_{k=0}^{T-1}
    \log\left|h_{c,n}'(u_k)\right|.
\]
On an attracting \(m\)-cycle,
\[
    \lambda_n(c)=\frac1m\log|\mu_m(c,n)|<0.
\]
At a neutral flip the corresponding exponent vanishes, while positive values indicate chaotic sensitivity on the attracting set.

Numerically, the finite-depth picture is consistent across \(n\): the fixed point flips at \(c=1\), the two-cycle loses stability near the limiting value \(c_1\), the principal cascade accumulates near \(c_\infty\), and the post-cascade region contains chaotic intervals interrupted by periodic windows. The following subsections explain why this picture is not an accident of numerics.

\subsection{Analytic Embedding of the Depth Sequence}
\label{sec:analytic-interpolation}

Set
\[
    \varepsilon=\frac1n .
\]
The finite-depth maps are defined only on the discrete sequence
\[
    \varepsilon=1,\frac12,\frac13,\ldots .
\]
On the principal branch, however, this sequence admits an analytic embedding. For \(u>0\) and
\[
    1-\varepsilon c\,u^{1-\varepsilon}(u-1)>0,
\]
define
\[
    h^{(\varepsilon)}_c(u)
    =
    u
    \exp\left\{
        \frac{2}{\varepsilon}
        \log\bigl(1-\varepsilon c\,u^{1-\varepsilon}(u-1)\bigr)
    \right\},
    \qquad
    \varepsilon>0,
\]
and set
\[
    h^{(0)}_c(u)=h_c(u)=u\exp\{2c\,u(1-u)\}.
\]
For \(\varepsilon=1/n\), this is exactly the finite-depth quotient map \(h_{c,n}\). For other \(\varepsilon\), it is an analytic interpolation, not a genuine network depth.

The apparent singularity at \(\varepsilon=0\) is removable on compact subsets of the principal branch. Indeed, if
\[
    g_\varepsilon(u)=c\,u^{1-\varepsilon}(u-1),
\]
then
\[
    \frac{2}{\varepsilon}\log(1-\varepsilon g_\varepsilon)
    =
    -2g_\varepsilon-\varepsilon g_\varepsilon^2-\frac23\varepsilon^2 g_\varepsilon^3-\cdots ,
\]
and \(g_\varepsilon\) depends analytically on \(\varepsilon\) for \(u>0\).

\begin{proposition}[Analytic embedding near the depth limit]
\label{prop:analytic-interpolation}
Let \(F(y,c,\varepsilon)=0\) be a finite system of equations built from \(h^{(\varepsilon)}_c\), its iterates, and its derivatives, evaluated at finitely many points on the principal branch. Suppose that
\[
    F(y_0,c_0,0)=0
\]
and that the Jacobian with respect to \((y,c)\) is invertible at \((y_0,c_0,0)\). Then there is an analytic branch
\[
    \varepsilon\mapsto (y_\star(\varepsilon),C_\star(\varepsilon))
\]
with
\[
    (y_\star(0),C_\star(0))=(y_0,c_0).
\]
Along the actual depth sequence \(\varepsilon=1/n\), the corresponding scaled finite-depth threshold satisfies
\[
    C_\star(1/n)
    =
    C_\star(0)
    +
    \frac{\kappa}{n}
    +
    O\!\left(\frac1{n^2}\right).
\]
\end{proposition}

This is the implicit function theorem applied to the analytic embedding. It applies to any nondegenerate limiting threshold: flip bifurcations, tangencies such as the monotone-loss boundary, superstable parameters, saddle-node window boundaries, and hyperbolic periodic orbits.

\begin{figure}[t]
    \centering
    \includegraphics[width=0.85\linewidth]{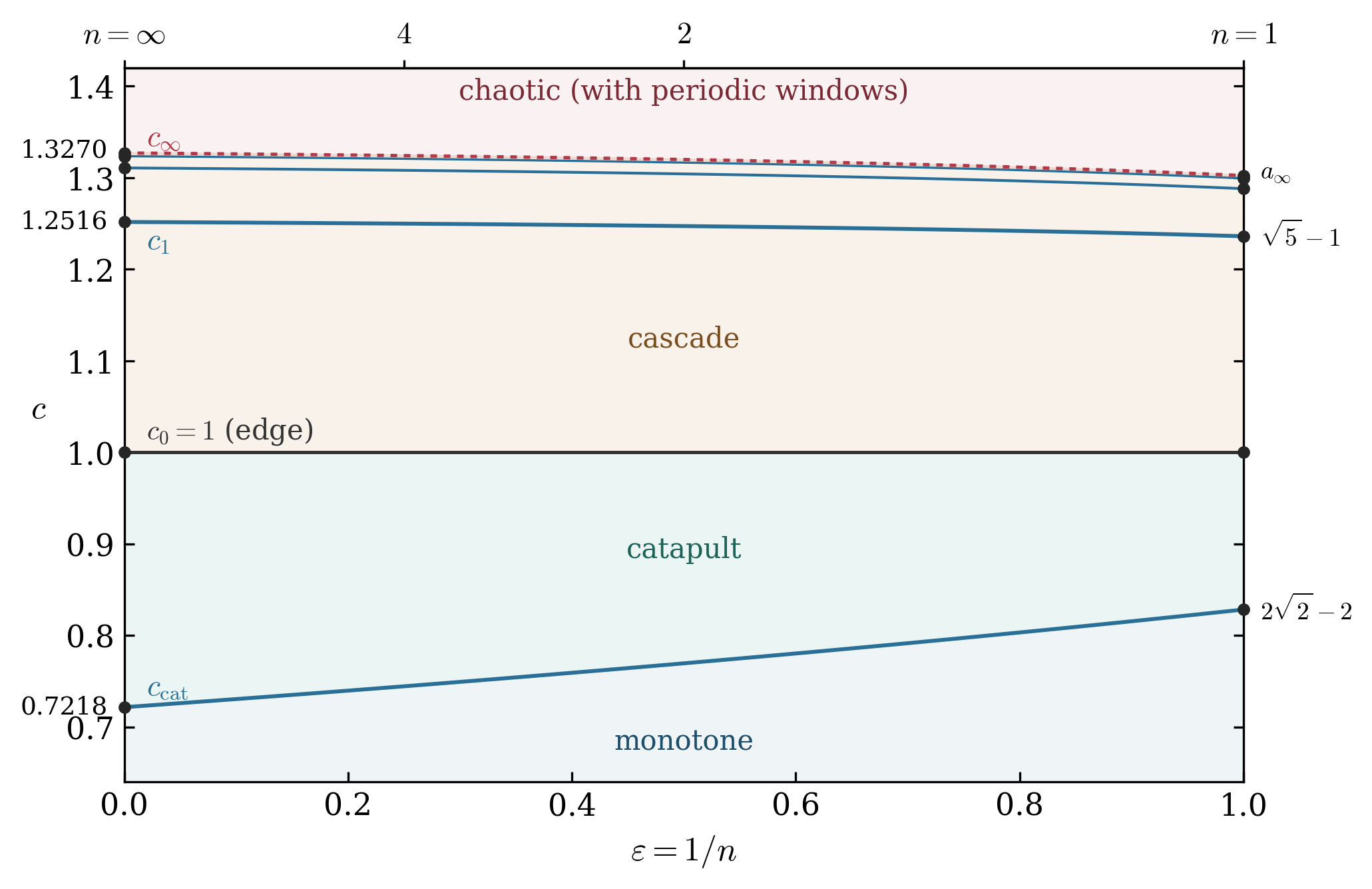}
    \caption{The middle ground as an analytic family. Phase boundaries of the quotient maps in the \((\varepsilon,c)\)-plane, \(\varepsilon=1/n\). The solid curves --- the catapult tangency \(c_{\mathrm{cat}}\), the edge \(c_0=1\), and the flip thresholds \(c_1,c_2,c_3\) --- are analytic branches joining the cubic constants at \(n=1\) to the Ricker constants at \(n=\infty\), as predicted by Proposition~\ref{prop:analytic-interpolation}. The accumulation boundary \(c_\infty\) is computed numerically; its analytic continuation is not asserted here.}
    \label{fig:middle-ground-phase}
\end{figure}

For example, the loss of stability of the two-cycle has limiting value
\[
    C_2(0)=1.251645\ldots .
\]
Proposition~\ref{prop:analytic-interpolation} gives the large-depth expansion
\[
    C_2(1/n)
    =
    C_2(0)
    +
    O\!\left(\frac1n\right),
\]
and the actual learning-rate threshold is
\[
    a_2(n)=\frac{C_2(1/n)}{n}.
\]
Numerically, this branch continues all the way to the cubic endpoint,
\[
    C_2(1)=\sqrt5-1,
\]
and the expansion begins as
\[
    C_2(1/n)
    =
    C_2(0)
    -
    \frac{0.00593\ldots}{n}
    +
    O\!\left(\frac1{n^2}\right).
\]
Likewise, the monotone-loss tangency continues from
\[
    C_{\mathrm{cat}}(0)=0.721813\ldots
\]
to the cubic value
\[
    C_{\mathrm{cat}}(1)=2\sqrt2-2 .
\]
For each fixed level \(j\), the same mechanism applies to the \(j\)-th flip in the period-doubling cascade. What it does not by itself prove is uniformity as \(j\to\infty\), and hence it does not by itself establish analytic continuation of the accumulation curve \(c_\infty\).

\subsection{Critical Orbits and Singer's Theorem}
\label{sec:critical-orbits-singer}

The critical orbit is the natural probe of the attracting dynamics. For the finite-depth quotient map, write
\[
    h_{c,n}(u)
    =
    u\,b_{c,n}(u)^{2n},
    \qquad
    b_{c,n}(u)
    =
    1-\frac{c}{n}u^{(n-1)/n}(u-1).
\]
The principal branch is the interval where
\[
    b_{c,n}(u)>0.
\]
Let \(u_{\mathrm{out}}(c,n)\) be the first positive zero of \(b_{c,n}\):
\[
    b_{c,n}(u_{\mathrm{out}})=0.
\]
Equivalently,
\[
    \frac{c}{n}u_{\mathrm{out}}^{(n-1)/n}(u_{\mathrm{out}}-1)=1,
    \qquad
    u_{\mathrm{out}}(c,n)\sim \sqrt{\frac{n}{c}} .
\]
This is the outer endpoint of the principal branch.

On this branch,
\[
    h_{c,n}'(u)
    =
    b_{c,n}(u)^{2n-1}
    \left[
        1+\frac{c}{n}u^{(n-1)/n}
        \bigl((2n-1)-(4n-1)u\bigr)
    \right].
\]
Thus the interior critical points are determined by
\[
    \varphi_{c,n}(u)=1,
\]
where
\[
    \varphi_{c,n}(u)
    :=
    \frac{c}{n}u^{(n-1)/n}
    \bigl((4n-1)u-(2n-1)\bigr).
\]
The function \(\varphi_{c,n}\) is negative up to
\[
    u=\frac{2n-1}{4n-1},
\]
and then strictly increases to \(+\infty\). Hence it crosses the level \(1\) exactly once. Moreover, this crossing lies strictly inside the principal branch: at \(u_{\mathrm{out}}\),
\[
    \varphi_{c,n}(u_{\mathrm{out}})
    >
    \frac{c}{n}u_{\mathrm{out}}^{(n-1)/n}(u_{\mathrm{out}}-1)
    =
    1 .
\]
Finally, the critical point is nondegenerate. On the principal branch,
\[
    h_{c,n}'(u)
    =
    b_{c,n}(u)^{2n-1}
    \bigl(1-\varphi_{c,n}(u)\bigr),
\]
and \(\varphi_{c,n}\) crosses the level \(1\) strictly increasingly at the critical point. Hence \(h_{c,n}'\) has a simple zero there, so the critical point is a quadratic maximum.

\begin{proposition}[Unique critical point on the principal branch]
\label{prop:unique-critical-point-finite}
For every \(c>0\) and every \(n\ge1\), the finite-depth quotient map \(h_{c,n}\) has exactly one critical point on the principal branch \(b_{c,n}(u)>0\). It is a nondegenerate maximum of the one-hump map on this branch and is characterized by
\[
    \frac{c}{n}u^{(n-1)/n}
    \bigl((4n-1)u-(2n-1)\bigr)
    =
    1 .
\]
\end{proposition}

In the limit \(n\to\infty\), this critical point converges to the unique positive critical point of the Ricker map,
\[
    u_{\mathrm{crit}}(c)
    =
    \frac{1+\sqrt{1+4/c}}{4}.
\]
Thus the cubic map, the finite-depth quotient maps, and the Ricker limit are all organized by the same object: the forward orbit of a single interior critical point.

Recall that for a \(C^3\) interval map \(f\), the Schwarzian derivative is
\[
    Sf
    =
    \frac{f'''}{f'}
    -
    \frac32\left(\frac{f''}{f'}\right)^2,
\]
defined away from critical points. The relevance of the Schwarzian sign is Singer's theorem: for a negative-Schwarzian interval map, every attracting cycle has an immediate basin containing either a critical point or a boundary point. Thus, away from boundary attractors, stable finite-step cycles are detected by critical orbits. In the present setting, the orbit of the unique critical point is therefore not merely a plotting device; it is the canonical probe of the attracting dynamics.

\subsection{Negative Schwarzian at Finite Depth}
\label{sec:finite-schwarzian}

The remaining structural question is whether the finite-depth quotient maps lie in the same negative-Schwarzian class as the cubic map and the Ricker limit. This is a coordinate-dependent question: the Schwarzian derivative is not invariant under arbitrary smooth coordinate changes. The original \(x\)-coordinate becomes misleading for \(n\ge2\) near the flat center. The quotient coordinate \(u=x^{2n}\), by contrast, carries the loss and radial dynamics.

For the analytically embedded family, the negative-Schwarzian property holds uniformly on the principal branch.

\begin{proposition}[Negative Schwarzian on the principal branch]
\label{prop:schwarzian-all-n}
For every \(0\le\varepsilon\le1\) and every \(c>0\), the map \(h^{(\varepsilon)}_c\) satisfies
\[
    Sh^{(\varepsilon)}_c(u)<0
\]
on the principal branch, away from critical points. In particular, this holds for the cubic map, for every finite depth \(n\), and for the Ricker limit.
\end{proposition}

The proof is given in Appendix~\ref{app:finite-depth-schwarzian}. Combining Proposition~\ref{prop:schwarzian-all-n} with Singer's theorem gives the structural consequence.

\begin{corollary}[Critical orbit controls attracting cycles]
\label{cor:singer-all-n}
At every finite depth, any attracting cycle of \(h_{c,n}\) whose immediate basin lies on the principal branch attracts either the unique interior critical point or a boundary point. In the post-edge regimes considered here, the relevant attractors are interior, so the critical orbit organizes the bifurcation diagram uniformly in \(n\).
\end{corollary}

The depth dependence changes the shape of the one-hump map and shifts the thresholds, but it does not introduce hidden attracting cycles on the principal branch.

\subsection{The Outer Scale}
\label{sec:outer-scale}

The universal dynamics lives on the \(u=O(1)\) scale. The finite-depth phenomenon absent from the Ricker limit is escape, and it occurs on a second scale. Let
\[
    u=s\sqrt{\frac{n}{c}},
    \qquad
    s=O(1).
\]
Then
\[
    b_{c,n}(u)
    =
    1-\frac{c}{n}u^{(n-1)/n}(u-1)
    =
    1-s^2+o(1).
\]
Therefore the quotient update
\[
    h_{c,n}(u)=u\,b_{c,n}(u)^{2n}
\]
has a sharp dichotomy. If
\[
    |1-s^2|<1,
    \qquad\text{equivalently}\qquad
    0<s<\sqrt2,
\]
then the factor \(b_{c,n}(u)^{2n}\) collapses exponentially and the iterate is sent back toward the inner window. If
\[
    s>\sqrt2,
\]
then \(|b_{c,n}(u)|>1\), and the iterate escapes outward in one step.

The escape boundary is determined by
\[
    b_{c,n}(u_{\mathrm{esc}})=-1,
\]
or equivalently
\[
    \frac{c}{n}u_{\mathrm{esc}}^{(n-1)/n}(u_{\mathrm{esc}}-1)=2.
\]
Hence
\[
    u_{\mathrm{esc}}(c,n)
    =
    \sqrt{\frac{2n}{c}}
    \left(1+O(n^{-1/2})\right).
\]
In the logarithmic coordinate \(z=\log u\), this boundary is
\[
    z_{\mathrm{esc}}(c,n)
    =
    \frac12\log\frac{2n}{c}
    +
    O(n^{-1/2}).
\]
Thus escape is pushed to infinity in the depth-limit coordinate \(z\). In the original parameter \(x\), however,
\[
    x_{\mathrm{esc}}
    =
    u_{\mathrm{esc}}^{1/(2n)}
    =
    \exp\left(\frac{z_{\mathrm{esc}}}{2n}\right)
    =
    1+\frac{1}{4n}\log\frac{2n}{c}
    +
    o\!\left(\frac1n\right).
\]
So the finite-depth escape cliff is far away on the quotient scale, but only logarithmically farther from the minimum in the original \(x\)-coordinate.

Between the universal \(u=O(1)\) window and the escape boundary there is no new finite-step dynamics of the same kind as the Ricker cascade. The outer region acts as a one-step sorter: points below the escape boundary are thrown back inward, while points above it escape.

\begin{proposition}[Outer escape scale]
\label{prop:outer-escape-scale}
At finite depth, escape occurs on the scale
\[
    u\asymp\sqrt{\frac{n}{c}}.
\]
More precisely, the one-step escape boundary satisfies
\[
    u_{\mathrm{esc}}(c,n)
    =
    \sqrt{\frac{2n}{c}}
    \left(1+O(n^{-1/2})\right).
\]
Equivalently,
\[
    x_{\mathrm{esc}}
    =
    1+\frac{1}{4n}\log\frac{2n}{c}
    +
    o\!\left(\frac1n\right).
\]
Thus escape is pushed to infinity in the \(u=O(1)\) depth-limit dynamics, even though it remains close to the minimum in the original \(x\)-coordinate.
\end{proposition}

The middle ground therefore contains no additional organizing principle between the cubic base case and the Ricker limit. Its role is to show that the whole depth family is one coherent object. The thresholds move smoothly with reciprocal depth, the attracting dynamics is controlled by a single critical orbit on the principal branch, and finite-depth escape lives on a separate outer scale. Depth changes the scale and the location of thresholds, but not the dynamical mechanism.

\subsection{History and Related Work}
\label{sec:classical-context}

The structure assembled in this section --- one-hump maps, critical orbits, negative Schwarzian derivative --- is the core toolkit of one-dimensional dynamics. The Schwarzian derivative entered interval dynamics through Singer~\cite{singer1978}, who proved that for a map with negative Schwarzian derivative every attracting cycle attracts a critical point or a boundary point. This is what makes the critical orbit the canonical probe and bounds the number of coexisting attractors by the number of critical points. The analytical reason the condition is powerful is distortion control: negative Schwarzian derivative yields Koebe-type bounds on cross-ratios, the engine of the modern theory as developed in de Melo and van Strien~\cite{demelo1993}. Guckenheimer~\cite{guckenheimer1979} and Misiurewicz~\cite{misiurewicz1981} built the structure theory of unimodal maps on this hypothesis, including absence of wandering intervals, sensitivity, and the classification of attractors. Kozlovski~\cite{kozlovski2000} later showed that for smooth maps with nondegenerate critical points the condition can often be dispensed with; in our family it is not an assumption but a theorem, Proposition~\ref{prop:schwarzian-all-n}, valid uniformly in depth.

The period-doubling cascade and its universal accumulation belong to Feigenbaum~\cite{feigenbaum1978} and Coullet--Tresser~\cite{coullet1978}, with the renormalization picture made rigorous by Lanford~\cite{lanford1982}, Sullivan, McMullen~\cite{mcmullen1996}, and Lyubich~\cite{lyubich1999}. Since the critical point of \(h_{c,n}\) is quadratic at every depth, the family lies in the standard universality class. The numerically observed gap ratios are consistent with convergence to the Feigenbaum constant. Finally, the Ricker map itself is a classical object of population dynamics~\cite{ricker1954stock}, popularized as a canonical one-hump family by May~\cite{may1976simple}; here it is not a model choice but the exact large-depth limit of gradient descent.
\part{Learning Dynamics}
\label{part:learning-dynamics}
\section{Two-Factor Model}
\label{sec:beyond-scalar-chain}

\subsection{Depth-Two Neural Chain}

The scalar chain isolates an invariant symmetric trajectory of the deep linear model. 
This made the preceding analysis possible, but it also removed the factorization directions present in the full parameter space. 
We now return to the smallest model in which these directions are visible: the rank-one, two-factor loss
\[
    \ell(x,y)=\frac12(xy-1)^2,
    \qquad
    r=xy-1 .
\]
Gradient descent with step size \(a\) gives
\begin{align}
    x^+=x-a r y,
    \qquad
    y^+=y-a r x .
    \label{eq:two-factor-gd}
\end{align}
The diagonal \(x=y\) is invariant. 
Writing \(x=y=s\), the two-factor dynamics restricts exactly to the scalar quartic map:
\begin{align}
    s^+
    =
    s-a\,s(s^2-1)
    =
    g_a(s).
    \label{eq:two-factor-diagonal-map}
\end{align}
Thus the quartic map studied above is not merely analogous to the two-factor model; it is the exact balanced restriction of the full two-dimensional dynamics. 
The purpose of this section is to identify which conclusions of the scalar analysis persist away from the diagonal, and which effects are created by the additional factorization direction.
\subsection{Finite-Step Balancing}
\label{sec:balancing-versus-sharpening}

The natural coordinates for the off-diagonal dynamics are the residual
\[
    r=xy-1
\]
and the imbalance
\[
    \Delta=x^2-y^2 .
\]
Gradient flow conserves \(\Delta\). 
Finite-step gradient descent does not, and the resulting drift toward balance is the main effect in this subsection.

We shall also use
\[
    S=x^2+y^2 .
\]
Since \(xy=1+r\), this can be written as
\[
    S
    =
    \sqrt{\Delta^2+4(1+r)^2}.
\]
The residual and imbalance form a closed two-dimensional system:
\begin{align}
    r^+
    &=
    r\Bigl(1-aS+a^2r(1+r)\Bigr),
    \notag\\
    \Delta^+
    &=
    (1-a^2r^2)\Delta.
    \label{eq:imbalance-update}
\end{align}
Thus residual excursions contract \(|\Delta|\) whenever
\[
    0<|ar|<\sqrt2,
\]
and preserve the sign of \(\Delta\) whenever
\[
    |ar|<1 .
\]
This is the basic finite-step balancing mechanism: motion in the residual direction produces contraction in the factorization direction.

On the solution valley \(r=0\), the Hessian has one zero tangential eigenvalue and one nonzero normal eigenvalue
\[
    \lambda(\Delta)
    =
    \sqrt{\Delta^2+4}.
\]
This is the normal sharpness of the valley point. 
The valley point is normally stable for gradient descent exactly when
\[
    |1-a\lambda(\Delta)|<1.
\]
Thus, for \(0<a<1\), the normally stable arc of the valley is
\[
    |\Delta|<\Delta^*(a),
    \qquad
    \Delta^*(a)
    =
    \frac{2}{a}\sqrt{1-a^2}.
\]
At \(a=1\), the balanced point \(\Delta=0\) is at the flip threshold. 
For \(a>1\), no point of the solution valley is normally stable.

The same mechanism also determines the small-step endpoint selected on the solution valley. 
For \(\delta>0\) and \(r_0>0\), set
\[
    J(\delta,r_0)
    =
    \int_{1}^{1+r_0}
    \frac{w-1}{\sqrt{\delta^2+4w^2}}\,dw .
\]
Equivalently,
\[
    J(\delta,r_0)
    =
    \left[
        \frac14\sqrt{\delta^2+4w^2}
        -
        \frac12\operatorname{asinh}\!\left(\frac{2w}{\delta}\right)
    \right]_{w=1}^{w=1+r_0}.
\]

\begin{proposition}[Finite-step balancing]
\label{prop:finite-step-balancing}
For the two-factor gradient descent map \eqref{eq:two-factor-gd}, the following hold.

\emph{Exact imbalance identity.}
The imbalance evolves according to
\[
    \Delta^+=(1-a^2r^2)\Delta .
\]
Thus residual excursions contract \(|\Delta|\) whenever \(0<|ar|<\sqrt2\), and preserve \(\operatorname{sign}(\Delta)\) whenever \(|ar|<1\).

\emph{Stable arc.}
The solution valley \(r=0\) is normally stable exactly for
\[
    |\Delta|<\Delta^*(a)
\]
when \(0<a<1\). 
At \(a=1\), the balanced point is at the flip threshold. 
For \(a>1\), no point of the valley is normally stable.

\emph{Small-step endpoint.}
Let \(r_0>0\) and \(\Delta_0\ne0\). 
If the small-step trajectory converges to the solution valley, then
\begin{align}
    \log\frac{|\Delta_\infty|}{|\Delta_0|}
    =
    -aJ(|\Delta_0|,r_0)+O(a^2).
    \label{eq:smallstep-balancing-main}
\end{align}
Since \(J(|\Delta_0|,r_0)>0\), the selected endpoint has smaller imbalance than the corresponding gradient-flow endpoint. 
Because the nonzero Hessian eigenvalue along the valley is
\[
    \lambda(\Delta)=\sqrt{\Delta^2+4},
\]
the selected endpoint also has smaller normal curvature.
\end{proposition}

The proof is given in Appendix~\ref{app:finite-step-balancing-proof}.
\subsection{Residual Oscillation Selects the Minimum}
\label{sec:balancing-oscillation}

Proposition~\ref{prop:finite-step-balancing} shows that balancing is driven by residual motion. 
On the solution valley \(r=0\), the imbalance \(\Delta\) is unchanged. 
Away from the valley, residual excursions contract \(|\Delta|\). 
A trajectory started just beyond the normally stable arc therefore develops a residual burst: the residual grows and oscillates, each oscillation contracts the imbalance, the normal sharpness decreases, and the burst dies out once the trajectory has moved back into the stable part of the valley. 
The burst is self-extinguishing: it is created by instability, but removes the instability that created it.

We now compute the selected endpoint of such a burst near the edge. 
Along the solution valley, the normal sharpness is
\[
    \lambda(\Delta)=\sqrt{\Delta^2+4}.
\]
For fixed \(0<a<1\), define the signed distance to the stability boundary by
\[
    \varepsilon=a\lambda(\Delta)-2.
\]
Thus \(\varepsilon=0\) is the stability boundary, \(\varepsilon>0\) is the unstable side, and \(\varepsilon<0\) is the stable side. 
We also write
\[
    \mu_a
    :=
    \frac{(\Delta^*(a))^2}{2}
    =
    \frac{2(1-a^2)}{a^2}.
\]
The parameter \(\mu_a\) is fixed by the learning rate; it measures the squared width of the normally stable arc and satisfies \(\mu_a\downarrow0\) as \(a\uparrow1\).

Consider the near-edge scaling
\[
    \varepsilon=O(\mu_a),
    \qquad
    r=O(\sqrt{\mu_a}).
\]
Averaging over the fast sign oscillation of \(r\) gives a slow planar system for the residual envelope and the remaining instability. 
With
\[
    \alpha=2a^4r^2,
\]
the leading-order slow system is
\begin{align}
    \dot\alpha=4\alpha(\varepsilon-\alpha),
    \qquad
    \dot\varepsilon=-\alpha(\mu_a+2\varepsilon).
    \label{eq:slow-system-main}
\end{align}
The derivation from the exact map is given in Appendix~\ref{app:selection-law}.

The slow system is integrable. 
Eliminating time gives
\[
    \frac{d\alpha}{d\varepsilon}
    =
    -\frac{4(\varepsilon-\alpha)}{\mu_a+2\varepsilon},
\]
and hence
\begin{align}
    \alpha
    =
    \frac{\mu_a}{2}
    +
    2\varepsilon
    +
    C(\mu_a+2\varepsilon)^2,
    \label{eq:burst-invariant}
\end{align}
where \(C\) is constant along the averaged burst. 
A burst starts and ends with vanishing residual envelope, \(\alpha=0\). 
Therefore the initial excess \(\varepsilon_0>0\) and the selected endpoint \(\varepsilon_\infty<0\) lie on the same level set of
\begin{align}
    g_{\mu_a}(\varepsilon)
    =
    \frac{\mu_a+4\varepsilon}{(\mu_a+2\varepsilon)^2}.
    \label{eq:gmu-definition}
\end{align}
The selection law is
\begin{align}
    g_{\mu_a}(\varepsilon_\infty)
    =
    g_{\mu_a}(\varepsilon_0).
    \label{eq:selection-main}
\end{align}
For \(\varepsilon_0>0\), this equation has a unique negative solution
\[
    \varepsilon_\infty\in(-\mu_a/4,0).
\]
Indeed, \(g_{\mu_a}\) increases from \(0\) to \(1/\mu_a\) on \((-\mu_a/4,0)\), and decreases from \(1/\mu_a\) to \(0\) on \((0,\infty)\). 
Thus the burst maps the unstable side of the edge to a definite point inside the stable arc.

Converting back to the valley coordinate gives
\[
    \lambda_\infty=\frac{2+\varepsilon_\infty}{a},
    \qquad
    |\Delta_\infty|=\sqrt{\lambda_\infty^2-4}.
\]
Thus the burst does not merely return the trajectory to the stability boundary. 
It overshoots into the stable arc and selects a point with smaller imbalance and smaller normal sharpness.

In the universal coordinate \(X=\varepsilon/\mu_a\), the leading-order law becomes parameter-free:
\begin{align}
    G(X_\infty)=G(X_0),
    \qquad
    G(X)=\frac{1+4X}{(1+2X)^2}.
    \label{eq:universal-selection}
\end{align}
For \(\varepsilon_0/\mu_a\) fixed as \(a\uparrow1\), the endpoint error is \(O(\mu_a^2)\). 
Thus minimum selection, measured in units of the stable-arc width, is universal to leading order near the edge.

Several corrections refine this leading law. 
At fixed \(a<1\), the small-\(\varepsilon_0\) reflection law has the expansion
\begin{align}
    \varepsilon_\infty
    =
    -\varepsilon_0
    +
    c(a)\varepsilon_0^2
    +
    O(\varepsilon_0^3),
    \qquad
    c(a)
    =
    \frac{2(4a^2-1)}{3(1-a^2)}.
    \label{eq:c-of-a}
\end{align}
In particular, \(c(1/2)=0\): at \(a=1/2\), the burst reflects exactly to second order. 
A finite initial residual envelope \(\alpha_0=2a^4r_0^2\) shifts the invariant constant in \eqref{eq:burst-invariant} by \(O(r_0^2)\), and hence changes the selected endpoint only at second order in the seed size. 
Both corrections are derived in Appendix~\ref{app:selection-law}.

The burst can flatten only down to the balanced point. 
At balance,
\begin{align}
    \varepsilon_{\rm bal}=2a-2.
    \label{eq:balanced-excess}
\end{align}
The selection law has the lower endpoint floor \(-\mu_a/4\). 
For
\[
    \frac{1+\sqrt{17}}{8}<a<1,
\]
one has
\[
    \varepsilon_{\rm bal}<-\frac{\mu_a}{4},
\]
so even an arbitrarily deep burst extinguishes before reaching the balanced point. 
For
\[
    a<\frac{1+\sqrt{17}}{8},
\]
the balanced point lies inside the reachable range, and sufficiently large initial excess can drive the trajectory all the way to balance. 
The boundary is determined by
\[
    g_{\mu_a}(\varepsilon_0)=g_{\mu_a}(\varepsilon_{\rm bal}).
\]
Beyond this point no flatter valley minimum exists, and the valley-recapture mechanism may fail. 
This is the catapult boundary: it is caused not by small learning rate alone, but by starting sufficiently far beyond the stable arc for that learning rate.

\begin{figure}[ht]
    \centering
    \includegraphics[width=\linewidth]{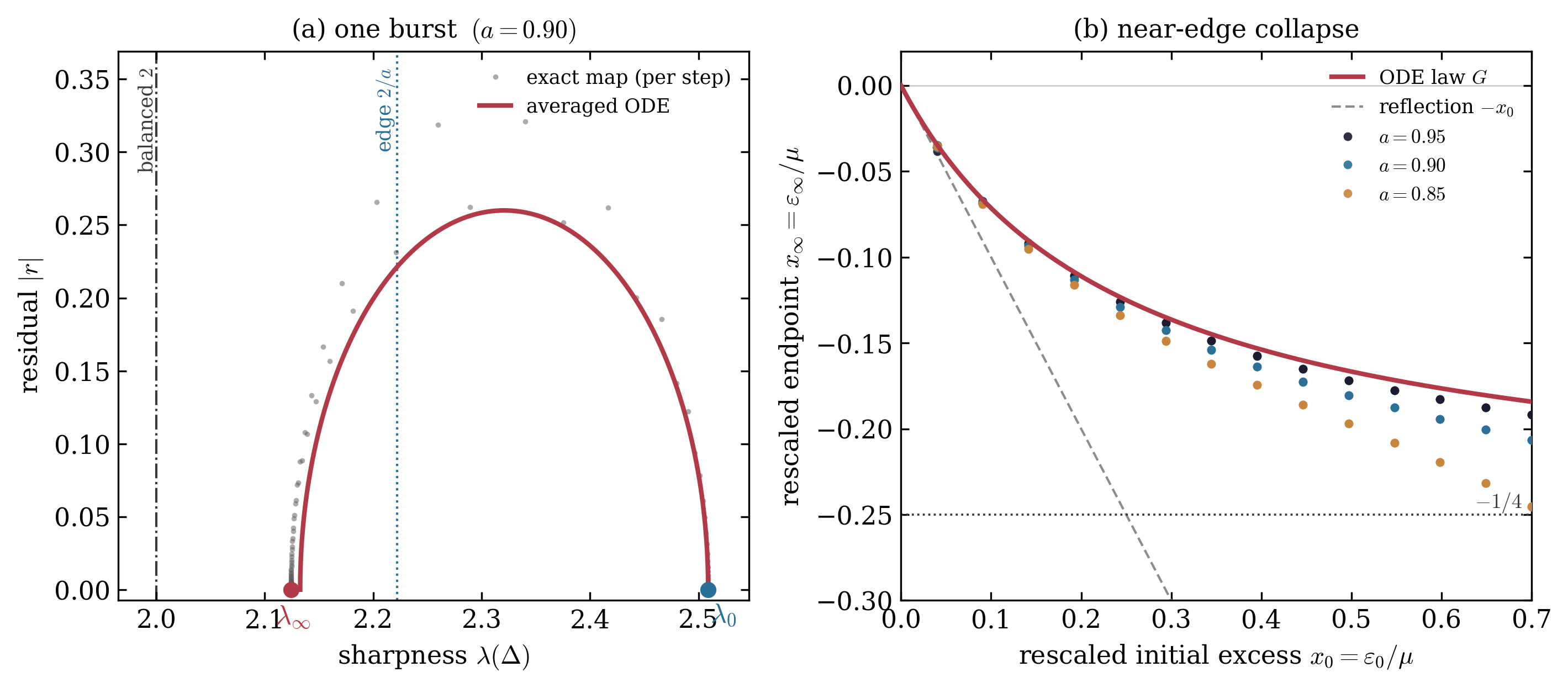}
    \caption{Minimum selection near the edge, from direct simulation of the exact two-factor map. 
    \textbf{(a)} One residual burst at \(a=0.90\), in the physical variables sharpness \(\lambda(\Delta)\) and residual \(|r|\): the map (grey, one point per step) starts at \(\lambda_0\) above the stability threshold \(2/a\) (dotted), grows a residual, and drives the sharpness below the edge to \(\lambda_\infty\), short of the balanced floor \(\lambda=2\) (dash-dotted). 
    The averaged slow system \eqref{eq:slow-system-main} traces the envelope. 
    \textbf{(b)} Endpoints in the universal variable \(X=\varepsilon/\mu_a\): for step sizes approaching the edge, measured endpoints from the exact map collapse onto the parameter-free law \(G(X_\infty)=G(X_0)\), \(G(X)=(1+4X)/(1+2X)^2\) from \eqref{eq:universal-selection}; dashed, the reflection \(-X_0\); dotted, the overshoot floor \(-1/4\). 
    The departure from \(G\) shrinks with \(1-a\), as predicted by the \(O(\mu_a)\) correction.}
    \label{fig:selection-law}
\end{figure}

Figure~\ref{fig:selection-law}(a) shows one such burst in the exact two-factor map: the residual grows, drives the sharpness from above the edge \(2/a\) to below it, and dies out at the selected point, short of the balanced value \(\lambda=2\). 
Figure~\ref{fig:selection-law}(b) shows the universal collapse of the endpoints in the variable \(X=\varepsilon/\mu_a\).
\subsection{The Post-Edge Diagonal Regime}
\label{sec:transverse-stability}

For \(a>1\), no point of the solution valley is normally stable, so the valley-recapture mechanism of Section~\ref{sec:balancing-oscillation} is no longer available. 
What remains is the invariant diagonal \(x=y\), on which the two-factor map reduces exactly to the scalar cubic map \eqref{eq:two-factor-diagonal-map}. 
In the bounded post-edge regime, this diagonal supplies the organizing dynamics. 
The remaining question is whether perturbations away from the diagonal are contracted.

Write a point near the positive diagonal as
\[
    x=s+d,
    \qquad
    y=s-d .
\]
Then \(d\) is the transverse coordinate and
\[
    \Delta=x^2-y^2=4sd .
\]
On the diagonal, \(d=0\), the residual is
\[
    r_k=s_k^2-1,
\]
and the scalar update is
\[
    s_{k+1}=s_k(1-a r_k).
\]
Linearizing the two-factor update in the transverse direction gives
\[
    d_{k+1}
    =
    (1+a r_k)d_k
    +
    O(d_k^3).
\]
Thus the same residual oscillations that generate the scalar post-edge dynamics also determine transverse contraction.

The exact imbalance identity gives an equivalent expression. 
Along the diagonal,
\[
    \Delta_{k+1}
    =
    (1-a^2r_k^2)\Delta_k .
\]
For a periodic diagonal orbit
\[
    s_0\mapsto s_1\mapsto\cdots\mapsto s_{m-1}\mapsto s_0,
\]
the scalar update implies
\[
    \prod_{k=0}^{m-1}(1-a r_k)=1.
\]
Therefore the transverse Lyapunov exponent can be written as
\[
    \chi_\perp
    =
    \frac1m\sum_{k=0}^{m-1}\log|1-a^2r_k^2|
    =
    \frac1m\sum_{k=0}^{m-1}\log|1+a r_k|.
\]
Negative \(\chi_\perp\) means that the diagonal orbit attracts nearby off-diagonal perturbations.

For the near-minimum period-two orbit, this transverse multiplier is explicit. 
Using the scalar two-cycle identities,
\[
    \mu_\perp(a)
    =
    (1+a r_-)(1+a r_+)
    =
    3-2a .
\]
Thus the period-two branch is transversely attracting for
\[
    1<a<2 .
\]
This is only a transverse statement. 
As a scalar orbit, the same two-cycle is attracting only for
\[
    1<a<\sqrt5-1 .
\]
After that, it loses scalar stability and the diagonal dynamics continues through the period-doubling cascade of Section~\ref{sec:quartic-period-doubling}. 
The transverse exponent is then evaluated along the resulting higher-period or chaotic diagonal attractor.

Numerically, we find
\[
    \chi_\perp<0
    \qquad
    \text{for}
    \qquad
    1<a<2,
\]
with \(\chi_\perp\uparrow0\) at the Chebyshev endpoint \(a=2\). 
Thus no transverse blowout is observed before the diagonal bounded dynamics itself disappears. 
In this sense, the two-factor model separates two roles: the scalar diagonal map organizes the post-edge oscillations, while residual oscillations contract the off-diagonal factorization direction.

\begin{figure}[ht]
    \centering
    \includegraphics[width=0.8\linewidth]{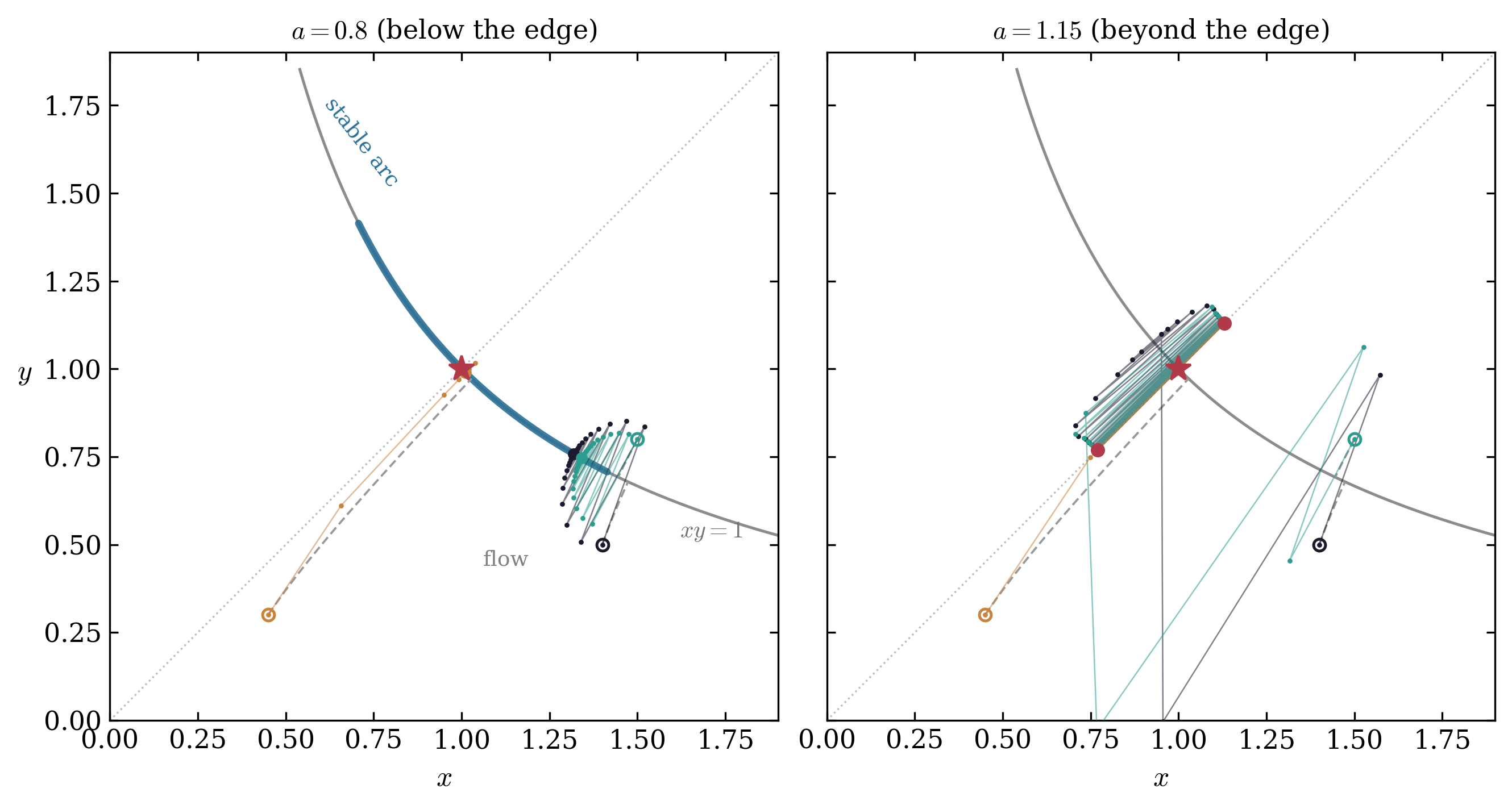}
    \caption{Two-factor dynamics below and beyond the edge, from the same initializations. 
    Left: below the edge, trajectories are captured by the stable arc of the valley \(xy=1\). 
    Right: beyond the edge, residual oscillations contract the imbalance and drive bounded trajectories toward the diagonal post-edge attractor.}
    \label{fig:two-factor-phase}
\end{figure}
\subsection{Conclusion}
\label{sec:two-factor-interpretation}

The two-factor model separates three mechanisms that are conflated on the scalar diagonal. First, finite-step gradient descent changes the factorization imbalance, whereas gradient flow preserves it. Second, below the edge, residual excursions can form a self-extinguishing burst: the residual grows, oscillates, contracts the imbalance, lowers the normal sharpness, and then dies out after the trajectory has moved back into the stable part of the valley. Third, beyond the local edge \(a=1\), no point of the solution valley is normally stable, but bounded trajectories remain organized by the invariant diagonal as long as the transverse exponent is negative.

The model also clarifies its limitations. It contains finite-step flattening and near-edge minimum selection, but it does not explain progressive sharpening from below. The minimum-selection law is an averaged near-edge approximation, with higher-order corrections away from the edge. Likewise, the negativity of the transverse exponent throughout the chaotic windows is supported by exact endpoint identities and numerical evidence, rather than by a complete analytic proof. These diagnostics should nevertheless be useful for deeper chains, stochastic gradients, and higher-dimensional factorizations, where the single imbalance \(\Delta\) is replaced by a family of gauge-like factorization directions.

\subsection{History and Related Work}
\label{sec:two-factor-related}

The conservation of the imbalance under gradient flow is the classical balancedness invariant of linear networks \cite{Saxe2014,arora2018optimization,du2018balanced}: continuous-time training preserves the differences \(x^2-y^2\) between layer norms, which is what makes the flow reduction to the balanced trajectory consistent. The observation that discrete steps violate this conservation law --- and do so in a systematically flattening direction --- is the finite-step effect quantified by Proposition~\ref{prop:finite-step-balancing}. Kreisler et al.~\cite{kreisler2023sharpness} proved for scalar networks that gradient descent monotonically decreases the sharpness of the associated gradient-flow solution; the exact identity \eqref{eq:imbalance-update} is the mechanism of this effect in the two-factor case, and the selection law \eqref{eq:selection-main} quantifies its endpoint near the edge.

The post-edge phenomenology connects to the edge-of-stability literature: the empirical observations of Cohen et al.~\cite{cohen2021gradient}, the analyses of unstable convergence and self-stabilization \cite{ahn2022unstable,damian2023self}, two-step oscillations \cite{chenbruna2023}, trajectory alignment near the flip bifurcation \cite{songyun2023trajectory}, and minimalist models \cite{zhu2023understanding}. The catapult effect of Lewkowycz et al.~\cite{lewkowycz2020large} appears here as the overshoot boundary \eqref{eq:balanced-excess} of the selection law. What the two-factor model adds to this literature is exactness away from the diagonal: the imbalance identity, the stable arc, and the transverse multiplier \(3-2a\) are closed-form statements about the full two-dimensional dynamics, not about a reduced or averaged model.
\section{Toward Realistic Networks}
\label{sec:wider}

The scalar chain admits several orthogonal extensions. 
This section adds width, data, and activation; stochasticity is treated separately in Section~\ref{sec:stochastic}. 
The aim is to show that the finite-step mechanisms isolated above are not artifacts of the scalar reduction. 
They reappear, in exact form, as alignment, spectral edges, balancing, and representation selection in progressively richer learning models.

\subsection{Wide Linear Networks: A Ladder of Edges}
\label{sec:wide-linear}

Consider a deep linear network
\[
    \ell(W_1,\ldots,W_L)=\frac12\|W_L\cdots W_1-M\|_F^2,
    \qquad
    M=USV^\top,
    \qquad
    S=\operatorname{diag}(s_1,\ldots,s_d).
\]
Call an initialization \emph{aligned} if all layers are diagonal in the singular frames of the target:
\[
    W_1=D_1V^\top,
    \qquad
    W_l=D_l \quad (1<l<L),
    \qquad
    W_L=UD_L .
\]
This aligned manifold is invariant under gradient descent. 
On it, the dynamics decouples across singular modes. 
For mode \(i\), the scalar product
\[
    \pi_i=\prod_l (D_l)_{ii}
\]
is trained toward the target singular value \(s_i\). 
Thus the scalar theory applies mode by mode, with the balanced solution satisfying \((D_l)_{ii}=s_i^{1/L}\).

\begin{proposition}[Ladder of spectral edges]
\label{prop:spectral-edge-ladder}
On the aligned manifold, the balanced solution for mode \(i\) has normal sharpness
\[
    Ls_i^{2-2/L}.
\]
Its local stability edge is therefore
\[
    a_i=\frac{2}{L s_i^{2-2/L}}.
\]
The edges are ordered oppositely to the singular values. 
Hence the top singular mode crosses first: if \(s_1>s_2\) and \(a_1<a<a_2\), the leading mode is post-edge while all remaining modes are still locally stable.
\end{proposition}

This gives a literal ladder of spectral edges. 
As the leading singular mode grows, its curvature rises, crosses \(2/a\), and enters the same two-cycle, sharpness hovering, and period-doubling behavior described by the scalar map. 
The other modes continue to converge. 
Thus the loss may keep decreasing while the top curvature oscillates around the stability threshold. 
This reproduces the empirical edge-of-stability trace on an exactly invariant manifold. 
It does not yet explain slow, data-driven progressive sharpening; here the rise of curvature is transient signal growth in a decoupled mode.
\subsection{Alignment of the Singular Frames}
\label{sec:frame-alignment}

The previous reduction assumes alignment. 
In the two-mode, depth-two case, local stability of this assumption can be checked exactly. 
Let \(s_1>s_2>0\). 
Suppose the first mode is on its post-edge two-cycle and the second mode has converged. 
Frame misalignment is governed by a two-step Floquet multiplier in the off-diagonal directions. 
The calculation is given in Appendix~\ref{sec:appendix-wider}.

\begin{proposition}[Rotational edge]
\label{prop:rotational-edge}
The aligned two-cycle is locally stable against physical frame misalignment if and only if
\[
    a(s_1+s_2)<2 .
\]
Equivalently, the rotational edge is the ordinary curvature edge of the valley direction that mixes the two singular modes. 
In particular, alignment is stable throughout the full top-mode two-cycle window
\[
    1<as_1<\sqrt5-1
\]
if and only if
\[
    \frac{s_2}{s_1}<\frac{\sqrt5-1}{2}.
\]
\end{proposition}
Thus the scalar post-edge backbone persists only when the spectral gap is large enough. 
The leading mode may already be beyond its own edge, while rotations mixing it with the second mode remain stable as long as \(a(s_1+s_2)<2\). 
Beyond this rotational edge, the unstable multipliers are complex. 
Misalignment therefore grows with precession rather than along a fixed direction. 
Numerically, the full two-mode system then saturates in a persistently misaligned oscillatory state with nonzero loss. 
Thus well-separated spectra retain the scalar aligned backbone, while near-degenerate spectra exhibit a genuinely non-scalar instability.
\subsection{Optimal Learning Rates Beyond the Edge}
\label{sec:optimal-rate}

The ladder of edges also changes the tuning problem. 
At depth two, write
\[
    A=as .
\]
For a single mode, the asymptotic per-step contraction factor is
\[
    \rho(A)=|1-2A|,
    \qquad
    A\le1,
\]
while on the stable two-cycle branch it is the square root of the two-step multiplier,
\[
    \rho(A)=\bigl|9-2(1+A)^2\bigr|^{1/2},
    \qquad
    1<A<\sqrt5-1 .
\]
A single learning rate must serve all modes, so the natural minimax problem is
\[
    a^*=\arg\min_a\max_i\rho(as_i).
\]
For two modes with spread
\[
    \sigma=\frac{s_2}{s_1},
\]
the best sub-edge rate gives contraction
\[
    \frac{1-\sigma}{1+\sigma}.
\]
For sufficiently spread spectra, however, this is not optimal. 
Solving the two-mode minimax equality shows that when
\[
    \sigma<0.5718\ldots ,
\]
the minimax rate lies beyond the top edge. 
The leading mode is then parked near its superstable two-cycle, while the trailing mode is accelerated. 
For example, at \(\sigma=1/2\), the optimum is
\[
    a^*s_1\approx1.120,
\]
giving contraction about \(0.120\) per step instead of the best sub-edge value \(0.333\). 
Thus the optimal learning rate need not lie at the first edge; for sufficiently spread spectra, it can lie slightly beyond it.
\subsection{Depth Beyond Two}
\label{sec:depth-beyond-two}

The same mechanism extends to \(L\) factors. 
On the balanced diagonal of a single mode, the dynamics is exactly the scalar chain of depth \(L\). 
The edge, quotient coordinate, two-cycle, period-doubling cascade, and depth limit therefore apply mode by mode. 
What is added are the \(L-1\) layer-imbalance directions within each mode.

Along a balanced one-mode trajectory, where all layer weights are equal to \(x\), the linearized multiplier in each imbalance direction is
\[
    \lambda_\perp=1+a r x^{L-2}.
\]
For \(L=2\), this reduces to the two-factor multiplier \(1+ar\). 
Thus the same residual oscillations that generate the scalar post-edge dynamics also act on the factorization directions. 
In the regimes where the corresponding products of transverse multipliers have modulus below one, these oscillations contract imbalance and drive the factors toward balance. 
Depth changes the constants and the scale of the edge, but not the mechanism.
\subsection{Coupling to Data}
\label{sec:data-coupling}

The preceding models learn a single scalar target. 
The smallest data-coupled extension makes the first factor a vector and keeps a shared scalar readout:
\begin{align}
    \ell(x,y)=\frac12\|xy-d\|^2,
    \qquad
    x,d\in\mathbb R^k,
    \quad
    y\in\mathbb R,
    \quad
    z=xy-d .
    \label{eq:data-model}
\end{align}
For \(k=1\), this is the two-factor loss. 
For \(k\ge2\), the residual direction can rotate, which is the ingredient absent from the scalar and rank-one models. 
We write \(k\) for the data dimension in order to reserve \(r\) for scalar residuals.

Let
\[
    \Delta=\|x\|^2-y^2 .
\]
Gradient flow conserves \(\Delta\). 
Finite-step gradient descent instead satisfies the exact identity
\begin{align}
    \Delta^+
    =
    (1-a^2\|z\|^2)\Delta
    +
    a^2\|x\wedge z\|^2,
    \qquad
    \|x\wedge z\|^2=\|x\|^2\|z\|^2-(x\cdot z)^2 .
    \label{eq:vector-imbalance-update}
\end{align}
The first term is finite-step damping; the second is a misalignment pump. 
Thus data coupling introduces a competition between balancing and residual rotation.

The zero-loss set is the aligned curve
\[
    x=\frac{d}{y}.
\]
Along this curve, the largest nonzero curvature is
\[
    \lambda(y)=y^2+\frac{\|d\|^2}{y^2},
\]
which is minimized at \(y^2=\|d\|\). 
Hence the flattest solution reaches the edge at
\[
    A:=a\|d\|=1 .
\]
Let \(\hat d=d/\|d\|\).

\begin{proposition}[Data-coupled two-cycle]
\label{prop:data-two-cycle}
For \(1<A<\sqrt5-1\), gradient descent on \eqref{eq:data-model} has an exactly aligned and balanced period-two orbit
\[
    x=\xi_\pm\sqrt{\|d\|}\,\hat d,
    \qquad
    y=\xi_\pm\sqrt{\|d\|},
\]
where \(\xi_\pm\) is the quartic two-cycle at parameter \(A\). 
This orbit is linearly stable. 
Its Floquet multipliers are
\[
    \mu_{\mathrm{cyc}}=9-2(1+A)^2,
    \qquad
    \mu_{\mathrm{bal}}=3-2A,
    \qquad
    \mu_{\mathrm{align}}=1-A,
\]
with \(\mu_{\mathrm{align}}\) of multiplicity \(k-1\).
\end{proposition}

The derivation of \eqref{eq:vector-imbalance-update} and the proof of Proposition~\ref{prop:data-two-cycle} are given in Appendix~\ref{app:data-coupling}.

Thus, in this data-coupled model, post-edge dynamics drives both alignment and balancing: the attracting two-cycle lies at the flattest zero-loss representation. 
This complements the landscape results of \cite{yoo2025minimalist}: their bounds identify balanced solutions as flat, while the orbit-level calculation above shows how finite-step gradient descent can select such a solution dynamically. 
With anisotropic data scales, the residual direction no longer disappears; the attractor can remain misaligned and imbalanced. 
A full treatment of residual rotations is beyond the scope of this paper.
\subsection{A Neuron with Linear Readout}
\label{sec:neuron-readout}

The minimal nonlinear extension is one activation and one linear readout,
\[
    \ell(w,v)=\frac12\bigl(v\sigma(w)-1\bigr)^2 .
\]
For \(\sigma=\mathrm{id}\), this is exactly the two-factor model. 
For general \(\sigma\), the zero-loss valley is
\[
    v=\frac1{\sigma(w)},
\]
and its nonzero curvature is
\[
    \lambda(w)=\sigma(w)^2+\frac{\sigma'(w)^2}{\sigma(w)^2}.
\]
Thus activation shape becomes a representation-selection problem: large steps exclude valley regions whose curvature exceeds the stability threshold \(2/a\).

For \(\sigma=\tanh\), the curvature profile has a wall near zero, a unique finite flattest point, and a saturation plateau:
\[
    \tanh^2 w_*=\frac1{\sqrt2},
    \qquad
    \lambda_* = 2(\sqrt2-1),
    \qquad
    \lambda(w)\to1
    \quad (w\to\infty).
\]
Consequently the stable part of the valley has three regimes:
\[
\begin{array}{ccl}
    a<2
    &:&
    \text{saturated solutions remain stable},\\[1mm]
    2<a<1+\sqrt2
    &:&
    \text{only the flat window around } w_* \text{ is stable},\\[1mm]
    a>1+\sqrt2
    &:&
    \text{no valley point is stable}.
\end{array}
\]
In the middle regime, saturated representations are beyond the edge, while the finite flat window remains stable. 
Post-edge oscillations can therefore expel the neuron from saturation and funnel it toward the flattest representation. 
The edge acts not only as a stability boundary, but as a representation selector.
\subsection{ReLU: An Absorbing State}
\label{sec:relu-death}

For \(\sigma=\mathrm{ReLU}\), the active half-space \(w>0\) is the two-factor model. 
On the balanced diagonal \(w=v=x\), it reduces to the quartic gradient map. 
The inactive half-line \(w\le0\) is absorbing: \(\sigma(w)=\sigma'(w)=0\), the gradients vanish, and the residual is frozen at \(-1\). 
A crossing of the origin, harmless in the smooth model, therefore inactivates the unit permanently.

Let
\[
    a_*=\frac{\sqrt{27}}2-1 .
\]

\begin{proposition}[ReLU inactivation threshold]
\label{prop:relu-death}
On the balanced diagonal, the positive well of the quartic map is forward invariant for \(a\le a_*\). 
Hence a unit initialized in this well remains active forever and follows the linear two-factor dynamics. 
For \(a>a_*\), the critical image crosses the sign-change boundary, so active trajectories can enter the absorbing half-line \(w\le0\). 
Any such crossing is terminal.
\end{proposition}

The first statement is exact and is precisely the envelope bound from Section~\ref{sec:quartic-envelope}. 
The second identifies the ReLU-specific failure mechanism. 
Numerically, beyond \(a_*\), typical active orbits are absorbed after intermittent excursions, with survival times growing as \(a\downarrow a_*\). 
The well-merging crisis of the smooth scalar map thus acquires an operational meaning under ReLU: it becomes the exact large-step threshold at which inactive, or ``dead,'' units can first occur on the balanced diagonal.
\subsection{What Survives in General}
\label{sec:what-survives}

The extensions above suggest three levels of robustness. 
First, on invariant manifolds --- aligned spectral modes, balanced diagonals, aligned data directions, and fixed activation patterns --- the scalar chain is exact. 
Second, near the edge, the birth of the two-cycle is a flip bifurcation, so curvature hovering and straddling are normal-form phenomena rather than artifacts of the quartic. 
Third, beyond exact reductions, realistic edge-of-stability dynamics often remains effectively one-dimensional along the leading curvature direction. 
In this sense, the scalar chain is not merely a toy model: in the regimes analyzed here, it is the normal form of the direction in which training oscillates.
\subsection{History and Related Work}

The exact reductions used in this section are rooted in the classical analysis of deep linear networks. 
Saxe et al.~\cite{Saxe2014} showed that, under aligned initializations, deep linear training decouples into scalar modes, making the singular directions analytically tractable. 
The ladder of spectral edges above is the finite-step version of this decoupling: each singular mode has its own curvature scale and therefore its own stability threshold.

Recent work on edge-of-stability dynamics in matrix factorization and deep linear networks has observed precisely this post-edge behavior: leading modes cross their edge first, enter oscillatory regimes, and may undergo period doubling or chaotic dynamics as the learning rate increases 
\cite{Ghosh2025DeepMatrixFactorizationEOS,ghosh2025dln}. 
The contribution here is to isolate the corresponding exact maps and to separate three effects that are usually intertwined: scalar post-edge oscillation, factorization balancing, and frame alignment.

Balancing in homogeneous models has a separate history. 
For gradient flow, layer imbalance is conserved in the simplest factorizations, while overparameterized homogeneous models exhibit implicit balancing and acceleration effects 
\cite{du2018balanced,arora2018optimization}. 
The two-factor calculation above shows the finite-step analogue in its sharpest form: residual oscillations break the gradient-flow conservation law and contract imbalance directly. 
This connects the dynamical edge-of-stability picture to the flatness of balanced factorizations.

The data-coupled and activation examples are meant as minimal extensions rather than full models of realistic networks. 
They show that the same finite-step mechanism can select aligned data directions, balanced representations, and nonsaturated activation states. 
This is consistent with recent minimalist analyses of sharpness dynamics and representation selection in neural-network training \cite{yoo2025minimalist}. 
The common point is that the edge is not only a local stability threshold. 
In these exact reductions, it also acts as a selection mechanism for the representations that remain dynamically accessible at large step size.
\section{Stochastic Targets}
\label{sec:stochastic}

\subsection{Learning Dynamics with Noisy Targets}
\label{sec:noisy-target-dynamics}

The theory so far has been deterministic. We now add persistent randomness in the simplest setting that still admits exact identities: the two-factor model with a noisy target,
\begin{align*}
    \ell_t(x,y)
    =
    \frac12\bigl(xy-1-\xi_t\bigr)^2,
    \qquad
    \xi_t\stackrel{\mathrm{iid}}{\sim}\mathcal N(0,\sigma^2).
\end{align*}
This is a tractable caricature of stochastic gradient descent. It differs from minibatch noise in one important respect: label noise does not vanish at interpolation. Even at a minimum, the update still sees the random residual \(-\xi_t\).\footnote{One could also study noise schedules with \(\sigma\to0\). Here \(\sigma\) is fixed, so the noise is persistent.} These persistent residuals create a flatness-selecting mechanism \cite{blanc2020implicit,damian2021label}.

Our methodology remains unchanged: preserve the exact deterministic identities wherever possible, and resort to asymptotic expansions only where noise introduces genuinely new behavior. The only formal change is that the deterministic residual
\begin{align*}
    r_t=x_ty_t-1
\end{align*}
is replaced by the noisy residual
\begin{align*}
    \rho_t=r_t-\xi_t.
\end{align*}
Thus gradient descent becomes
\begin{align*}
    x_{t+1}
    =
    x_t-a\,y_t\rho_t,
    \qquad
    y_{t+1}
    =
    y_t-a\,x_t\rho_t.
\end{align*}
The next subsection records the exact imbalance identity that survives this replacement.

\subsection{Noise Removes the Freeze Below the Edge}
\label{sec:noise-removes-freeze}

The central deterministic identity survives pathwise. 
With
\[
    \Delta_t:=x_t^2-y_t^2,
\]
the noisy update satisfies
\begin{align}
    \Delta_{t+1}
    =
    (1-a^2\rho_t^2)\Delta_t .
    \label{eq:noisy-imbalance-update}
\end{align}
Thus the imbalance evolves as a random product. 
A single step contracts \(|\Delta_t|\) whenever
\[
    |a\rho_t|<\sqrt2 .
\]
With Gaussian noise, however, \(\rho_t\) is unbounded, so there is no uniform stepwise contraction condition. 
The relevant stability notion is logarithmic: the transverse Lyapunov exponent, the noisy counterpart of \(\chi_\perp\) from Section~\ref{sec:transverse-stability}, is
\begin{align}
    \chi(a,\sigma)
    =
    \lim_{T\to\infty}
    \frac1T
    \sum_{t<T}
    \log |1-a^2\rho_t^2|.
    \label{eq:chi-def}
\end{align}
In a stationary ergodic regime this becomes
\[
    \chi(a,\sigma)
    =
    \mathbb E\,\log|1-a^2\rho^2|.
\]
A negative value of \(\chi\) means that the random product in \eqref{eq:noisy-imbalance-update} contracts exponentially almost surely. 
In this sense, the balanced diagonal is transversely attracting, even though rare individual noisy steps may expand the imbalance.

For \(0<a<1\), the residual dynamics near the balanced point is stable. 
Persistent label noise therefore creates a stationary residual cloud rather than a deterministic transient. 
Linearizing the noisy residual recursion at the balanced point gives a stable autoregressive process. 
Its stationary variance is
\[
    \operatorname{Var}(\rho)
    =
    \frac{\sigma^2}{1-a}
    +
    O(\sigma^4).
\]
Expanding the logarithm in \eqref{eq:chi-def} then gives the leading transverse contraction rate.

\begin{proposition}[Stochastic balancing below the edge]
\label{prop:stochastic-balancing}
Fix \(0<a<1\). 
In the small-noise stationary regime near the balanced valley point,
\begin{align}
    \chi(a,\sigma)
    =
    -\frac{a^2\sigma^2}{1-a}
    +
    O(\sigma^4).
    \label{eq:stochastic-balancing-rate}
\end{align}
In particular, for sufficiently small target noise, the balanced diagonal is transversely attracting throughout the sub-edge range.
\end{proposition}

The exact derivation is given in Appendix~\ref{app:stochastic-rate}. 
The mechanism is the stochastic counterpart of finite-step balancing: balancing requires non-vanishing residual sequences. 
Below the edge, deterministic residuals die out and the imbalance freezes; persistent label noise keeps producing residual fluctuations, and these fluctuations contract the imbalance on average.

\paragraph{Numerical validation.}

Figure~\ref{fig:noise-freeze} validates the sub-edge balancing law. 
The markers show direct simulations of the target-noise model for several noise levels. 
The solid curves use the stationary Gaussian approximation for the noisy residual, while the dashed curves show the leading small-noise truncation from Proposition~\ref{prop:stochastic-balancing}. 
Over the plotted sub-edge range, the two predictions are nearly indistinguishable and match the simulations well.

\begin{figure}[t]
    \centering
    \includegraphics[width=0.68\linewidth]{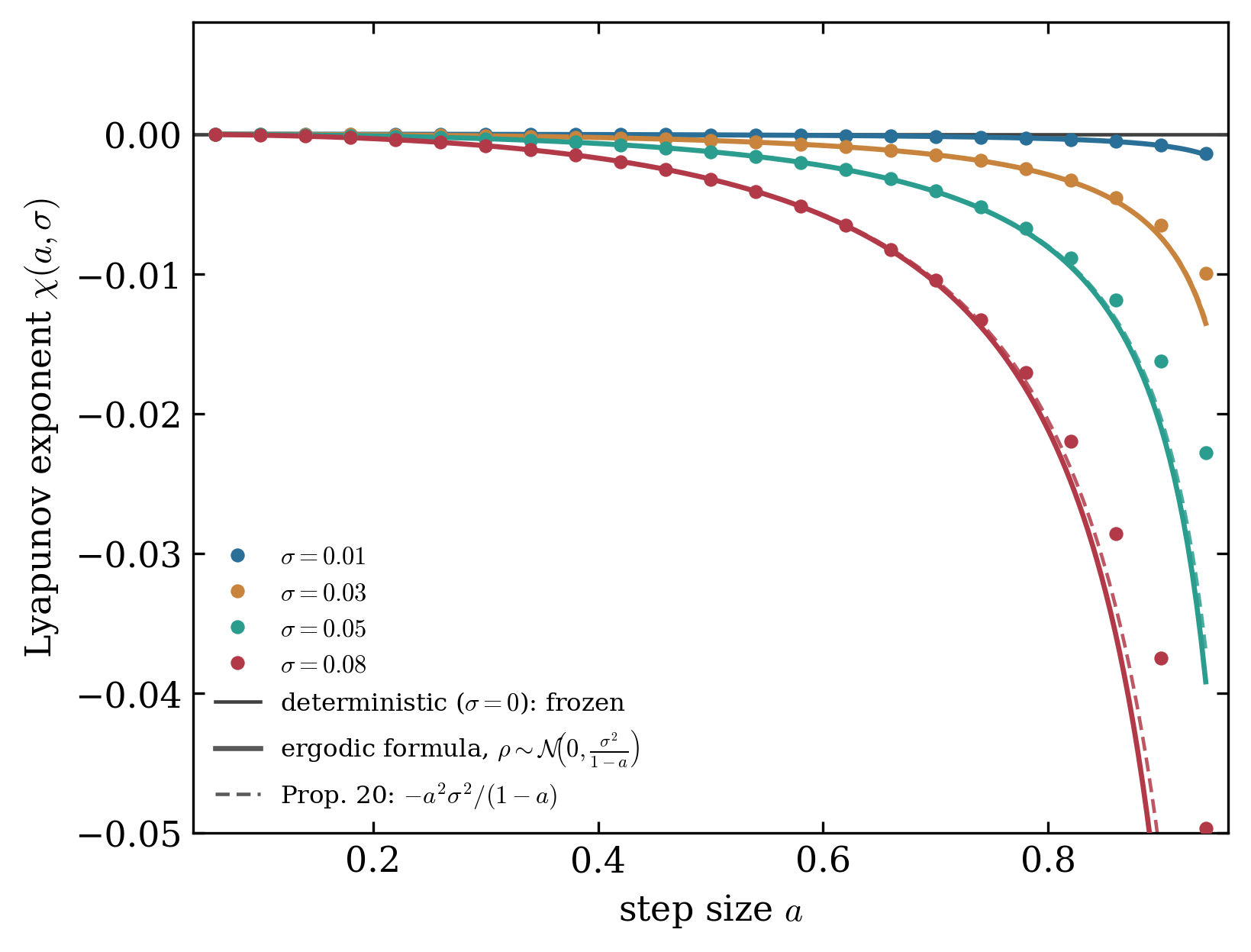}
    \caption{Noise removes the freeze. 
    Markers: simulated \(\chi(a,\sigma)\), for \(\sigma=0.01,0.03,0.05,0.08\). 
    Solid: ergodic formula \(\mathbb E[\log|1-a^2\rho^2|]\) under the stationary approximation
    \(\rho\sim\mathcal N(0,\sigma^2/(1-a))\). 
    Dashed: the small-noise truncation \(-a^2\sigma^2/(1-a)\) from Proposition~\ref{prop:stochastic-balancing}. 
    Black: deterministic freeze, \(\sigma=0\), where \(\chi=0\).}
    \label{fig:noise-freeze}
\end{figure}

\subsection{Crossover at the Edge}
\label{sec:noise-crossover}

The below-edge rate \eqref{eq:stochastic-balancing-rate} becomes singular as \(a\uparrow1\). 
This does not mean that the true balancing rate diverges; it means that the small-noise expansion has reached the edge of its validity. 
Near \(a=1\), two mechanisms have comparable size. 
Below the edge, noise creates residual fluctuations and
\[
    \chi(a,\sigma)\sim-\frac{\sigma^2}{1-a}.
\]
Above the edge, even without noise, the fixed point has been replaced by a deterministic two-cycle, whose transverse exponent near the edge is
\[
    \chi(a,0)=\frac12\log(3-2a)=-(a-1)+O((a-1)^2).
\]
Equating the two magnitudes gives a crossover window of width \(O(\sigma)\) around \(a=1\). 
We therefore zoom into the edge by keeping \((a-1)/\sigma\) fixed. 
In this scaling, the sub-edge stochastic law and the deterministic post-edge law are connected by a single universal curve.

On the diagonal, the residual dynamics near the edge has the noisy flip form
\begin{align}
    r_{t+1}
    =
    -\bigl(1+2(a-1)\bigr)r_t
    -
    r_t^2
    +
    2\xi_t
    +
    \cdots .
    \label{eq:noisy-flip-map}
\end{align}
The leading minus sign is the flip: the residual changes sign from one step to the next. 
As in the deterministic burst analysis, the slowly varying quantity is the squared residual envelope
\[
    v_t=r_t^2.
\]

In the edge scaling
\[
    a-1=O(\sigma),
    \qquad
    v=O(\sigma),
\]
the one-step change in \(v\) is small compared with \(v\) itself. 
We therefore approximate the slow envelope by a continuous-time diffusion whose drift and variance match the conditional mean and variance of the one-step increment. 
Squaring \eqref{eq:noisy-flip-map}, averaging over the fast sign alternation, and matching the nonlinear saturation to the deterministic two-cycle gives
\begin{align}
    \mathbb E[v_{t+1}-v_t\mid v_t=v]
    =
    4(a-1)v
    -
    4v^2
    +
    4\sigma^2
    +
    \text{higher-order terms}.
    \label{eq:v-drift}
\end{align}
The term \(4(a-1)v\) is the linear instability of the flip, the term \(-4v^2\) is the deterministic nonlinear saturation, and the term \(4\sigma^2\) comes from squaring the additive noise. 
The leading conditional variance comes from the cross term between the current residual and the new noise, and equals
\[
    16\sigma^2 v .
\]
Thus the slow envelope is approximated by the diffusion
\begin{align}
    \mathrm d v
    =
    \bigl(4(a-1)v-4v^2+4\sigma^2\bigr)\,\mathrm dt
    +
    4\sigma\sqrt v\,\mathrm dB_t .
    \label{eq:noisy-edge-v-diffusion}
\end{align}
The derivation from the discrete map is given in Appendix~\ref{app:stochastic-crossover}.

The diffusion \eqref{eq:noisy-edge-v-diffusion} is explicitly solvable. 
Its stationary density, obtained from the zero-flux stationary Fokker--Planck equation, is
\begin{align}
    p(v)
    \propto
    v^{-1/2}
    \exp\!\left\{
        \frac{(a-1)v}{2\sigma^2}
        -
        \frac{v^2}{4\sigma^2}
    \right\},
    \qquad v>0 .
    \label{eq:noisy-edge-stationary-density}
\end{align}
For the normal form, the transverse balancing exponent is
\[
    \chi=-\mathbb E[v]+o(\sigma).
\]
Thus the crossover reduces to computing the stationary mean of \(v\).

\begin{proposition}[Noisy edge crossover normal form]
\label{prop:noise-crossover}
For the diffusion normal form \eqref{eq:noisy-edge-v-diffusion}, in the scaling limit \(\sigma\to0\) with
\[
    w=\frac{a-1}{\sqrt2\,\sigma}
\]
fixed, the balancing exponent satisfies
\begin{align}
    \chi(a,\sigma)
    =
    -\sqrt2\,\sigma\,F(w)+o(\sigma),
    \label{eq:noisy-edge-crossover-law}
\end{align}
where
\begin{align}
    F(w)
    =
    \frac{\int_0^\infty z^{1/2}e^{wz-z^2/2}\,\mathrm dz}
         {\int_0^\infty z^{-1/2}e^{wz-z^2/2}\,\mathrm dz}.
    \label{eq:noisy-edge-F}
\end{align}
The asymptotes
\[
    F(w)\sim w
    \quad (w\to+\infty),
    \qquad
    F(w)\sim\frac1{2|w|}
    \quad (w\to-\infty)
\]
recover the deterministic two-cycle rate above the edge and the near-edge stochastic balancing rate below it. 
At the edge,
\[
    \chi(1,\sigma)
    =
    -\frac{2\Gamma(3/4)}{\Gamma(1/4)}\,\sigma
    +
    o(\sigma)
    \approx
    -0.6760\,\sigma .
\]
\end{proposition}

Proposition~\ref{prop:noise-crossover} is a theorem for the diffusion normal form. 
For the original noisy map, it is an asymptotic normal-form prediction: the diffusion is derived from the slow envelope \(v=r^2\), but we do not prove convergence of the discrete process to \eqref{eq:noisy-edge-v-diffusion}. 
The numerical experiment below tests this prediction directly.

\paragraph{Numerical validation.}

Figure~\ref{fig:noise-crossover} simulates the noisy map for \(\sigma=0.01,0.03,0.05,0.08\). 
After rescaling by the single variable
\[
    w=\frac{a-1}{\sqrt2\,\sigma},
\]
all four runs collapse onto the same curve near \(w=0\), with no free parameters. 
The residual gap is a finite-\(\sigma\) effect: the measured-to-predicted ratio at the point nearest the edge is \(1.35\), \(1.13\), \(1.05\), and \(0.99\) for \(\sigma=0.08,0.05,0.03,0.01\), respectively, converging to one as \(\sigma\to0\).

\begin{figure}[t]
    \centering
    \includegraphics[width=0.65\linewidth]{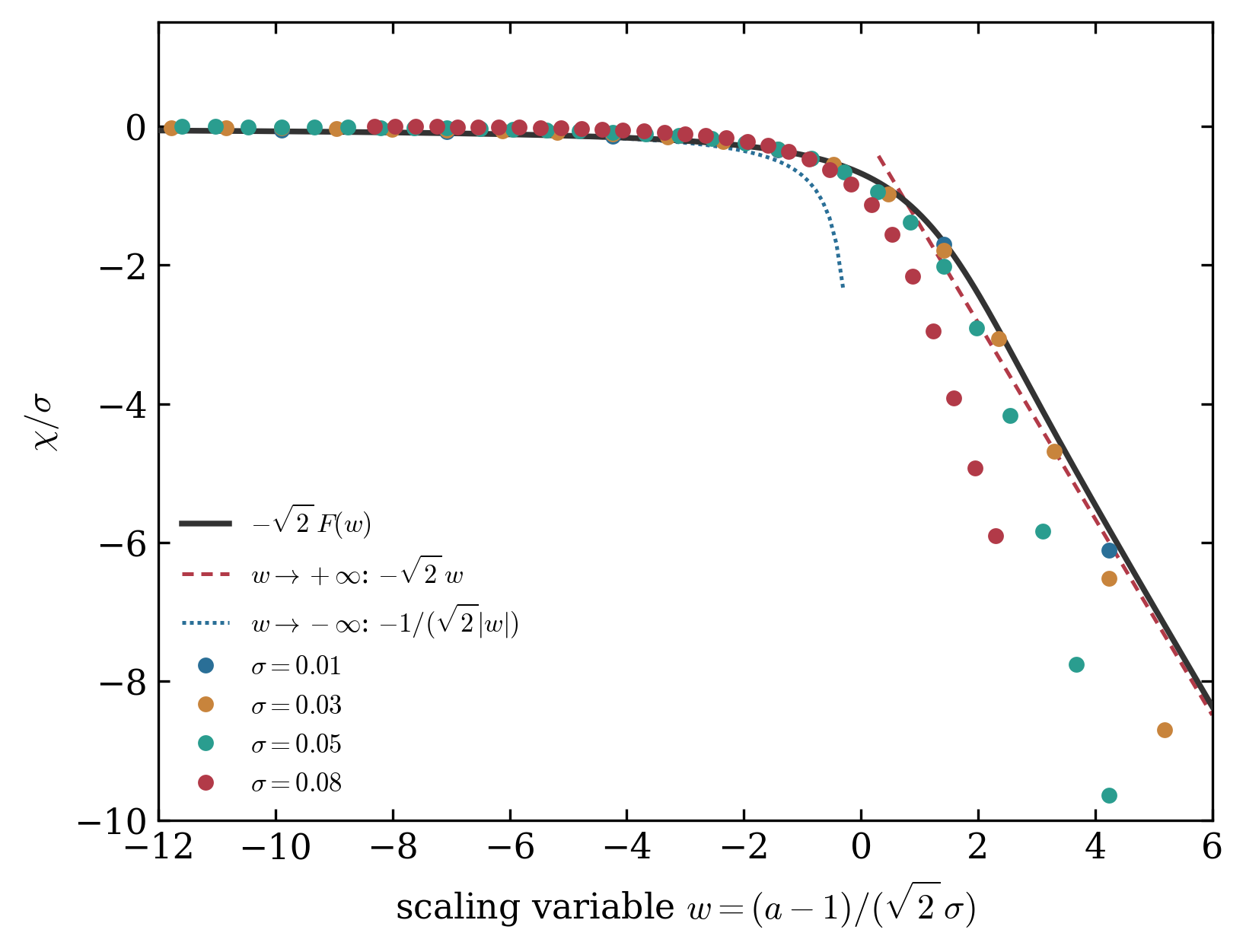}
    \caption{Noisy edge crossover. 
    Markers: measured \(\chi/\sigma\) against \(w=(a-1)/(\sqrt2\,\sigma)\), for \(\sigma=0.01,0.03,0.05,0.08\). 
    Solid: \(-\sqrt2\,F(w)\), the closed form of Proposition~\ref{prop:noise-crossover}. 
    Dashed/dotted: the asymptotes \(-\sqrt2\,w\) as \(w\to+\infty\) and \(-1/(\sqrt2\,|w|)\) as \(w\to-\infty\).}
    \label{fig:noise-crossover}
\end{figure}
\subsection{Attractors Under Noise}
\label{sec:noise-attractors}

The crossover calculation concerns one observable: the transverse balancing exponent. 
It does not describe the full stochastic phase diagram. 
We therefore record how the deterministic attractors of the earlier sections are expected to respond to persistent noise, and where the exact results above stop.

First, noise broadens attracting cycles. 
Near a stable periodic orbit, the return map has the local form
\[
    e_+=m e+\sigma \zeta+O(e^2),
    \qquad |m|<1,
\]
where \(m\) is the deterministic multiplier and \(\zeta\) is an effective noise term. 
The corresponding stationary cloud has typical width
\[
    O\!\left(\frac{\sigma}{\sqrt{1-m^2}}\right).
\]
Thus the cloud widens near a bifurcation, where \(|m|\uparrow1\). 
The noisy-edge crossover of Proposition~\ref{prop:noise-crossover} is the solvable version of this broadening at the first flip.

Second, noise truncates fine period-doubling structure. 
The deterministic cascade of Section~\ref{sec:quartic-period-doubling} contains arbitrarily small scales. 
At fixed \(\sigma>0\), sufficiently fine levels of the cascade are washed out. 
In the classical theory of noisy period doubling, each successive level tolerates about \(6.6\) times less noise than the previous one \cite{crutchfield1981scaling,shraiman1981scaling}. 
Thus only the coarse part of the deterministic diagram remains visible as distinguishable stochastic structure.

Third, noise creates leakage between basins. 
Deterministic attractors near basin boundaries become metastable: away from a crisis, escape rates are exponentially small in \(\sigma^{-2}\), while near a crisis they rise rapidly. 
A sharp deterministic threshold therefore becomes a smooth escape-rate curve. 
In the ReLU model of Section~\ref{sec:relu-death}, for instance, target noise produces rare inactivation events even below the deterministic dying-ReLU threshold, with rates that increase quickly as the threshold is approached.

The lesson for the present paper is modest. 
Persistent noise does not destroy the balancing mechanism. 
Below the edge, it creates the residual fluctuations that keep balancing active. 
Near the edge, it produces the universal crossover of Proposition~\ref{prop:noise-crossover}. 
Farther into the post-edge regime, it rounds and blurs the deterministic attractor structure rather than producing a new exactly solvable phase diagram \cite{agarwala2024stochastic}.
\subsection{History and Related Work}
\label{sec:stochastic-related}

The flatness-selecting effect of label and target noise is known from the implicit-regularization literature. 
Blanc et al.~\cite{blanc2020implicit} and Damian et al.~\cite{damian2021label} show that label-noise SGD drifts toward flat minima, with rates proportional to the squared step size and the noise variance in the small-step limit. 
Proposition~\ref{prop:stochastic-balancing} is the finite-step version of this effect in the two-factor model. 
The rate
\[
    \frac{a^2\sigma^2}{1-a}
\]
reduces to the small-step scaling \(a^2\sigma^2\) as \(a\to0\), while the factor \(1/(1-a)\) is a finite-step amplification invisible to continuous-time analyses.

The behavior of period-doubling systems under external noise is classical. 
The scaling theory of noisy cascades was developed by Crutchfield, Nauenberg, and Rudnick~\cite{crutchfield1981scaling} and by Shraiman, Wayne, and Martin~\cite{shraiman1981scaling}. 
Proposition~\ref{prop:noise-crossover} is the corresponding normal-form statement at the first flip, specialized to the observable that matters here: the transverse balancing exponent. 
On the machine-learning side, a stochastic edge of stability, where sharpness stabilizes below \(2/\eta\) under SGD, has been analyzed in high-dimensional models by Agarwala and Pennington~\cite{agarwala2024stochastic}. 
The label-noise two-factor model gives a minimal solvable counterpart.
\section{Conclusion}
\label{sec:conclusion}

The aim of this paper was to derive exact results for the effect of finite step sizes in gradient-based training. 
The starting point was the cubic gradient map generated by the quartic potential
\[
    \ell(x)=\frac14(x^2-1)^2 .
\]
Despite its simplicity, this map already contains a striking range of finite-step phenomena: loss of fixed-point stability, stable two-cycles, sharpness hovering, period doubling, chaotic windows, and an exactly solvable Chebyshev endpoint. 
These phenomena are not tied to a single depth. 
After the appropriate learning-rate scaling, the depth family admits a universal Ricker-type limit, and the post-edge phase diagram becomes depth-independent.

The broader strategy was then to move outward from this solvable core. 
The two-factor model adds a factorization direction; wider linear models add multiple spectral modes and frame alignment; activation functions add representation-dependent curvature; persistent label noise adds stochastic residuals. 
In each case the dynamics becomes richer, but the central lesson is the same: the scalar finite-step dynamics is not simply discarded or replaced. 
It remains at work inside the more structured models, either as an invariant diagonal subsystem, as the backbone of post-edge dynamics, or as the deterministic structure around which additional effects act.

This perspective separates mechanisms that are often conflated. 
The edge of stability is a local statement about the loss of one-step stability of a fixed point. 
The post-edge regime is non-local: it is governed by multi-step orbits and their return multipliers. 
Factorization directions add another mechanism, because finite steps can change imbalance and therefore sharpness even when the represented function is unchanged. 
Persistent noise adds yet another mechanism, because it keeps residuals alive and turns transient balancing into a long-time effect.

Thus the solvable models do not claim to be realistic neural networks in miniature. 
Their role is more specific: they isolate finite-step mechanisms that any broader theory of edge-of-stability training must account for. 
Progressive sharpening itself remains outside the scope of the present analysis. 
What we have described is what can happen once the dynamics reaches the edge: pointwise convergence can give way to organized oscillation; oscillation can stabilize and balance factorized representations; and finite-step training can remain structured far beyond the local quadratic threshold.
\part{Appendix}
\label{part:appendix}
\appendix

\noindent
The appendices mirror the sections of the paper: Appendix~\ref{app:cubic} collects the derivations for Section~\ref{sec:cubic}, Appendix~\ref{app:depth} for Section~\ref{sec:depth-limit}, and so on. Appendix~\ref{app:toolbox} is a self-contained, informal glossary of the dynamical-systems tools used throughout; it is supplementary and assumes no background.

\section{Derivations for Section~\ref{sec:cubic}: The Cubic Map}
\label{app:cubic}

\subsection{Derivation of the Period-Two Orbit}
\label{sec:appendix-quartic}

We derive the identities used in Section~\ref{sec:quartic-two-cycle}. Recall that the two-cycle satisfies
\begin{align*}
x_+=(1+a)x_- -a x_-^3,
\qquad
x_-=(1+a)x_+ -a x_+^3 .
\end{align*}

Adding the two equations gives
\begin{align*}
x_-+x_+
=
(1+a)(x_-+x_+)
-
a(x_-^3+x_+^3).
\end{align*}
Since the orbit born from the positive minimum satisfies \(x_-+x_+>0\), we divide by
\(a(x_-+x_+)\) and use
\begin{align*}
x_-^3+x_+^3
=
(x_-+x_+)(x_-^2-x_-x_++x_+^2)
\end{align*}
to obtain
\begin{align}
x_-^2-x_-x_++x_+^2
=
1.
\label{eq:cycle-id1}
\end{align}

Similarly, subtracting the two equations yields
\begin{align*}
x_+-x_-
=
(1+a)(x_- -x_+)
-
a(x_-^3-x_+^3).
\end{align*}
Since \(x_+\neq x_-\), we divide by \(a(x_+-x_-)\) and use
\begin{align*}
x_-^3-x_+^3
=
(x_- -x_+)
(x_-^2+x_-x_++x_+^2)
\end{align*}
to obtain
\begin{align}
x_-^2+x_-x_++x_+^2
=
1+\frac2a.
\label{eq:cycle-id2}
\end{align}

Subtracting \eqref{eq:cycle-id1} from \eqref{eq:cycle-id2} immediately gives
\begin{align*}
x_-x_+
=
\frac1a.
\end{align*}
Adding the two identities yields
\begin{align*}
x_-^2+x_+^2
=
1+\frac1a.
\end{align*}
Hence
\begin{align*}
(x_-+x_+)^2
=
x_-^2+x_+^2+2x_-x_+
=
1+\frac3a,
\end{align*}
and therefore
\begin{align*}
x_-+x_+
=
\sqrt{1+\frac3a},
\end{align*}
where we choose the positive root because \(x_\pm>0\).

The two orbit points are therefore the roots of
\begin{align*}
t^2
-
(x_-+x_+)t
+
x_-x_+
=
0,
\end{align*}
which gives
\begin{align*}
x_\pm
=
\frac12
\left(
\sqrt{1+\frac3a}
\pm
\sqrt{1-\frac1a}
\right).
\end{align*}

Finally,
\begin{align*}
x_\pm^2
=
\frac14
\left(
1+\frac3a
+
1-\frac1a
\pm
2\sqrt{\left(1+\frac3a\right)\left(1-\frac1a\right)}
\right),
\end{align*}
which simplifies to
\begin{align*}
x_\pm^2
=
\frac1{2a}
\left[
a+1
\pm
\sqrt{(a-1)(a+3)}
\right].
\end{align*}

\subsection{Multiplier of the Cubic Two-Cycle}
\label{sec:appendix-cycle-multiplier}

We derive the multiplier formula used in Section~\ref{sec:quartic-two-cycle}. The derivative of the quartic gradient map is
\begin{align*}
    g_a'(x)=1+a-3ax^2.
\end{align*}
For the period-two orbit \(x_-\mapsto x_+\mapsto x_-\), the multiplier is
\begin{align*}
    \mu_{\mathrm{cyc}}(a)
    =
    g_a'(x_-)g_a'(x_+).
\end{align*}
Therefore
\begin{align*}
    \mu_{\mathrm{cyc}}(a)
    &=
    \bigl(1+a-3a x_-^2\bigr)
    \bigl(1+a-3a x_+^2\bigr)\\
    &=
    (1+a)^2
    -
    3a(1+a)(x_-^2+x_+^2)
    +
    9a^2x_-^2x_+^2 .
\end{align*}
From the period-two identities,
\begin{align*}
    x_-+x_+
    =
    \sqrt{1+\frac3a},
    \qquad
    x_-x_+
    =
    \frac1a,
\end{align*}
we get
\begin{align*}
    x_-^2+x_+^2
    =
    (x_-+x_+)^2-2x_-x_+
    =
    1+\frac3a-\frac2a
    =
    1+\frac1a,
\end{align*}
and
\begin{align*}
    x_-^2x_+^2
    =
    \frac1{a^2}.
\end{align*}
Substitution gives
\begin{align*}
    \mu_{\mathrm{cyc}}(a)
    &=
    (1+a)^2
    -
    3a(1+a)\left(1+\frac1a\right)
    +
    9\\
    &=
    (1+a)^2
    -
    3(1+a)^2
    +
    9\\
    &=
    9-2(1+a)^2.
\end{align*}
This proves
\begin{align*}
    \mu_{\mathrm{cyc}}(a)=9-2(1+a)^2.
\end{align*}

\subsection{Sharpness Straddling Along the Two-Cycle}
\label{sec:appendix-sharpness-hovering}

We prove the sharpness-straddling relation used in Section~\ref{sec:quartic-sharpness-hovering}. Since
\begin{align*}
    \ell''(x)=3x^2-1,
\end{align*}
the local edge-of-stability threshold
\begin{align*}
    \ell''(x)=\frac2a
\end{align*}
is equivalent to
\begin{align*}
    x^2=\frac13\left(1+\frac2a\right).
\end{align*}
Thus it suffices to prove
\begin{align*}
    x_-^2
    <
    \frac13\left(1+\frac2a\right)
    <
    x_+^2 .
\end{align*}

Using the explicit formula
\begin{align*}
    x_\pm^2
    =
    \frac1{2a}
    \left[
        a+1
        \pm
        \sqrt{(a-1)(a+3)}
    \right],
\end{align*}
we first check the upper inequality. We need
\begin{align*}
    \frac13\left(1+\frac2a\right)
    <
    \frac1{2a}
    \left[
        a+1
        +
        \sqrt{(a-1)(a+3)}
    \right].
\end{align*}
Multiplying by \(6a>0\), this is equivalent to
\begin{align*}
    2(a+2)
    <
    3(a+1)+3\sqrt{(a-1)(a+3)}.
\end{align*}
Equivalently,
\begin{align*}
    0
    <
    a-1+3\sqrt{(a-1)(a+3)},
\end{align*}
which holds for \(a>1\).

For the lower inequality, we need
\begin{align*}
    \frac1{2a}
    \left[
        a+1
        -
        \sqrt{(a-1)(a+3)}
    \right]
    <
    \frac13\left(1+\frac2a\right).
\end{align*}
Multiplying again by \(6a>0\), this is equivalent to
\begin{align*}
    3(a+1)-3\sqrt{(a-1)(a+3)}
    <
    2(a+2),
\end{align*}
or
\begin{align*}
    a-1
    <
    3\sqrt{(a-1)(a+3)}.
\end{align*}
For \(a>1\), both sides are nonnegative, and squaring gives
\begin{align*}
    (a-1)^2
    <
    9(a-1)(a+3),
\end{align*}
which is immediate. Hence
\begin{align*}
    x_-^2
    <
    \frac13\left(1+\frac2a\right)
    <
    x_+^2 .
\end{align*}
This proves
\begin{align*}
    \ell''(x_-)
    <
    \frac2a
    <
    \ell''(x_+).
\end{align*}

Finally, we compute the average curvature. From the two-cycle identities,
\begin{align*}
    x_-^2+x_+^2
    =
    1+\frac1a.
\end{align*}
Therefore
\begin{align*}
    \frac{\ell''(x_-)+\ell''(x_+)}2
    &=
    \frac12\left(3x_-^2-1+3x_+^2-1\right)\\
    &=
    \frac32(x_-^2+x_+^2)-1\\
    &=
    \frac32\left(1+\frac1a\right)-1\\
    &=
    \frac12+\frac{3}{2a}.
\end{align*}
Its difference from the nominal edge value is
\begin{align*}
    \left(
        \frac12+\frac{3}{2a}
    \right)
    -
    \frac2a
    =
    \frac12\left(1-\frac1a\right).
\end{align*}

\subsection{Period-Doubling Thresholds}
\label{sec:appendix-period-doubling}

We record how the later period-doubling thresholds are defined. Let
\begin{align*}
    x_0\mapsto x_1\mapsto\cdots\mapsto x_{m-1}\mapsto x_0
\end{align*}
be an attracting \(m\)-cycle of the quartic map \(g_a\). Its multiplier is the derivative of the \(m\)-step return map:
\begin{align*}
    (g_a^m)'(x_0)
    =
    \prod_{j=0}^{m-1} g_a'(x_j).
\end{align*}
The cycle loses stability by a flip bifurcation when this multiplier reaches \(-1\):
\begin{align*}
    \prod_{j=0}^{m-1} g_a'(x_j)=-1.
\end{align*}
Equivalently, the period-doubling parameters are obtained by solving
\begin{align*}
    g_a^m(x_0)=x_0,
    \qquad
    (g_a^m)'(x_0)=-1,
\end{align*}
for \(m=2,4,8,\ldots\), following the attracting branch.

For \(m=2\), this calculation can be done explicitly and gives
\begin{align*}
    a_2=\sqrt5-1.
\end{align*}
For the later thresholds, the equations are solved numerically. The first values are
\begin{align*}
\begin{array}{ccl}
2\to4  &:& a_2=\sqrt5-1=1.2360679\ldots,\\[1mm]
4\to8  &:& a_4=1.2880317\ldots,\\[1mm]
8\to16 &:& a_8=1.2992279\ldots,\\[1mm]
16\to32&:& a_{16}=1.3016289\ldots,\\[1mm]
32\to64&:& a_{32}=1.3021432\ldots .
\end{array}
\end{align*}
The consecutive gaps shrink geometrically. More precisely,
\begin{align*}
    \frac{a_4-a_2}{a_8-a_4},
    \quad
    \frac{a_8-a_4}{a_{16}-a_8},
    \quad
    \frac{a_{16}-a_8}{a_{32}-a_{16}},
    \quad \ldots
\end{align*}
approach the Feigenbaum constant. The cascade accumulates at
\begin{align*}
    a_\infty
    =
    1.3022834\ldots .
\end{align*}

\section{Derivations for Section~\ref{sec:depth-limit}: The Depth Limit}
\label{app:depth}

\subsection{Derivation of the Scaling Limit}
\label{app:scaling-around-minimum}

We derive the depth-limit map used in Section~\ref{sec:depth-limit}. The finite-depth gradient descent map is
\begin{align*}
    x_{k+1}
    =
    x_k-a\,x_k^{2n-1}(x_k^{2n}-1).
\end{align*}
Set
\begin{align*}
    u_k=x_k^{2n},
    \qquad
    a=\frac cn .
\end{align*}
Then
\begin{align*}
    x_k=u_k^{1/(2n)}
\end{align*}
and
\begin{align*}
    x_{k+1}
    =
    x_k
    \left(
        1-\frac cn\,x_k^{2n-2}(x_k^{2n}-1)
    \right).
\end{align*}
Raising to the power \(2n\) gives the exact quotient map
\begin{align}
    u_{k+1}
    =
    h_{c,n}(u_k),
    \qquad
    h_{c,n}(u)
    =
    u
    \left(
        1-\frac cn\,u^{(n-1)/n}(u-1)
    \right)^{2n}.
    \label{eq:app-depth-quotient-map}
\end{align}

For fixed \(u\) in a compact subset \(K\subset(0,\infty)\),
\begin{align*}
    u^{(n-1)/n}
    =
    u\,u^{-1/n}
    =
    u\left(1+O_K\left(\frac1n\right)\right).
\end{align*}
Hence
\begin{align*}
    1-\frac cn\,u^{(n-1)/n}(u-1)
    =
    1-\frac cn\,u(u-1)
    +
    O_K\left(\frac1{n^2}\right).
\end{align*}
Taking logarithms,
\begin{align*}
    2n
    \log\left(
        1-\frac cn\,u^{(n-1)/n}(u-1)
    \right)
    =
    -2c\,u(u-1)
    +
    O_K\left(\frac1n\right).
\end{align*}
Therefore
\begin{align*}
    h_{c,n}(u)
    =
    u\exp\{2c\,u(1-u)\}
    \left(
        1+O_K\left(\frac1n\right)
    \right),
\end{align*}
uniformly on compact subsets of \((0,\infty)\). Thus
\begin{align*}
    h_{c,n}(u)
    =
    h_c(u)
    +
    O_K\left(\frac1n\right),
    \qquad
    h_c(u)=u\exp\{2c\,u(1-u)\}.
\end{align*}

In the additive quotient coordinate
\begin{align*}
    w=\log u=2n\log |x|,
\end{align*}
the limiting map is
\begin{align*}
    \Phi_c(w)
    =
    \log h_c(e^w)
    =
    w+2c e^w(1-e^w).
\end{align*}
Equivalently,
\begin{align*}
    \Phi_c=\log\circ h_c\circ\exp .
\end{align*}
This is the additive-coordinate representative of the Ricker-type quotient map \(h_c\).

Finally,
\begin{align*}
    h_c'(1)=\Phi_c'(0)=1-2c.
\end{align*}
Thus the limiting fixed point is locally attracting precisely for \(0<c<1\), and it loses stability at \(c=1\). In the original variables this is exactly the finite-depth threshold \(a=1/n\).
\subsection{Negative Schwarzian of the Ricker Limit}
\label{app:ricker-schwarzian}

\begin{proof}[Proof of Proposition~\ref{prop:ricker-schwarzian}]
Write
\begin{align*}
    h_c=e^\psi,
    \qquad
    \psi(u)=\log u+2cu(1-u).
\end{align*}
Using the composition rule
\begin{align*}
    S(e^\psi)=S\psi-\frac12(\psi')^2,
\end{align*}
set
\begin{align*}
    A=u\psi'=1+2cu-4cu^2,
    \qquad
    B=-u^2\psi''=1+4cu^2 .
\end{align*}
A direct calculation gives
\begin{align*}
    2u^4(\psi')^2\,Sh_c
    =
    \bigl(4A-A^4\bigr)-3B^2 .
\end{align*}
Since
\begin{align*}
    4A-A^4\le 3
\end{align*}
for all real \(A\), while \(B>1\) for \(u>0\), the right-hand side is strictly negative. Hence
\begin{align*}
    Sh_c(u)<0
\end{align*}
for all \(u>0\) with \(h_c'(u)\neq0\).
\end{proof}

\section{Derivations for Section~\ref{sec:middle-ground}: The Middle Ground}
\label{app:middle}

\subsection{Negative Schwarzian at Finite Depth}
\label{app:finite-depth-schwarzian}

\begin{proof}[Proof of Proposition~\ref{prop:schwarzian-all-n}]
Write \(\varepsilon=1/n\) and \(h=e^\psi\) with
\(\psi=\log u+\tfrac2\varepsilon\log b_{c,n}\). As in the limiting case,
\begin{align*}
    2u^4\psi'^2\,Sh
    =
    2CA-3B^2-A^4=:F,
    \qquad
    A=u\psi',\quad B=-u^2\psi'',\quad C=u^3\psi''' .
\end{align*}
Set \(\kappa=2c\,u^{1-\varepsilon}/b_{c,n}>0\), where \(b_{c,n}(u)=1-\varepsilon c\,u^{1-\varepsilon}(u-1)\) is the branch factor of Section~\ref{sec:critical-orbits-singer}, and
\begin{align*}
    \hat x=(2-\varepsilon)u-(1-\varepsilon),
    \quad
    \hat y=(1-\varepsilon)\bigl[(2-\varepsilon)u+\varepsilon\bigr]\ge0,
    \quad
    \tilde z=\varepsilon(1-\varepsilon)\bigl[(2-\varepsilon)u+1+\varepsilon\bigr]\ge0 .
\end{align*}
A direct computation gives
\(A=1-\kappa\hat x\),
\(B=1+\kappa\hat y+\tfrac\varepsilon2\kappa^2\hat x^2\),
\(C=2+\kappa\tilde z-\tfrac{3\varepsilon}2\kappa^2\hat x\hat y-\tfrac{\varepsilon^2}2\kappa^3\hat x^3\),
and hence the exact factorization
\begin{align*}
    F=\kappa\left[
        f_1+f_2\,\kappa
        +\Bigl(1-\tfrac{\varepsilon^2}4\Bigr)\hat x^3\kappa^2\,(4-\kappa\hat x)
    \right],
\end{align*}
with
\begin{align*}
    f_1=2\tilde z-6\hat y
    =-2(1-\varepsilon)(2-\varepsilon)\bigl[(3-\varepsilon)u+\varepsilon\bigr]\le0,
    \qquad
    f_2=-3(2+\varepsilon)\hat x^2-3\hat y^2-3\varepsilon\hat x\hat y-2\hat x\tilde z .
\end{align*}
If \(\hat x\ge0\): since \(X(4-X)\le4\) for \(X=\kappa\hat x\), the last term
is at most \(4\kappa\hat x^2\), and
\(f_2+4\hat x^2=-(2+3\varepsilon)\hat x^2-3\hat y^2
-3\varepsilon\hat x\hat y-2\hat x\tilde z<0\),
so \(F/\kappa\le f_1+\kappa(f_2+4\hat x^2)<0\).
If \(\hat x<0\): the last term is negative, \(f_1<0\), and \(f_2<0\) as well.
Writing \(\beta=1-\varepsilon\), \((2-\varepsilon)u=\beta-\xi\) with
\(\xi\in(0,\beta]\), and \(t=(1-\xi)/\xi\),
\begin{align*}
    -\frac{f_2}{\xi^2}
    =
    3\beta^2t^2-7\beta(1-\beta)\,t+\bigl(9-5\beta+2\beta^2\bigr),
\end{align*}
a quadratic in \(t\) with discriminant
\(\beta^2\bigl(25\beta^2-38\beta-59\bigr)<0\) on \((0,1]\).
Thus every term of \(F/\kappa\) is negative, and \(F<0\).
The case \(\varepsilon=0\) recovers the Ricker limit of
Appendix~\ref{app:ricker-schwarzian}.
\end{proof}

\section{Derivations for Section~\ref{sec:beyond-scalar-chain}: The Two-Factor Model}
\label{app:beyond-scalar-chain}

This appendix collects the calculations behind Section~\ref{sec:beyond-scalar-chain}. Throughout, \(a\) is the actual step size of the two-factor gradient descent map \eqref{eq:two-factor-gd}.

\subsection{Proof of Proposition~\ref{prop:finite-step-balancing}}
\label{app:finite-step-balancing-proof}

\emph{Exact closed coordinates.} Let
\begin{align*}
    r=xy-1,
    \qquad
    x=s+d,
    \qquad
    y=s-d ,
\end{align*}
matching the diagonal coordinates of Section~\ref{sec:transverse-stability}. Then \(xy=s^2-d^2\), and the update \eqref{eq:two-factor-gd} gives
\begin{align*}
    x^+ &= (s+d)-a r(s-d), \\
    y^+ &= (s-d)-a r(s+d).
\end{align*}
Adding and subtracting yields the diagonal form
\begin{align*}
    s^+=s(1-a r),
    \qquad
    d^+=d(1+a r).
\end{align*}
Since \(\Delta=x^2-y^2=4sd\), we obtain the exact imbalance identity
\begin{align}
    \Delta^+=(1-a^2r^2)\Delta ,
    \label{eq:app-imbalance-update}
\end{align}
which is \eqref{eq:imbalance-update}. No linearization is involved. This is why \(\Delta\) is the right transverse variable: in continuous-time gradient flow the corresponding equation is \(\dot\Delta=0\), while in discrete time the imbalance is multiplied by a finite-step factor.

The residual also closes with \(\Delta\). Let \(S=x^2+y^2\). Because
\begin{align*}
    S^2=(x^2-y^2)^2+4x^2y^2=\Delta^2+4(1+r)^2,
\end{align*}
we have
\begin{align*}
    S=\sqrt{\Delta^2+4(1+r)^2} .
\end{align*}
Furthermore,
\begin{align*}
    x^+y^+&=(x-a r y)(y-a r x) \\
    &=xy-a r(x^2+y^2)+a^2r^2xy \\
    &=1+r-a rS+a^2r^2(1+r).
\end{align*}
Thus
\begin{align}
    r^+=r\Bigl(1-a S+a^2r(1+r)\Bigr),
    \qquad
    S=\sqrt{\Delta^2+4(1+r)^2} .
    \label{eq:app-r-update}
\end{align}
Equations \eqref{eq:app-imbalance-update} and \eqref{eq:app-r-update} form an autonomous two-dimensional system in \((r,\Delta)\).

\emph{Stable arc.} On the valley \(r=0\), the Hessian is
\begin{align*}
    \nabla^2 \ell(x,y)
    =
    \begin{pmatrix}
        y^2 & xy \\
        xy & x^2
    \end{pmatrix}.
\end{align*}
It has one zero tangential eigenvalue and one normal eigenvalue
\(x^2+y^2=\sqrt{\Delta^2+4}\). The normal multiplier of gradient descent is therefore \(1-a\lambda(\Delta)\), and normal stability is equivalent to
\begin{align*}
    |1-a\lambda(\Delta)|<1
    \qquad\Longleftrightarrow\qquad
    a\sqrt{\Delta^2+4}<2 .
\end{align*}
This gives the stable arc
\begin{align*}
    |\Delta|<\Delta^*(a),
    \qquad
    \Delta^*(a)=\sqrt{\frac{4}{a^2}-4},
\end{align*}
which is nonempty exactly for \(a<1\). At \(a=1\), only the balanced point reaches the boundary, with multiplier \(-1\). For \(a>1\), even the balanced point has \(a\lambda(0)>2\), so no valley point is normally stable.

\emph{Small-step endpoint.} Iterating the exact imbalance identity gives
\begin{align*}
    \log\frac{|\Delta_\infty|}{|\Delta_0|}
    =
    \sum_k \log(1-a^2r_k^2)
    =
    -a\left(a\sum_k r_k^2\right)+O(a^2).
\end{align*}
As \(a\to0\), the discrete trajectory is evaluated along gradient flow. Along gradient flow,
\begin{align*}
    \dot x=-r y,
    \qquad
    \dot y=-r x,
    \qquad\text{hence}\qquad
    \dot\Delta=2x\dot x-2y\dot y=0 ,
\end{align*}
so \(\Delta(t)=\Delta_0\). With \(w=xy=1+r\) and \(\delta=|\Delta_0|\),
\begin{align*}
    \dot w
    =
    \dot x\,y+x\,\dot y
    =
    -r(x^2+y^2)
    =
    -(w-1)\sqrt{\delta^2+4w^2}.
\end{align*}
Consequently,
\begin{align*}
    a\sum_k r_k^2
    \;\longrightarrow\;
    \int_0^\infty r(t)^2\,dt
    =
    \int_1^{1+r_0}
    \frac{w-1}{\sqrt{\delta^2+4w^2}}\,dw
    =
    J(\delta,r_0),
\end{align*}
with the closed form
\begin{align*}
    J(\delta,r_0)
    =
    \left[\frac14\sqrt{\delta^2+4w^2}-\frac12\operatorname{asinh}\!\left(\frac{2w}{\delta}\right)\right]_{w=1}^{w=1+r_0} .
\end{align*}
This proves \eqref{eq:smallstep-balancing-main}; since the integrand is positive for \(w>1\), one has \(J(\delta,r_0)>0\). The diagonal case \(\delta=0\) is singular only because it is already exactly balanced; then \(\Delta_k\equiv0\). \hfill\(\square\)

\begin{figure}[t]
\centering
\includegraphics[width=\textwidth]{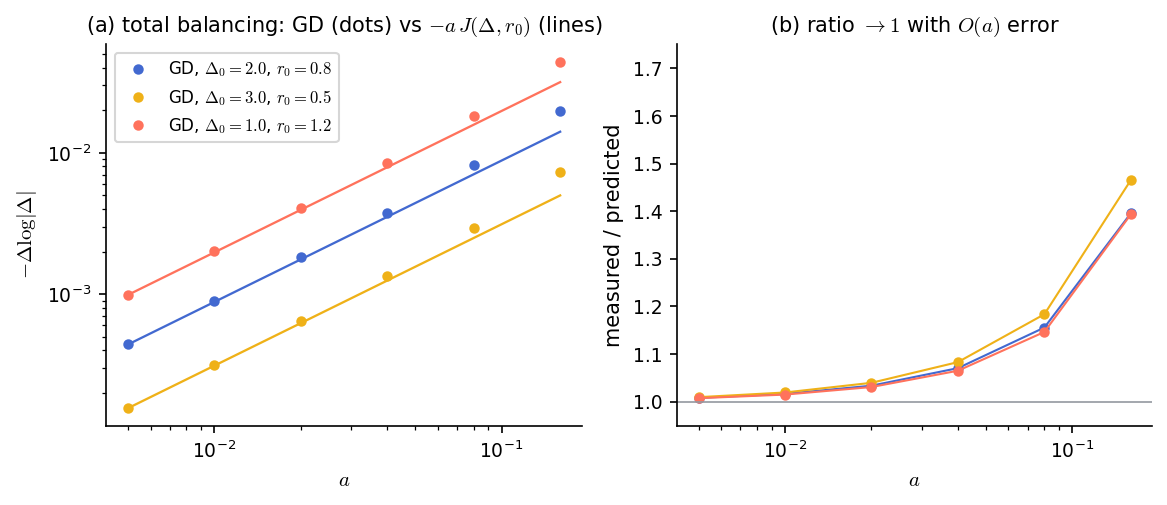}
\caption{Small-step balancing. Dots show the measured total imbalance shift under gradient descent; lines show the closed form \(-aJ(\delta,r_0)\). The ratio of measured to predicted shift approaches \(1\) linearly in the step size \(a\).}
\label{fig:app-smallstep}
\end{figure}

\subsection{Averaged Minimum Selection Beyond the Stable Arc}
\label{app:selection-law}

We derive the slow system \eqref{eq:slow-system-main}, the selection law \eqref{eq:selection-main}, and the fixed-step reflection coefficient \eqref{eq:c-of-a} of Section~\ref{sec:balancing-oscillation}. Assume \(0<a<1\) and start on, or just beyond, the stability boundary of the valley. Write
\begin{align*}
    \varepsilon=a\lambda(\Delta)-2,
    \qquad
    \lambda(\Delta)=\sqrt{\Delta^2+4},
    \qquad
    \mu_a=\frac{(\Delta^*(a))^2}{2}=\frac{2(1-a^2)}{a^2},
\end{align*}
and work in the near-edge scaling \(\varepsilon=O(\mu_a)\), \(r=O(\sqrt{\mu_a})\) as \(a\uparrow1\).

\emph{Normal form.} Expanding the exact residual update \eqref{eq:app-r-update} about the edge \(a\lambda=2\), with \(\Delta\) held fixed across a single step,
\begin{align}
    r^+=-(1+\varepsilon)\,r-a^2r^2+a^4r^3+O(r^4,\varepsilon r^2).
    \label{eq:app-normal-form}
\end{align}
The factor \(-(1+\varepsilon)\) reverses the sign of \(r\) at each step while its magnitude drifts slowly, so the slow variable is the squared residual, which we rescale as
\begin{align*}
    \alpha=2a^4 r^2 .
\end{align*}

\emph{Two-step envelope.} Compose \eqref{eq:app-normal-form} with itself. Over two steps \(r\) returns to its own sign, \(|r|\) is multiplied by \((1+\varepsilon)^2\) to linear order, and the quadratic terms combine into a saturation; a direct expansion gives
\begin{align}
    \frac{d\alpha}{d\tau}=4\alpha(\varepsilon-\alpha)+O(\mu_a^2),
    \label{eq:app-envelope}
\end{align}
where \(\tau\) counts residual oscillations (one per two steps). The odd cross term \(-2a^2\varepsilon r\) generated by the composition reverses sign with \(r\) and cancels between consecutive steps; the envelope therefore depends on \(r\) only through \(\alpha\), which is why the selected endpoint is independent of the sign of the initial seed.

\emph{Coupling to the excess.} The exact imbalance identity \eqref{eq:app-imbalance-update} gives \(d\log|\Delta|/d\tau=-a^2r^2\cdot2=-\alpha/a^2\) per oscillation. Differentiating \(\varepsilon=a\lambda-2\) along the valley,
\begin{align*}
    \frac{d\varepsilon}{d\log|\Delta|}
    =
    \frac{a\Delta^2}{\lambda}
    =
    \frac{(2+\varepsilon)^2-4a^2}{2+\varepsilon}
    =
    a^2\mu_a+2\varepsilon+O(\mu_a^2),
\end{align*}
using \(2a^2\mu_a=4(1-a^2)\). Multiplying, and using \(a^2\to1\) at the edge so that \(2\varepsilon/a^2=2\varepsilon+O(\mu_a^2)\),
\begin{align}
    \frac{d\varepsilon}{d\tau}=-\alpha\,(\mu_a+2\varepsilon)+O(\mu_a^2).
    \label{eq:app-eps-drift}
\end{align}
Equations \eqref{eq:app-envelope} and \eqref{eq:app-eps-drift} are the slow system \eqref{eq:slow-system-main}.

\emph{Integration and the selection law.} Only the ratio of the two rates matters for the orbit. Eliminating \(\tau\),
\begin{align*}
    \frac{d\alpha}{d\varepsilon}=-\frac{4(\varepsilon-\alpha)}{\mu_a+2\varepsilon},
\end{align*}
which is linear in \(\alpha(\varepsilon)\) and solved by
\begin{align}
    \alpha
    =
    \frac{\mu_a}{2}+2\varepsilon+C\,(\mu_a+2\varepsilon)^2 .
    \label{eq:app-invariant}
\end{align}
A burst starts and ends with vanishing residual envelope, \(\alpha=0\); at such a point \(C=-\tfrac12 g_{\mu_a}(\varepsilon)\) with \(g_{\mu_a}\) as in \eqref{eq:gmu-definition}. Since \(C\) is conserved, the initial and final excesses satisfy
\begin{align}
    g_{\mu_a}(\varepsilon_\infty)=g_{\mu_a}(\varepsilon_0),
    \qquad
    g_{\mu_a}(\varepsilon)=\frac{\mu_a+4\varepsilon}{(\mu_a+2\varepsilon)^2},
    \label{eq:app-selection}
\end{align}
which is \eqref{eq:selection-main}, with the unique root \(\varepsilon_\infty\in(-\mu_a/4,0)\) for \(\varepsilon_0>0\). The selected valley point is \(\lambda_\infty=(2+\varepsilon_\infty)/a\), \(|\Delta_\infty|=\sqrt{\lambda_\infty^2-4}\). Because both endpoints have \(\alpha=0\), the selected minimum is independent, to leading order, of the initial residual seed; and since the burst cannot flatten past the balanced point \(\varepsilon_{\rm bal}=2a-2\), a predicted \(\varepsilon_\infty<\varepsilon_{\rm bal}\) marks the catapult boundary.

\emph{Fixed-step reflection coefficient.} For small \(\varepsilon_0\) the scaling-limit law \eqref{eq:app-selection} reflects almost symmetrically. Writing \(g_{\mu_a}(\varepsilon)=\mu_a^{-1}\bigl(1-4X^2+16X^3+\cdots\bigr)\) with \(X=\varepsilon/\mu_a\) and solving \(g_{\mu_a}(\varepsilon_\infty)=g_{\mu_a}(\varepsilon_0)\) gives
\begin{align*}
    \varepsilon_\infty=-\varepsilon_0+\frac{4}{\mu_a}\varepsilon_0^2+O(\varepsilon_0^3).
\end{align*}
At fixed \(a<1\) the exact map is not the continuous flow: the two-step composition adds discretization corrections to the invariant \eqref{eq:app-invariant}, of relative size \(O(\varepsilon_0)\). Carrying these corrections shifts the quadratic coefficient by a universal amount \(-\tfrac23\), independent of \(a\), so that
\begin{align*}
    \varepsilon_\infty=-\varepsilon_0+c(a)\,\varepsilon_0^2+O(\varepsilon_0^3),
    \qquad
    c(a)=\frac{4}{\mu_a}-\frac23=\frac{2(4a^2-1)}{3(1-a^2)},
\end{align*}
which is \eqref{eq:c-of-a}. In particular \(c(1/2)=0\): at \(a=1/2\) the burst reflects exactly to second order. Direct iteration of the exact map confirms \(c(a)\) to the digits shown: \(c(0.3)=-0.469\), \(c(0.5)=0.000\), \(c(0.6)=0.458\), \(c(0.8)=2.889\).

For a worked example away from the edge, take \(a=0.6\) and \(\Delta_0=3.2\), so \(\varepsilon_0\approx0.264\) and \(\mu_a=32/9\). The selection law \eqref{eq:app-selection} predicts \(\varepsilon_\infty\approx-0.204\), hence \(|\Delta_\infty|\approx2.23\), while direct gradient descent gives \(\approx2.15\); the stable boundary is \(\Delta^*\approx2.67\). The \(4\%\) gap is the expected \(O(\mu_a^2)\) error of the leading-order law at a step size this far below the edge.

\begin{figure}[t]
\centering
\includegraphics[width=\textwidth]{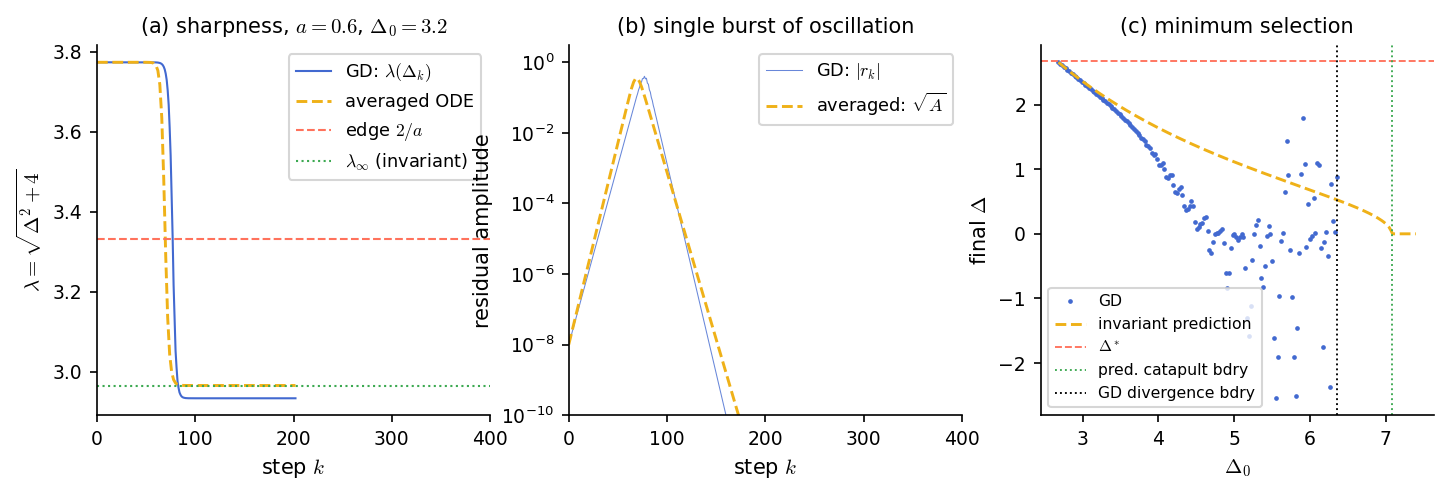}
\caption{Post-edge drift at step size \(a=0.6\). The trajectory performs a single residual burst, burns excess instability, and returns to the valley inside the stable arc. The selection law predicts the final sharpness and the onset of the catapult regime.}
\label{fig:app-eos-drift}
\end{figure}

\begin{figure}[t]
\centering
\includegraphics[width=\textwidth]{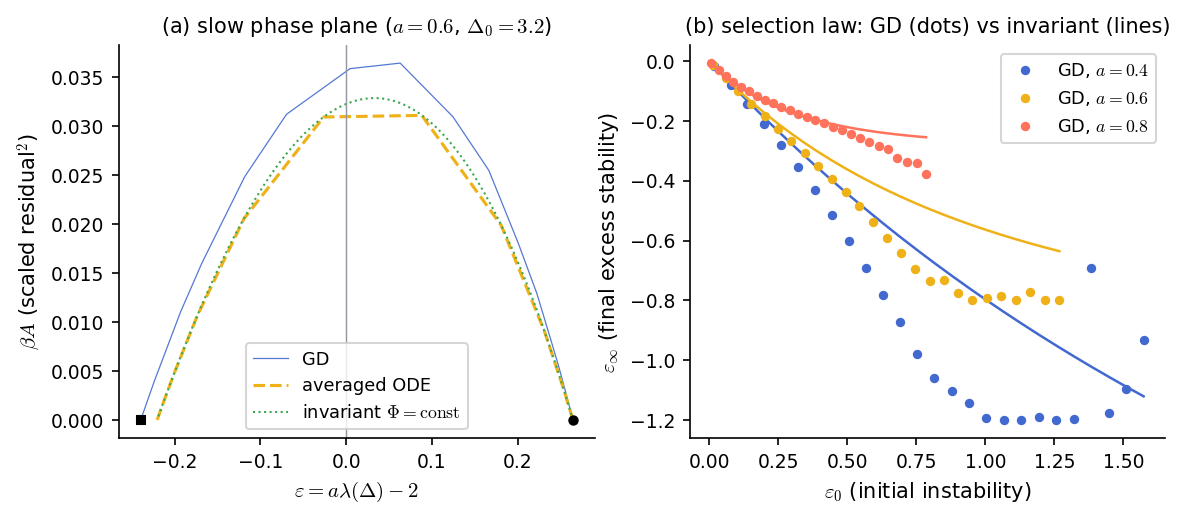}
\caption{Selection law. The averaged phase plane follows the integral curve \eqref{eq:app-invariant} from \((\varepsilon_0,0)\) to \((\varepsilon_\infty,0)\). Across steps, direct gradient descent agrees with the implicit law \eqref{eq:app-selection} until the burst reaches the hard floor imposed by the balanced point.}
\label{fig:app-selection-law}
\end{figure}

\subsection{Diagonal Cascade and Exact Crisis}
\label{app:diagonal-cascade}

On the diagonal, the map is
\begin{align*}
    g_a(s)=s-a s(s^2-1),
\end{align*}
in the diagonal coordinate \(s\) of Section~\ref{sec:transverse-stability}. The first stability edge is \(a=1\). The stable period-two orbit satisfies the exact identities
\begin{align}
    s_-^2+s_+^2=1+\frac1a,
    \qquad
    s_-s_+=\frac1a .
    \label{eq:app-two-cycle-identities}
\end{align}
Continuation of the principal \(2^j\)-cycles gives the following flip thresholds:
\begin{center}
\begin{tabular}{c|ccccc}
transition & \(1\to2\) & \(2\to4\) & \(4\to8\) & \(8\to16\) & \(16\to32\) \\
\hline
\(a\) & \(1\) & \(\sqrt5-1\) & \(1.28803175\) & \(1.29922794\) & \(1.30162891\)
\end{tabular}
\end{center}
The ratios of successive gaps converge to Feigenbaum's constant, and the accumulation point is \(a_\infty\approx1.3022834\).

The bounded diagonal attractor is destroyed in a boundary crisis at the exact step size \(a=2\). The scalar map has, for every \(a>0\), the antisymmetric two-cycle
\begin{align*}
    \{+A_a,-A_a\},
    \qquad
    A_a=\sqrt{1+\frac2a} .
\end{align*}
If \(|s|>A_a\), the orbit diverges monotonically in magnitude; hence \([-A_a,A_a]\) is the natural trapping interval. The critical point of \(g_a\) satisfies
\begin{align*}
    s_c^2=\frac{1+a}{3a},
    \qquad
    g_a(s_c)=\frac{2(1+a)}{3}s_c .
\end{align*}
The trapping condition \(g_a(s_c)\le A_a\) fails when
\begin{align}
    4(1+a)^3=27(a+2).
    \label{eq:app-crisis-equation}
\end{align}
The positive boundary relevant to the attractor is \(a=2\); both sides of \eqref{eq:app-crisis-equation} then equal \(108\).

\begin{figure}[t]
\centering
\includegraphics[width=\textwidth]{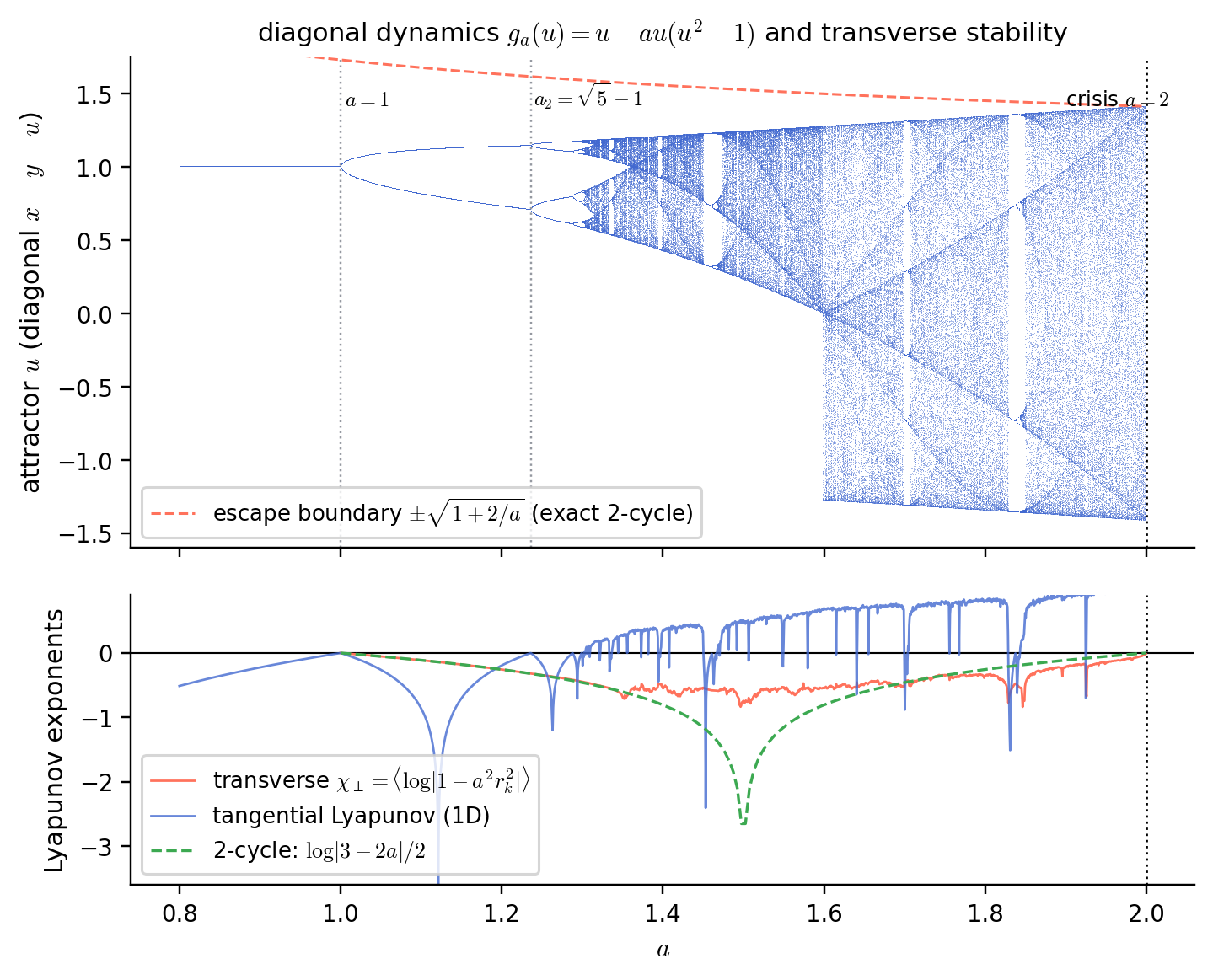}
\caption{Diagonal bifurcation diagram. The exact escape boundary is \(\pm\sqrt{1+2/a}\), and the attractor touches it at the crisis \(a=2\).}
\label{fig:app-bifurcation}
\end{figure}

\subsection{Transverse Lyapunov Exponent}
\label{app:transverse-exponent}

The same identity that gives balancing also gives the transverse Lyapunov exponent of the diagonal attractor. Along the diagonal \(s_k\), the residual is \(r_k=s_k^2-1\), and
\begin{align*}
    s_{k+1}=s_k(1-a r_k).
\end{align*}
Thus, for any bounded diagonal orbit that stays away from \(s=0\) in time average,
\begin{align*}
    \left\langle \log|1-a r_k|\right\rangle=\lim_{T\to\infty}\frac1T\log\left|\frac{s_T}{s_0}\right|=0 .
\end{align*}
Since the exact imbalance multiplier is \((1-a r_k)(1+a r_k)=1-a^2r_k^2\), we get
\begin{align*}
    \chi_\perp=\left\langle \log|1-a^2r_k^2|\right\rangle=\left\langle \log|1+a r_k|\right\rangle .
\end{align*}
For a periodic orbit this is simply the statement that \(\prod_k(1-a r_k)=1\) over the period. On the period-two branch, \eqref{eq:app-two-cycle-identities} gives the two-step transverse multiplier
\begin{align}
    \mu_\perp=3-2a .
    \label{eq:app-period-two-transverse}
\end{align}
It reaches \(-1\) exactly at the crisis \(a=2\).

At the crisis the transverse exponent can be computed in closed form. For \(a=2\),
\begin{align*}
    g_a(s)=3s-2s^3 .
\end{align*}
With \(s=\sqrt2\sin\varphi\), this becomes
\begin{align*}
    g_a(\sqrt2\sin\varphi)=\sqrt2\sin(3\varphi),
\end{align*}
so the invariant density on \([-\sqrt2,\sqrt2]\) is
\begin{align*}
    \rho(s)=\frac{1}{\pi\sqrt{2-s^2}} .
\end{align*}
The telescoped transverse factor is
\begin{align*}
    1+2r=1+2(s^2-1)=1-2\cos(2\varphi).
\end{align*}
Therefore
\begin{align}
    \chi_\perp(a=2)=\frac{1}{2\pi}\int_0^{2\pi}\log|1-2\cos\psi|\,d\psi=m(z^2-z+1)=0,
    \label{eq:app-mahler}
\end{align}
where \(m(\cdot)\) denotes logarithmic Mahler measure. The polynomial \(z^2-z+1\) is cyclotomic, which explains the final equality.

The exact endpoint information is consistent with numerical sweeps over the bounded interval \(1<a<2\): the measured transverse exponent remains negative until it approaches zero at the crisis. The approach is well fit by
\begin{align*}
    \chi_\perp(a)\approx -0.58\sqrt{2-a}
\end{align*}
over several decades in \(2-a\). A complete analytic proof of \(\chi_\perp<0\) throughout the interval would have to use the physical measure of the scalar attractor; it cannot be a statement about all invariant measures, because the escape cycle \(\{\pm A_a\}\) has positive transverse exponent.

\begin{figure}[t]
\centering
\includegraphics[width=\textwidth]{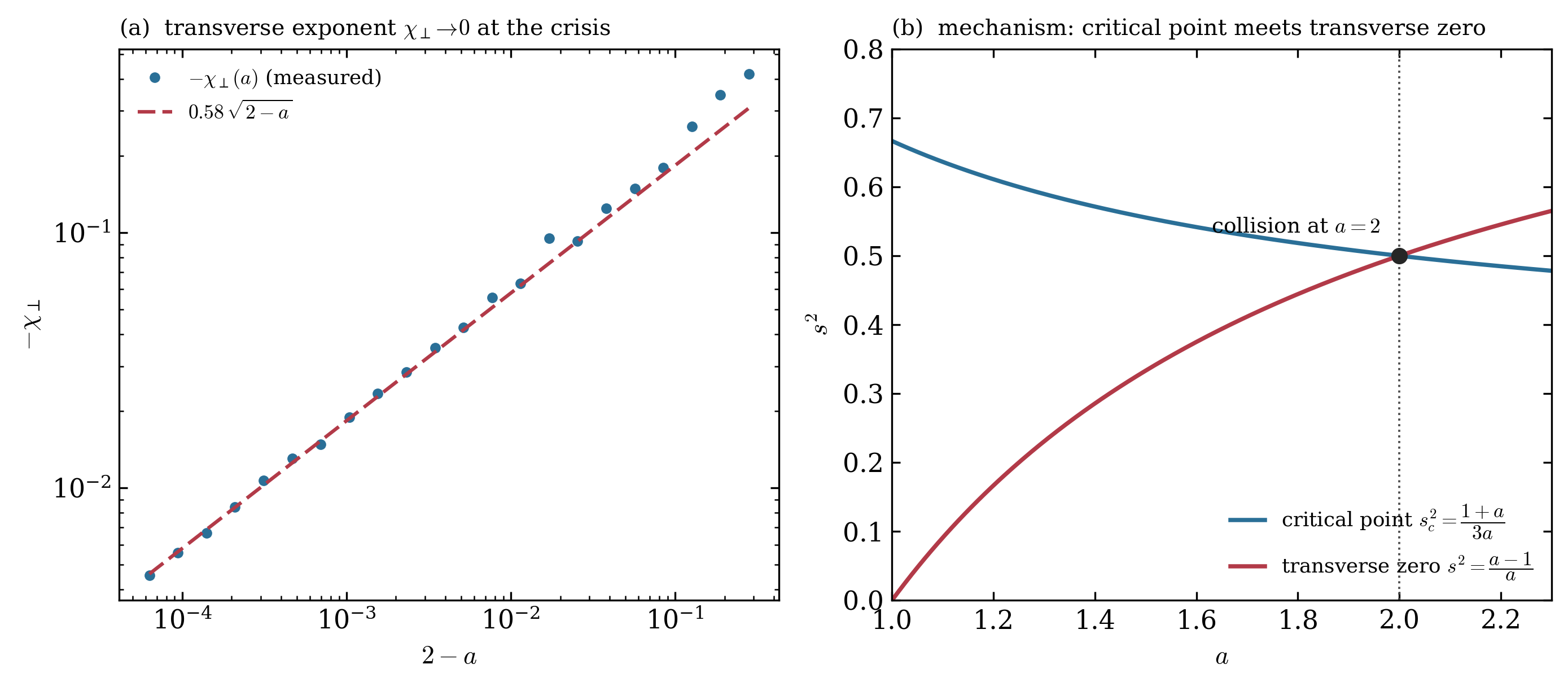}
\caption{Structure of the transverse exponent near the crisis. The exponent approaches zero with an apparent square-root law. At \(a=2\), the zero of the transverse factor collides with the critical point of the scalar map, and \eqref{eq:app-mahler} gives \(\chi_\perp=0\) exactly.}
\label{fig:app-crisis}
\end{figure}

\subsection{Regime Map and Representative Trajectories}
\label{app:regime-map}

The results above assemble into the following qualitative picture. For \(a<1\), the valley has a stable arc \(|\Delta|<\Delta^*(a)\). Initial points beyond the arc undergo a residual burst that reduces \(|\Delta|\), with the selected endpoint predicted by \eqref{eq:app-selection} near the edge. At larger initial imbalance, the same burst can overshoot the balanced point and lead to divergence. For \(1<a<2\), the whole valley is locally unstable, but bounded trajectories are attracted toward the diagonal scalar attractor. At \(a=2\), the scalar attractor itself is destroyed in the boundary crisis.

\begin{figure}[t]
\centering
\includegraphics[width=0.78\textwidth]{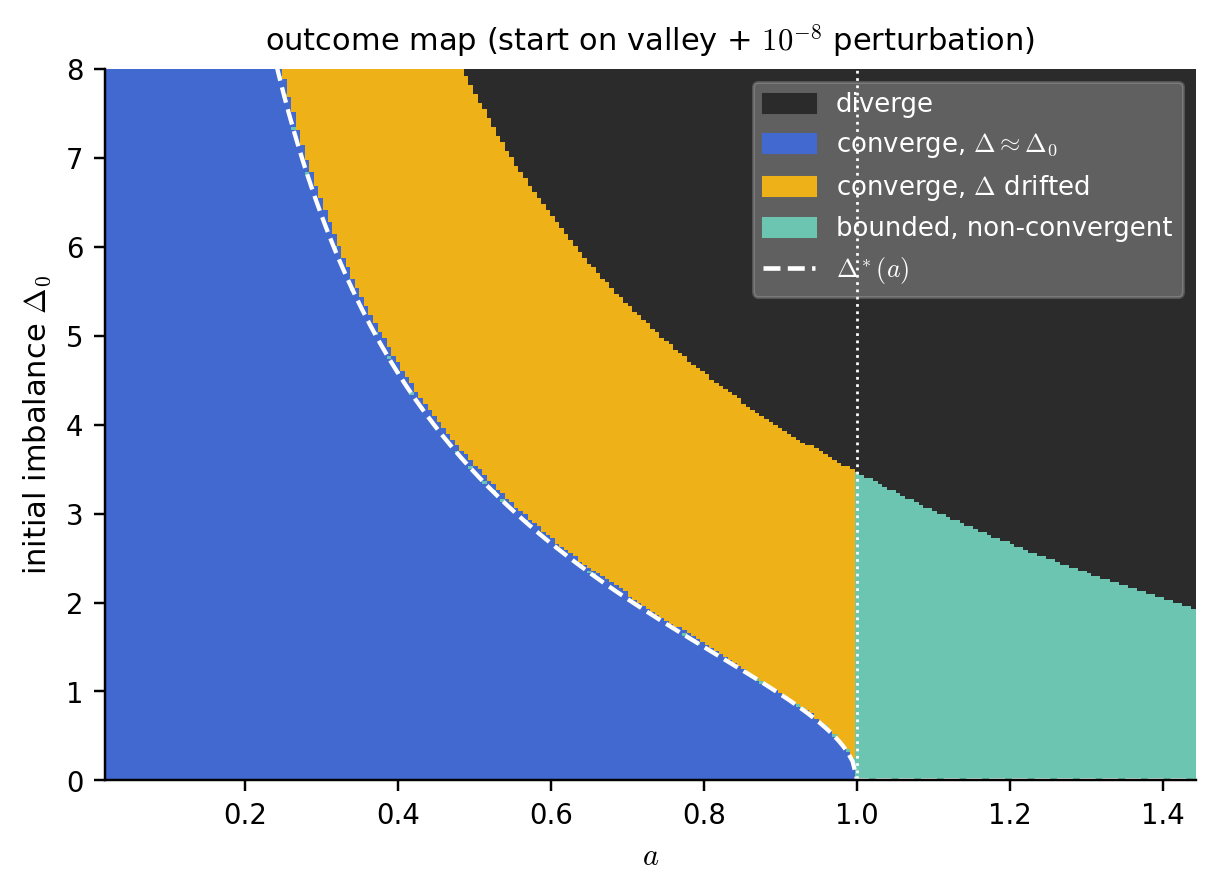}
\caption{Outcome map for starts on the valley with a small transverse seed. The stable arc, the drift region, the catapult boundary, and the post-edge diagonal attractor are the four visible regimes.}
\label{fig:app-regime-map}
\end{figure}

\begin{figure}[t]
\centering
\includegraphics[width=\textwidth]{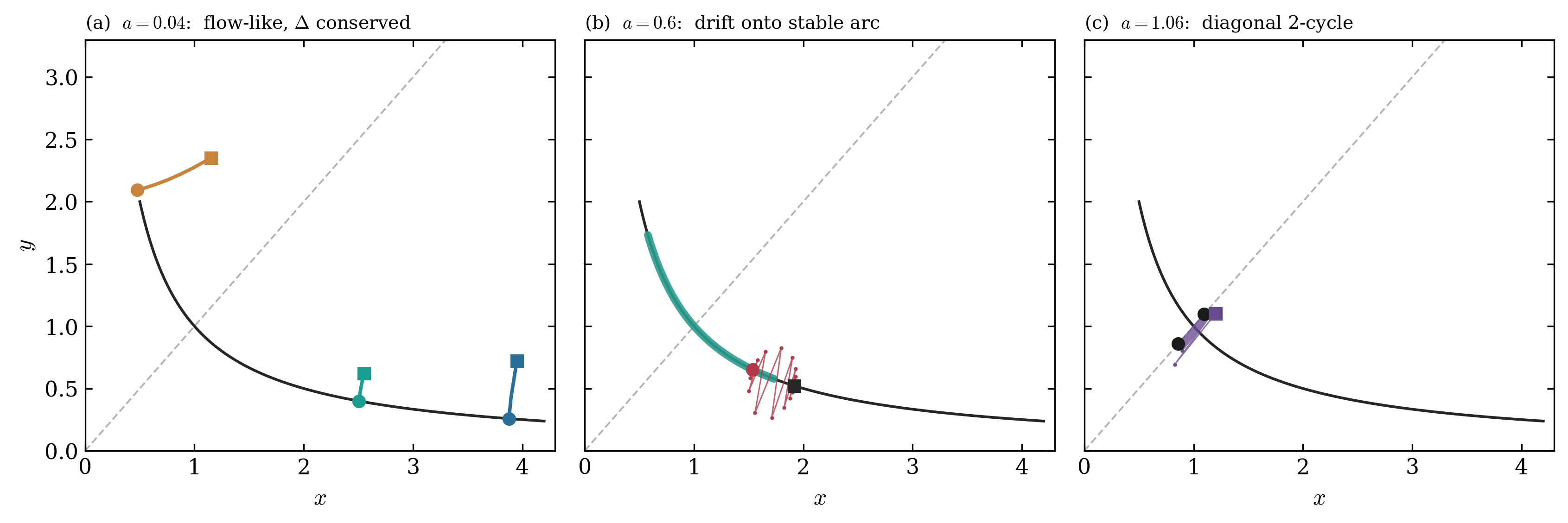}
\caption{Representative trajectories in the \((x,y)\) plane. At very small step size, imbalance is nearly conserved. Below the edge \(a=1\), an unstable valley point can drift onto the stable arc. Above the edge, bounded trajectories collapse onto the diagonal scalar attractor.}
\label{fig:app-trajectories}
\end{figure}

\section{Derivations for Section~\ref{sec:wider}: Toward Realistic Networks}
\label{sec:appendix-wider}

\subsection{Alignment and the Ladder of Edges}
\label{app:alignment-ladder}

\emph{Alignment is invariant.} With \(W_1=D_1V^\top\), \(W_l=D_l\) for \(1<l<L\), \(W_L=UD_L\), the residual is \(R=U(D_L\cdots D_1-S)V^\top\). The gradients are
\begin{align*}
    \nabla_{W_1}\ell
    =
    (W_L\cdots W_2)^\top R
    =
    \bigl[\textstyle\prod_{l\ge2}D_l\bigr](D_L\cdots D_1-S)\,V^\top,
\end{align*}
and similarly for the other layers: every gradient is of the same aligned form, with a diagonal factor. Hence gradient descent moves only the diagonals,
\begin{align*}
    (D_l)_{ii}
    \;\longmapsto\;
    (D_l)_{ii}
    -
    a\,\frac{\pi_i}{(D_l)_{ii}}\,(\pi_i-s_i),
    \qquad
    \pi_i=\prod_l (D_l)_{ii},
\end{align*}
which is the \(L\)-layer scalar chain for each mode separately. Within a mode, the balanced diagonal \((D_l)_{ii}=x_i\) is invariant, and the update reduces to
\begin{align*}
    x_i
    \;\longmapsto\;
    x_i-a\,x_i^{L-1}\bigl(x_i^{L}-s_i\bigr).
\end{align*}

\emph{Proof of Proposition~\ref{prop:spectral-edge-ladder}.} At the balanced minimum \(x_i=s_i^{1/L}\), the curvature of the mode loss \(\frac12(\pi_i-s_i)^2\) along the balanced diagonal is \(L\,s_i^{2-2/L}\). The gradient descent multiplier crosses \(-1\) when this curvature equals \(2/a\), giving the mode edge
\begin{align*}
    a_i=\frac{2}{Ls_i^{2-2/L}}.
\end{align*}
Since this is decreasing in \(s_i\), the largest singular value crosses first. \hfill\(\square\)

\subsection{Proof of Proposition~\ref{prop:rotational-edge}}
\label{app:rotational-edge-proof}

For two depth-two modes, linearize around the orbit where the first mode follows its two-cycle \(x_1\in\{x_-,x_+\}\) and the second is fixed at \(x_2=\sqrt{s_2}\). The off-diagonal directions split into two blocks
\begin{align*}
    M_\pm
    =
    I-a
    \begin{pmatrix}
        x_1^2 & x_1x_2\\
        x_1x_2 & x_2^2\pm r_1
    \end{pmatrix},
    \qquad
    r_1=x_1^2-s_1 .
\end{align*}
The \(+\) block contains the neutral internal factorization rotation: the transformation \(W_1\mapsto RW_1\), \(W_2\mapsto W_2R^\top\) commutes with gradient descent and moves only between equivalent factorizations of the same product. The physical frame misalignment is the \(-\) block. Using the two-cycle identities for the quartic map gives
\begin{align*}
    |\mu_{\mathrm{rot}}|^2
    =
    \det\bigl(M_-(x_+)M_-(x_-)\bigr)
    =
    D^2-(as_1-1)D-(as_1-1),
    \qquad
    D=1-a(s_1+s_2).
\end{align*}
This expression is less than one exactly when \(a(s_1+s_2)<2\). Requiring this for every \(a\) in the top-mode two-cycle window \(1<as_1<\sqrt5-1\) gives
\begin{align*}
    \frac{s_2}{s_1}<\frac{\sqrt5-1}{2}. \tag*{\(\square\)}
\end{align*}

\subsection{The Post-Edge Tuning Calculation}
\label{app:tuning}

At depth two, a mode with target \(s\) has effective parameter \(A=as\). Below the edge, the fixed-point multiplier is \(1-2A\). On the two-cycle, the two-step multiplier is \(9-2(1+A)^2\), so the per-step contraction is its square root in modulus. The two-mode minimax problem therefore reduces to
\begin{align*}
    \min_a \max\{\rho(as_1),\rho(as_2)\}.
\end{align*}
Solving the one-dimensional balance gives the numerical transition \(\sigma_c=0.5718\ldots\) and, for \(\sigma=1/2\), the post-edge optimum \(a^*s_1\approx1.120\).

\subsection{Data Coupling: Imbalance Law and Floquet Multipliers}
\label{app:data-coupling}

\emph{The imbalance law.} For \(\ell(x,y)=\frac12\|xy-d\|^2\), gradient descent is \(x^+=x-a\,y\,z\), \(y^+=y-a\,x\cdot z\) with \(z=xy-d\). Hence
\begin{align*}
    \|x^+\|^2=\|x\|^2-2ay\,(x\cdot z)+a^2y^2\|z\|^2,
    \qquad
    (y^+)^2=y^2-2ay\,(x\cdot z)+a^2(x\cdot z)^2 .
\end{align*}
The cross terms cancel in the difference, and writing
\(y^2\|z\|^2-(x\cdot z)^2=-\|z\|^2\Delta+\bigl(\|x\|^2\|z\|^2-(x\cdot z)^2\bigr)\)
gives \eqref{eq:vector-imbalance-update}. With per-mode scales \(\sigma_i\), the gradients carry \(z'=\sigma\odot z\) in place of \(z\), and the identical computation gives the same law in \(z'\).

\emph{Proof of Proposition~\ref{prop:data-two-cycle}.} On the aligned diagonal \(x=\xi\sqrt{\|d\|}\,\hat d\), \(y=\xi\sqrt{\|d\|}\), the update reduces to the quartic map \(\xi^+=\xi-A\xi(\xi^2-1)\) with \(A=a\|d\|\), so the longitudinal multiplier is \eqref{eq:quartic-cycle-multiplier} at parameter \(A\), namely \(\mu_{\mathrm{cyc}}=9-2(1+A)^2\).

Perturbing the alignment, \(x=\xi\sqrt{\|d\|}\,\hat d+\varepsilon e_\perp\) with \(e_\perp\perp\hat d\), the residual acquires the transverse component \(\varepsilon y\,e_\perp\), and
\begin{align*}
    \varepsilon^+
    =
    \varepsilon\,(1-a y^2)
    =
    \varepsilon\,(1-A\xi^2),
\end{align*}
while the feedback on \((\xi,y)\) is \(O(\varepsilon^2)\); each of the \(k-1\) transverse directions carries this same factor. Over the two-cycle, the identities \eqref{eq:app-two-cycle-identities} (with \(s_\pm=\xi_\pm\), at parameter \(A\)) give
\begin{align*}
    \mu_{\mathrm{align}}
    =
    (1-A\xi_-^2)(1-A\xi_+^2)
    =
    1-A\bigl(\xi_-^2+\xi_+^2\bigr)+A^2\xi_-^2\xi_+^2
    =
    1-A\Bigl(1+\frac1A\Bigr)+1
    =
    1-A .
\end{align*}
The imbalance direction is the two-factor transverse direction along \(\hat d\), with per-step factor \(1+Ar_\pm\) and two-step multiplier \(\mu_{\mathrm{bal}}=3-2A\) by \eqref{eq:app-period-two-transverse}. For \(1<A<\sqrt5-1\), all three multipliers lie in \((-1,1)\), so the orbit is linearly stable. \hfill\(\square\)

\subsection{The Neuron Valley and ReLU}
\label{app:neuron-relu}

\emph{The neuron valley.} For \(\ell(w,v)=\frac12(v\sigma(w)-1)^2\), write \(r=v\sigma(w)-1\). On the valley \(r=0\), the Hessian is the rank-one matrix \(\nabla r\,\nabla r^\top\) with \(\nabla r=(v\sigma'(w),\,\sigma(w))\), so the nonzero eigenvalue is
\begin{align*}
    \lambda(w)
    =
    v^2\sigma'(w)^2+\sigma(w)^2
    =
    \frac{\sigma'(w)^2}{\sigma(w)^2}+\sigma(w)^2,
    \qquad
    v=\frac1{\sigma(w)} .
\end{align*}
For \(\sigma=\tanh\), with \(t=\tanh^2 w\), \(\lambda=\frac1t-2+2t\); differentiating gives the unique minimum at \(t=1/\sqrt2\) with
\(\lambda_\ast=2(\sqrt2-1)\) and \(2/\lambda_\ast=1+\sqrt2\); as \(w\to\infty\), \(\lambda\to1\), so the saturated plateau is stable precisely for \(a<2\).

\emph{ReLU.} For \(\sigma=\mathrm{ReLU}\), the active region is exactly the linear two-factor model. The threshold \(a_*=(\sqrt{27}/2)-1\) is the envelope value of Section~\ref{sec:quartic-envelope} at which the positive invariant well first reaches the origin. Below this value the active well is forward invariant; above it, the image of the positive well extends into the absorbing half-line, which proves Proposition~\ref{prop:relu-death}.

\section{Derivations for Section~\ref{sec:stochastic}: Stochastic Targets}
\label{app:stochastic}

Throughout, \(r=xy-1\), \(\rho_t=r_t-\xi_t\) with \(\xi_t\sim\mathcal N(0,\sigma^2)\) i.i.d.,
and \(\xi_t\) is independent of the state \((x_t,y_t)\). (The noise \(\xi_t\) is unrelated to
the two-cycle coordinate \(\xi_\pm\) of Appendix~\ref{app:data-coupling}.)

\subsection{The Sub-Edge Balancing Rate}
\label{app:stochastic-rate}

We prove Proposition~\ref{prop:stochastic-balancing}. On the diagonal \(x=y=s\) the noisy update
is \(s_{t+1}=s_t(1-a\rho_t)\), so the residual satisfies the exact recursion
\begin{align*}
    r_{t+1}=(1+r_t)(1-a\rho_t)^2-1 .
\end{align*}
Expanding at the balanced point \(r=0\) and keeping leading orders,
\begin{align}
    r_{t+1}=(1-2a)\,r_t+2a\,\xi_t+O(r_t^2,r_t\xi_t,\xi_t^2).
    \label{eq:app-ar1}
\end{align}
For \(0<a<1\) the coefficient satisfies \(|1-2a|<1\), so \eqref{eq:app-ar1} is a stable AR(1)
recursion driven by the fresh noise \(2a\xi_t\). Its stationary variance is
\begin{align*}
    \operatorname{Var}(r)
    =
    \frac{(2a)^2\sigma^2}{1-(1-2a)^2}
    =
    \frac{a}{1-a}\,\sigma^2+O(\sigma^4).
\end{align*}
Since \(\xi_t\) is independent of \(r_t\) (which depends only on \(\xi_{t-1},\xi_{t-2},\dots\)),
\begin{align*}
    \operatorname{Var}(\rho)
    =
    \operatorname{Var}(r)+\sigma^2
    =
    \frac{\sigma^2}{1-a}+O(\sigma^4),
    \qquad
    \mathbb E\rho=O(\sigma^2).
\end{align*}
In the stationary regime \(\rho\) is Gaussian to leading order, so \(\mathbb
E\rho^4=O(\sigma^4)\), and expanding the logarithm in \eqref{eq:chi-def},
\begin{align*}
    \chi(a,\sigma)
    =
    \mathbb E\log|1-a^2\rho^2|
    =
    -a^2\,\mathbb E\rho^2+O(\mathbb E\rho^4)
    =
    -\frac{a^2\sigma^2}{1-a}+O(\sigma^4).
\end{align*}
The expectation is finite despite the logarithmic singularity at \(|a\rho|=1\): the singularity
is integrable against the Gaussian density, and since \(\operatorname{sd}(\rho)=O(\sigma)\), the
region \(|a\rho|\approx1\) carries probability \(O(e^{-c/\sigma^2})\), which contributes beyond
all orders in \(\sigma\). Off the diagonal, a small imbalance perturbs the recursion
\eqref{eq:app-ar1} only at \(O(\Delta^2)\), and \(\chi<0\) contracts \(\Delta\), so the balanced
stationary regime is self-consistent. \hfill\(\square\)

\subsection{The Noisy Edge: Normal Form, Stationary Density, and Crossover}
\label{app:stochastic-crossover}

\emph{The noisy flip map.}
On the diagonal, the exact residual recursion is
\begin{align*}
    r_{t+1}
    =
    (1+r_t)(1-a\rho_t)^2-1,
    \qquad
    \rho_t=r_t-\xi_t .
\end{align*}
Near \(a=1\), with \(\delta=a-1\) and \(r\) small, this expands as
\begin{align*}
    r_{t+1}
    =
    -(1+2\delta)\,r_t
    -
    r_t^2
    +
    2\xi_t
    +
    O(r_t^3,\delta r_t^2,r_t\xi_t,\xi_t^2).
\end{align*}
This is the noisy flip normal form \eqref{eq:noisy-flip-map}: the deterministic residual alternates sign, while the fresh target noise enters additively to leading order.

\emph{The slow envelope.}
The leading multiplier \(-(1+2\delta)\) changes the sign of \(r_t\) at every step, while the magnitude changes slowly near the edge. The slow variable is therefore
\begin{align*}
    v_t=r_t^2 .
\end{align*}
Squaring the noisy flip map gives an increment of the form
\begin{align*}
    v_{t+1}-v_t
    =
    4\delta v_t
    +
    2\operatorname{sign}(r_t)v_t^{3/2}
    -
    4r_t\xi_t
    +
    4\xi_t^2
    +
    \cdots .
\end{align*}
The odd term alternates with the sign of \(r_t\) and cancels to leading order over consecutive steps. Its net effect is the deterministic nonlinear saturation already fixed by the two-cycle normal form. Without noise, the squared residual relaxes to the deterministic two-cycle amplitude
\begin{align*}
    v^*=a-1+O((a-1)^2).
\end{align*}
The quadratic drift
\begin{align*}
    4\delta v-4v^2
\end{align*}
has exactly this positive equilibrium \(v^*=\delta\) to leading order, and its corresponding transverse rate \(-v^*\) agrees with
\begin{align*}
    \frac12\log(3-2a)=-(a-1)+O((a-1)^2).
\end{align*}

\emph{Noise terms.}
The square of the fresh noise contributes the positive mean drift
\begin{align*}
    \mathbb E[4\xi_t^2]=4\sigma^2 .
\end{align*}
The cross term \(-4r_t\xi_t\) has conditional mean zero and conditional variance
\begin{align*}
    16\sigma^2 v_t .
\end{align*}
Higher noise terms contribute only at subleading order in the edge scaling below.

\emph{Diffusion approximation.}
Balancing the drift terms gives
\begin{align*}
    \delta v\sim v^2\sim\sigma^2,
    \qquad
    \text{hence}
    \qquad
    v=O(\sigma),
    \quad
    \delta=O(\sigma).
\end{align*}
In this regime, the one-step increment of \(v\) is small compared with \(v\) itself. Matching the conditional mean and variance of the slow increment gives the diffusion approximation
\begin{align*}
    \mathrm d v
    =
    \bigl(4\delta v-4v^2+4\sigma^2\bigr)\,\mathrm dt
    +
    4\sigma\sqrt v\,\mathrm dB_t ,
\end{align*}
which is \eqref{eq:noisy-edge-v-diffusion}. This is a normal-form approximation for the slow envelope of the original map; we do not prove convergence of the discrete process to this diffusion.

\emph{Stationary density.}
For a diffusion
\begin{align*}
    \mathrm dv=\mu(v)\,\mathrm dt+g(v)\,\mathrm dB_t
\end{align*}
with
\begin{align*}
    \mu(v)=4\delta v-4v^2+4\sigma^2,
    \qquad
    g^2(v)=16\sigma^2v,
\end{align*}
the zero-flux stationary density is
\begin{align*}
    p(v)
    \propto
    \frac1{g^2(v)}
    \exp\left\{
        \int^v \frac{2\mu(s)}{g^2(s)}\,ds
    \right\}.
\end{align*}
Since
\begin{align*}
    \frac{2\mu(v)}{g^2(v)}
    =
    \frac{\delta-v}{2\sigma^2}
    +
    \frac1{2v},
\end{align*}
we obtain
\begin{align*}
    p(v)
    \propto
    v^{-1/2}
    \exp\!\left\{
        \frac{\delta v}{2\sigma^2}
        -
        \frac{v^2}{4\sigma^2}
    \right\},
    \qquad v>0,
\end{align*}
which is \eqref{eq:noisy-edge-stationary-density}. The singular factor \(v^{-1/2}\) is integrable at the origin, and the positive drift \(4\sigma^2\) pushes the process into the interior.

\emph{The exponent reduces to the mean.}
Near the edge, \(a^2=1+O(\sigma)\) and
\begin{align*}
    \rho^2=(r-\xi)^2=v-2r\xi+\xi^2.
\end{align*}
Therefore
\begin{align*}
    \chi
    =
    \mathbb E\log|1-a^2\rho^2|
    =
    -\mathbb E\rho^2+O(\mathbb E\rho^4)
    =
    -\langle v\rangle-\sigma^2+O(\sigma^2).
\end{align*}
Since \(\langle v\rangle=O(\sigma)\), the \(\sigma^2\) term is lower order, and
\begin{align*}
    \chi=-\langle v\rangle+o(\sigma).
\end{align*}

\emph{The crossover function.}
Set
\begin{align*}
    w=\frac{\delta}{\sqrt2\,\sigma},
    \qquad
    v=\sqrt2\,\sigma z .
\end{align*}
Then the stationary density becomes, up to normalization,
\begin{align*}
    z^{-1/2}e^{wz-z^2/2}\,dz .
\end{align*}
Hence
\begin{align*}
    \langle v\rangle
    =
    \sqrt2\,\sigma\,F(w),
\end{align*}
where
\begin{align*}
    F(w)
    =
    \frac{\int_0^\infty z^{1/2}e^{wz-z^2/2}\,\mathrm dz}
         {\int_0^\infty z^{-1/2}e^{wz-z^2/2}\,\mathrm dz}.
\end{align*}
This gives
\begin{align*}
    \chi(a,\sigma)
    =
    -\sqrt2\,\sigma F(w)+o(\sigma),
    \qquad
    w=\frac{a-1}{\sqrt2\,\sigma}.
\end{align*}

\emph{Asymptotes and the edge value.}
For \(w\to+\infty\), the weight \(e^{wz-z^2/2}\) concentrates near \(z=w\), so
\begin{align*}
    F(w)=w(1+o(1)),
\end{align*}
and
\begin{align*}
    \chi=-\sqrt2\,\sigma w=-(a-1),
\end{align*}
recovering the deterministic two-cycle rate. For \(w\to-\infty\), the integrals localize near \(z=0\), and
\begin{align*}
    F(w)=\frac1{2|w|}(1+o(1)).
\end{align*}
Thus
\begin{align*}
    \chi=-\frac{\sigma^2}{1-a},
\end{align*}
recovering the near-edge limit of \eqref{eq:stochastic-balancing-rate}.

At \(w=0\),
\begin{align*}
    F(0)
    =
    \sqrt2\,\frac{\Gamma(3/4)}{\Gamma(1/4)},
\end{align*}
and hence
\begin{align*}
    \chi(1,\sigma)
    =
    -\frac{2\Gamma(3/4)}{\Gamma(1/4)}\,\sigma
    +
    o(\sigma)
    \approx
    -0.6760\,\sigma .
\end{align*}

\emph{Numerical check.}
We measure \(\chi\) as the time average \eqref{eq:chi-def} along long trajectories of the exact noisy map. Plotted as \(\chi/\sigma\) against
\begin{align*}
    w=\frac{a-1}{\sqrt2\sigma},
\end{align*}
the data collapse onto \(-\sqrt2 F(w)\) with no fitted parameter, as shown in Figure~\ref{fig:noise-crossover}. The measured-to-predicted ratio approaches one as \(\sigma\) decreases, consistent with the \(o(\sigma)\) error in \eqref{eq:noisy-edge-crossover-law}.

\section{A Reader's Toolbox}
\label{app:toolbox}

This appendix explains, informally and with a minimum of notation, the concepts from one-dimensional and low-dimensional dynamics used in the paper. It can be read independently of everything else.

\paragraph{Maps and orbits.}
A \emph{map} is a rule \(f\) that turns a number \(x\) into a new number \(f(x)\). Iterating the rule --- \(x_0,\ x_1=f(x_0),\ x_2=f(x_1),\ldots\) --- produces an \emph{orbit}. Gradient descent with a fixed step size \(a\) on a loss \(\ell\) is exactly such a rule, \(f(x)=x-a\ell'(x)\), and training is the orbit. Everything in this paper is a statement about orbits of explicit maps.

\paragraph{Fixed points and the multiplier.}
A \emph{fixed point} is a value with \(f(x^*)=x^*\); for gradient descent these are the critical points of the loss. Whether nearby orbits approach or flee is decided by the \emph{multiplier} \(\mu=f'(x^*)\): a small deviation is multiplied by \(\mu\) at every step, so the fixed point attracts when \(|\mu|<1\) and repels when \(|\mu|>1\). For gradient descent, \(f'(x)=1-a\ell''(x)\), so a minimum with curvature (``sharpness'') \(\lambda\) has multiplier \(1-a\lambda\): stable exactly when \(a\lambda<2\). This is the origin of the ubiquitous threshold \(2/a\). If \(0<\mu<1\), orbits approach from one side (monotone convergence); if \(-1<\mu<0\), they approach while alternating sides. Losing stability through \(\mu=-1\) --- the case everywhere in this paper --- means the iterates begin to overshoot alternately.

\paragraph{Cycles and their multipliers.}
A \emph{period-\(m\) cycle} is a set of \(m\) points that the map visits in turn before returning. Its stability is governed by the product of the derivatives along the cycle, \(\mu=f'(x_1)\cdots f'(x_m)\) --- the natural generalization of the fixed-point multiplier, in higher dimensions called a Floquet multiplier. Again \(|\mu|<1\) means attraction. Stability is a property of the \emph{product}: individual factors may well exceed one in absolute value, which is precisely what happens in ``sharpness hovering,'' where the cycle visits one point that is too sharp and one that is flat enough, and the pair is jointly stable. A cycle with \(\mu=0\) is called \emph{superstable}; convergence to it is exceptionally fast.

\paragraph{Bifurcations.}
A \emph{bifurcation} is a qualitative change of the orbit structure as a parameter (here: the step size) is varied. The one that dominates this paper is the \emph{flip} (period-doubling) bifurcation: as the multiplier of a fixed point decreases through \(-1\), the fixed point hands its stability to a newborn cycle of twice the period. In a \emph{supercritical} flip the new cycle is stable and starts with zero amplitude, growing like the square root of the parameter excess --- a soft transition, not an explosive one. Two other types appear briefly: the \emph{saddle-node} bifurcation, where a pair of cycles (one stable, one unstable) appears out of nothing --- this creates the periodic ``windows'' inside chaos --- and the \emph{Neimark--Sacker} bifurcation, where a pair of complex multipliers crosses the unit circle and the instability grows as a rotation (``precession'') rather than along a fixed direction.

\paragraph{The period-doubling cascade and universality.}
After the first flip, the new two-cycle can itself flip, producing a four-cycle, then an eight-cycle, and so on. The parameter values of these successive flips pile up geometrically fast at a finite \emph{accumulation point}, beyond which the dynamics is chaotic, interrupted by windows. The geometric rate is the same number, \(4.669\ldots\), for essentially every smooth one-hump map --- Feigenbaum universality. This is why the cascade of a gradient map looks exactly like the cascade of a population model: it is the same universal object.

\paragraph{One-hump maps and the critical point.}
Many maps in this paper, after a change of coordinates, rise to a single maximum and then fall --- a \emph{one-hump} (unimodal) map. The point where the map turns around, i.e.\ where \(f'=0\), is called the \emph{critical point} of the map (not of the loss). Its orbit is the single most informative trajectory of the whole system: plotting it against the parameter produces the bifurcation diagrams shown in this paper.

\paragraph{The Schwarzian derivative and Singer's theorem.}
The \emph{Schwarzian derivative} \(Sf=f'''/f'-\tfrac32(f''/f')^2\) is a technical quantity with one powerful consequence. Singer's theorem: if \(Sf<0\) everywhere (away from critical points), then every attracting cycle must attract the orbit of a critical point or of a boundary point. In plain terms, a negative-Schwarzian map can hide nothing --- following the critical orbit finds every attractor there is, and the number of coexisting attractors is bounded by the number of critical points. This is what elevates a numerically drawn bifurcation diagram to a faithful picture: proving \(Sf<0\), as we do at every depth, certifies that the diagram is complete.

\paragraph{Chaos, Lyapunov exponents, and invariant densities.}
An orbit is \emph{chaotic} when nearby orbits separate exponentially fast; the average logarithmic separation rate per step is the \emph{Lyapunov exponent} --- negative for orbits attracted to cycles, positive for chaos. A chaotic orbit typically fills an interval with a well-defined visiting frequency, its \emph{invariant density}. In rare cases this density is known in closed form; the endpoint \(a=2\) of the cubic map, where the dynamics is equivalent to tripling an angle, is such a case.

\paragraph{Crises.}
A \emph{crisis} is a sudden change of a chaotic attractor when it collides with another invariant object as the parameter varies. Two kinds occur here: an \emph{interior} crisis, where two separate chaotic bands merge into one (the well-merging threshold \(a_*\)), and a \emph{boundary} crisis, where the attractor touches the boundary of its basin and is destroyed --- beyond it, orbits escape (the endpoint \(a=2\)). Near a crisis, orbits linger for long stretches in the old attractor before finding the exit, and the closer to the threshold, the longer the lingering (\emph{intermittency}).

\paragraph{Invariant manifolds and transverse stability.}
In systems with more than one variable, a subset that the dynamics never leaves --- like the balanced diagonal \(x=y\) of the two-factor model --- is an \emph{invariant manifold}. Whether nearby orbits converge to it is measured by the \emph{transverse} multiplier or exponent: the average factor by which the distance to the manifold shrinks per step. A negative transverse exponent means that the reduced, lower-dimensional description is self-justifying: the full system falls onto it.

\paragraph{Scaling limits and normal forms.}
Two complementary simplification devices recur. A \emph{scaling limit} magnifies the neighborhood of a point while a parameter grows (here: depth), so that a whole family of maps converges to a single limiting map --- the source of the universal depth-limit dynamics. A \emph{normal form} is the polynomial truncation of a map near a bifurcation that keeps exactly the terms governing the local behavior; it is how one decides, for instance, whether a flip is supercritical, and it is the starting point of the noisy analysis in Section~\ref{sec:stochastic}.

\paragraph{Random maps.}
When noise is added to a map, single orbits are replaced by probability distributions. A stable cycle becomes a narrow stationary distribution concentrated on the cycle; a bifurcation is smeared over a parameter window of the order of the noise strength; and quantities like the transverse exponent become averages over the stationary distribution --- Lyapunov exponents of products of random factors. The key structural fact used in Section~\ref{sec:stochastic} is that some identities hold \emph{pathwise}, i.e.\ for every realization of the noise separately, so they survive averaging untouched.

\bibliographystyle{acm}
\bibliography{bib}
\end{document}